\documentclass[11pt]{article}
\usepackage{cite}
\usepackage{amsmath,amsthm,amssymb,enumerate,fullpage}
\usepackage{xcolor}
\usepackage{svg}
\usepackage{enumitem}
\usepackage{array}
\usepackage{caption}
\usepackage{subcaption}
\usepackage{graphicx}
\usepackage{diagbox,url}
\usepackage{algorithm}
\usepackage{algorithmic}
\usepackage{float}
\usepackage{adjustbox}
\usepackage{multirow}
\usepackage{hyperref}

\usepackage{tikz}
\tikzstyle{startstop} = [rectangle, rounded corners, minimum width=1cm, minimum height=2cm,text centered, draw=black, fill=red!30]
\tikzstyle{process} = [rectangle, minimum width=1cm, minimum height=2cm, text centered, draw=black, fill=orange!30]
 \tikzstyle{decision} = [diamond, minimum width=3cm, minimum height=1cm, text centered, draw=black, fill=green!30]
\usetikzlibrary{shapes, arrows}
\usetikzlibrary{decorations.markings}
\usetikzlibrary{calc}
\usetikzlibrary{spy}

\newtheorem{definition}{Definition}
\newtheorem{proposition}{Proposition}

\renewcommand{\bezier}{B\'{e}zier }

\title{Image Vectorization with Depth:\\ 
convexified shape layers with depth ordering}

\author{Ho Law and Sung Ha Kang \footnote{School of Mathematics, Georgia Institute of Technology, Atlanta, GA, USA (hlaw@gatech.edu, and Kang@math.gatech.edu). This work is partially supported by Simons Foundation 584960.} }

\date{}
\begin{document}

\maketitle

\begin{abstract}
Image vectorization is a process to convert a raster image into a scalable vector graphic format. Objective is to effectively remove the pixelization effect while representing boundaries of image by scaleable parameterized curves.  
We propose new image vectorization with depth which considers depth ordering among shapes and use curvature-based inpainting for convexifying shapes in vectorization process. 
From a given color quantized raster image, we first define each connected component of the same color as a shape layer, and construct depth ordering among them using a newly proposed depth ordering energy.  Global depth ordering among all shapes is described by a directed graph, and we propose an energy to remove cycle within the graph. 
After constructing depth ordering of shapes, we convexify occluded regions by Euler's elastica curvature-based variational inpainting, and leverage on the stability of Modica-Mortola double-well potential energy to inpaint large regions.  This is following human vision perception that boundaries of shapes extend smoothly, and we assume shapes are likely to be convex. 
Finally, we fit B\'{e}zier curves to the boundaries and save vectorization as a SVG file which allows superposition of curvature-based inpainted shapes following the depth ordering. 
This is a new way to vectorize images, by decomposing an image into scalable shape layers with computed depth ordering.  This approach makes editing shapes and images more natural and intuitive.  We also consider grouping shape layers for semantic vectorization.  We present various numerical results and comparisons against recent layer-based vectorization methods to validate the proposed model.  
\end{abstract}

\section{Introduction}
Image vectorization, also known as image tracing, is a crucial technique in animation, graphic design, and printing~\cite{ma2022layer, li2020diffvg}. Unlike raster images which store color values at each pixel, vectorized images use geometric primitives like lines, curves, and shapes to represent the image. These images are typically saved in Scalable Vector Graphics (SVG) format, offering several advantages. First, they can be infinitely scaled without losing quality, eliminating the staircase effect on edges. Second, the file size can be significantly reduced, especially for large images with simple shapes and limited colors. Finally, SVG files are easier to load and render. As a text-based format, SVGs can be edited directly or through a user interface, allowing for easy modification and rearrangement of elements.
One of the earliest software for image vectorization is AutoTrace \cite{autortrace} developed in 1999. Then, Potrace~\cite{Selinger2003PotraceA} and Adobe Streamline~\cite{adobestreamline} emerged in 2001.  Some of current state-of-the-art commercial software includes Adobe Illustrator~\cite{adobeillustrator} and VectorMagic~\cite{vectormagic}.
Image vectorization can be characterized to contour-based and patch-based methods. Contour-based methods understand input images as a collection of curves, and aim to restore inputs as vector graphics of a collection of simple geometry elements, including lines, \bezier curves and ellipses, to represent the color intensity discontinuity.  Kopf et al. \cite{kopf2011pixel} aimed at preserving feature connectivity and reduce pixel aliasing artifacts when fitting spline curves to contours for pixel art.  In \cite{he2021silhouette}, the authors considered binary images and used affine-shortening flow to reduce pixelization effect and to find the meaningful high curvature points before fitting \bezier curves.  This approach is geometrically stable under affine transformation, and shows reduction in the number of control points while maintaining high image quality.  In ~\cite{he2022silhouette}, from the color quantized raster image,  color image vectorization is explored by carefully keeping track of T-junctions and X-junctions during the curve smoothing process of vectorization.  A new color image vectorization  based on region merging is explored in \cite{he2023viva} which is free from color quantization.
Patch-based methods, on the other hand, utilize meshes in different ways to capture fine details in raster images, such as using gradient mesh to capture the contrast changes in \cite{lai2009patch}, and using curvilinear feature alignment in \cite{xia2009patch}.
Some deployed neural networks for vectorization tasks: for drawings~\cite{egiazarian2020vectorization}, floorplans~\cite{liu2017vectorization}, a generative model for font vectorization~\cite{lopes2019learned} and exploring latent space for vectorized output in \cite{reddy2021im2vec}.

Color image vectorization is a challenging problem. Color quantized real images usually contain a lot of tiny piecewise constant regions due to contrast, reflection, and shading. Improper denoising of such regions may result in oscillatory boundaries and poor vectorization quality. Another challenge is due to the staircase effect of raster images, even for piecewise constant images. Figure \ref{fig:intro} shows a typical example, where (a) shows the color quantized image $f$, and (b) shows a typical vectorization result~\cite{adobeillustrator}.  Each connected component is vectorized separately and T-junctions may show unnatural artifacts.  It is desirable if the boundaries of triangles are reconstructed as straight lines and arcs of the circles are reconstructed  following the curvature directions as in (c).  
In order to achieve such geometrically meaningful reconstruction, we allow each connected component to be considered as a region possibly occluded by another shape.  We leverage on Euler's elastica curvature-based inpainting to convexify these shapes, but in order to define occluded region, we propose an energy to give depth ordering among the shapes. 

\begin{figure}
  \centering
  \begin{tabular}{ccc}
  (a) & (b) & (c)\\
  \begin{tikzpicture} [spy using outlines={rectangle,red,magnification=15,size=1.75cm, connect spies}]
  \node[anchor=south west,inner sep=0] (image) at (1.5,0) {\includegraphics[width = 4cm]{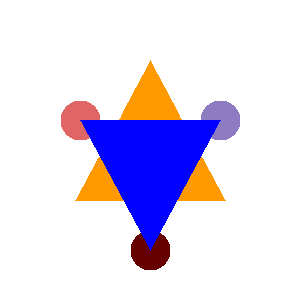}};
  \spy[red, connect spies] on (3.15, 1.315) in node at (2.5,-0.51);
  \spy[red, connect spies] on (4.32, 2.16) in node at (4.5,-0.51);
  \end{tikzpicture}
  &\begin{tikzpicture} [spy using outlines={rectangle,red,magnification=15,size=1.75cm, connect spies}]
  \node[anchor=south west, inner sep=0] (image) at (1.5,0){ \includegraphics[width = 4cm]{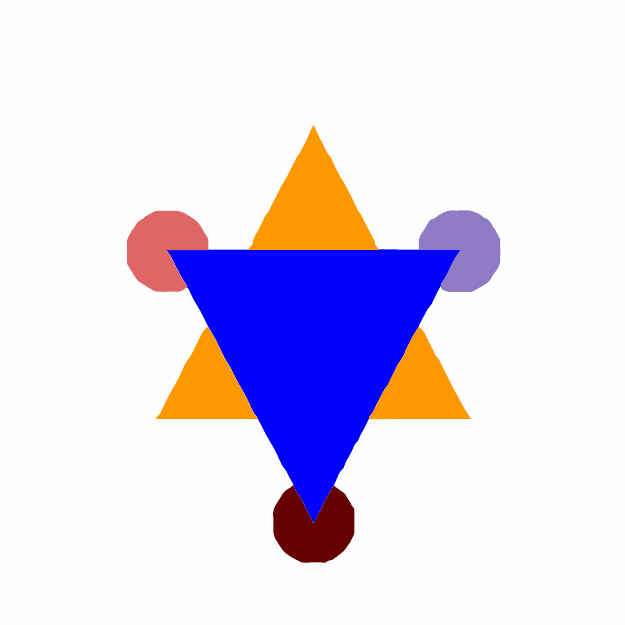}};
  \spy[red, connect spies] on (3.15, 1.315) in node at (2.5,-0.51);
  \spy[red, connect spies] on (4.32, 2.16) in node at (4.5,-0.51);
  \end{tikzpicture}
  &
  \begin{tikzpicture} [spy using outlines={rectangle,red,magnification=15,size=1.75cm, connect spies}]
\node[anchor=south west,inner sep=0] (image) at (1.5,0) {\includegraphics[width = 4cm]{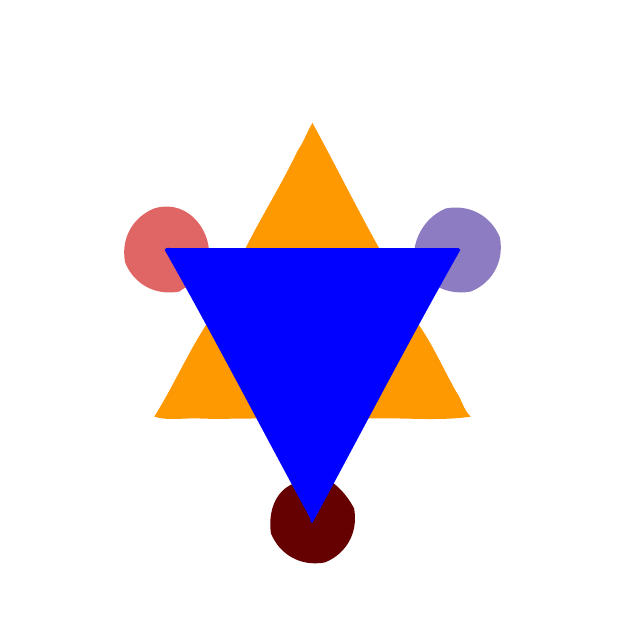}};
  \spy[red, connect spies] on (3.15, 1.315) in node at (2.5,-0.51);
  \spy[red, connect spies] on (4.32, 2.16) in node at (4.5,-0.51);
  \end{tikzpicture}
  \end{tabular}
\caption{(a) Given color quantized raster input image $f$. (b) Typical vectorization result by \cite{adobeillustrator} which considers each connected component separately. (c) Desired vectorization of curves using our propose method. }
\label{fig:intro}
\end{figure}

Finding depth ordering and inpainting is closely related to segmentation with depth. In \cite{nitzberg1990seg,nitzberg1993seg}, the authors propose Nitzberg-Mumford-Shiota (NMS) functional which  decomposes image into shapes that are allowed to be superposed and minimizes curvature of occluded boundaries, while keeping the ordering result faithful to the input image.  Different optimization schemes are suggested for NMS functional: Nitzberg et al.~\cite{nitzberg1993seg} utilize T-junctions and combinatorial algorithm to avoid minimizing the functional directly, and in \cite{esedoglu2004segmentationWD}, authors minimize the functional directly without detecting T-junctions by approximating it with elliptic functionals. Zhu et al.~\cite{zhu2006segmentation} use the level set approach~\cite{osher2001levelset} to minimize the curvature in distribution sense and apply a fast semi-implicit discretization scheme. These methods all minimize the NMS functional first for every possible ordering, and choose the one that gives the minimum value to be the final ordering.  

For images with many objects, real image particularly, the complexity of considering every possible ordering increases geometrically with the number of shapes, and minimizing NMS functional for every possible ordering is nearly impractical. To circumvent this issue, we first estimate the depth ordering, then inpaint each shape with a curvature-based inpainting model in this paper.
Determining objects' relative depth ordering based on a single image is often referred as monocular depth ordering. In \cite{palou2013depth}, the authors locate the T-junctions in the input image, and consider several factors such as color, angles, curvature and local depth gradient, to determine the relative depth ordering.  In \cite{rezaeirowshan2016depth}, authors use convexity and T-junction cues to determine local depth ordering between two neighbouring objects, and then aggregate to a global depth ordering of all objects in the image.  There are training-based methods for identifying depth ordering, e.g., using convolutional neural network \cite{zhang2015depth}, simultaneously training segmentation and depth estimation \cite{mousavian2016depth}, unsupervised learning \cite{zhao2019geometry} and others \cite{depthreview1, depthreview2}. 
While T-junction is an important clue for depth ordering, it is not only difficult to identify them in raster images, but also often gives conflicting depth ordering information (mentioned in later section).
We view the given image $f$ as layered shapes to give a more semantic vectorization result rather than focusing on T-junctions. We consider the perception of completed occluded objects, such as convexity and its area measure, for a more stable computation of depth ordering. 

There are limited recent vectorization methods considering some layering approach. Ma et al.~\cite{ma2022layer} propose Layer-wise Image Vectorization (LIVE), learning-based method which vectorizes image while keeping image topology. LIVE~\cite{ma2022layer} progressively adds more curves to fit the given image, to minimize a loss function for both the color difference between the input and rendered output, and the geometry of produced \bezier curves. 
Wang et al. \cite{wang2024layered} propose Layered Image vectorization via Semantic Simplification (LIVSS). This method generates a sequence of simplified images given by sampling and segmentation, then using two modules, one for simplification and another for layered vectorization, LIVSS finds various level of details in vectorization. These methods use differentiable rasterizer (DiffVG)~\cite{li2020diffvg}, which allows computing gradient of a differentiable loss function with one raster and one vectorized image as inputs.  In~\cite{guo2017depth}, depth information is given in addition to the raster input, and the method outputs a diffusion curve image.

In this paper, we propose image vectorization with depth, which uses depth ordering and curvature-based inpainting for convexifying each shape layer. This approach is training-free and does not have progressive addition of curves. We make assumptions that shapes tend to be convex and level lines should be extended following the curvature direction. We propose a new depth ordering energy that gives depth ordering between  two shape layers based on the ratio of occluded area approximated by convex hulls. From the pairwise depth ordering information, we construct a directed graph amongst all shape layers in the image, where a directed edge indicates one shape is above the other. If there is a directed cycle in the graph, we use a new proposed energy, convex hull symmetric difference, to remove one edge in the cycle. 
To properly convexify each shape layer under occluded regions, we use Euler's elastica curvature-based inpainting to extend the boundary curves smoothly.  In particular, we construct inpainting corner phase functions defined at appropriate corners to guide the inpainting process.  Once each shape layer is reconstructed, we use \bezier curves to fit the boundary, and write as an SVG file following the reverse depth ordering, since unlike bitmap format, SVG format allows stacking shape layers. 
Contribution of our paper is as follow:
\begin{enumerate}
\item{We propose a new image vectorization method incorporating depth information.  The proposed method lowers computation complexity compared to traditional image segmentation with depth, and avoids long computation compared to learning-based methods. }
\item{We propose two new energies for depth ordering based on the convexity assumption of each shape layer: one determines pairwise depth ordering between any two shapes, and another removes cycles for building a linear global depth ordering.}
\item{This method decomposes image into sequence of shape layers considering each connected component of the same color as one shape later. Compared to existing layer-based vectorization methods, our method outputs more semantic layers of shape, which allow easy post-vectorization editing.}
\item{We utilize curvature-based inpainting for reconstructing occluded regions determined by the depth ordering. We leverage a stable and effective method of Modica-Mortola double-well potential approach for large domain curvature-based inpainting. }

\end{enumerate}
Our paper is organized as follow: We first present basic definitions and give an overview of our proposed method in Section \ref{sec: main method}.  The details of each steps are given in each subsection, starting from the new depth ordering energy of shape layers. Some analytical properties of depth ordering energy are explored in Section \ref{sec:analysis}. The details of Euler's elastica curvature-based inpainting functional and the inpainting corner phase function are presented in Section \ref{sec:elastica} and numerical details are presented in Section \ref{sec:numerical}.  In Section \ref{sec: experiments}, we present experiment results of the proposed method, and comparisons against other layer-based vectorization.  We conclude the paper in Section \ref{sec:conclude}.

\section{The proposed method: Image Vectorization with Depth} \label{sec: main method}
 
Let $\Omega$ be a discrete image domain $\{1, 2, \cdots, h \} \times \{1, 2, \cdots, w \}$ with a rectangular grid that $(i,j)$ is connected to $(i - 1, j), (i + 1, j), (i, j - 1)$, and $ (i, j + 1)$. Let $\tilde{f}$ be the input raster image, and we first color quantize the raster image $\tilde{f}$ and consider 
\[f: \Omega \to \{ c_l \}_{l=1}^{K}\]
as the \emph{color quantized input image} for $c_l \in \mathbb{R}^3$.    
Since image vectorization usually represents input raster images with fewer colors and simpler shapes, color quantization can effectively reduce the number of color.  There are different color quantization methods such as K-mean clustering on RGB color space~\cite{macqueen1967kmean}, total squared error minimisation~\cite{orchard1991quantization}, and adaptive distributing units algorithm~\cite{celebi2014quantization}.  We use $K$-mean clustering~\cite{macqueen1967kmean} on the image $\Omega$ with a pre-determined positive integer $K$, which is much smaller than the number of colors in $\tilde{f}$, for color quantization in this paper.  

\begin{definition}\label{def:shapelayer}
Let the color quantized input image be $f: \Omega \to \{c_l \}_{l=1}^{K}$ with $K$ number of colors.  For each color $c_l$, let $N_l$ be the total number of disjoint connected components $S^j_l \subset \Omega$ such that $ f(\bigcup_{j=1}^{N_l} S^j_{l}) = \{ c_l \} $ and $S^j_l \cap S^k_l = \emptyset$ for $j$ and $k = 1, \cdots, N_l$ and $j \neq k$. 
We define each $S^j_l$ as a \textbf{shape layer}, which we simply denote as $S_i$ with associated color $c_l$, i.e., 
\[
f(x) =  c_l \quad \text{ for } \quad\forall x \in S_i,\quad \text{ i.e., }\quad f(S_i) =  c_l.
\] 
We let $N_\mathcal{S}$ be the total number of shape layers, i.e. $N_\mathcal{S}=\sum_{l=1}^K N_l$, and $\mathcal{S}$ be the set of all shape layers of $f$.
\end{definition}
We note that each connected component with the same color in the discrete domain $\Omega$ is defined as a shape layer $S_i$, thus this is a region, and each associated color $c_l$ is only needed at the final vectorization step to record as SVG format.    Figure \ref{fig:shapelayer} (a) shows the color quantized rater image $f$, and (b) shows seven shape layers $S_i$ for $i=1,2,\dots, 7$. The shape layer index $i$ and the color index $l$ are independent to each other: $S_1$ is associated with black $c_1$, $S_2$ also with $c_1$, $S_3$ with orange $c_2$, $S_4$ with yellow $c_3$, $S_5$ and $S_6$ with white $c_4$, and the background $S_7$ with green $c_5$.  Note $S_1$ and $S_2$ have the same color black, and $S_5$ and $S_6$ have the same color white, but each connected component is defined as a separate shape layer.  

\begin{figure}
\centering
\begin{tabular}{ccccc}
(a) & \multicolumn{4}{c}{(b)} \\
\multirow{4}{*}{
\includegraphics[width = 5cm]{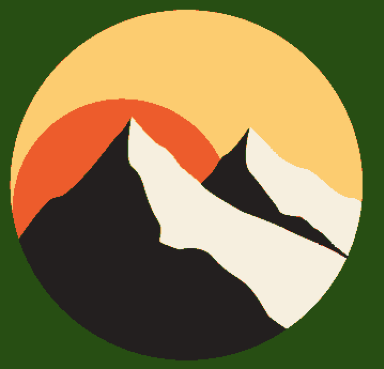}} & $S_1$ & $S_2$ &  $S_3$ &  $S_4$ \\
& \includegraphics[width=0.12\textwidth]{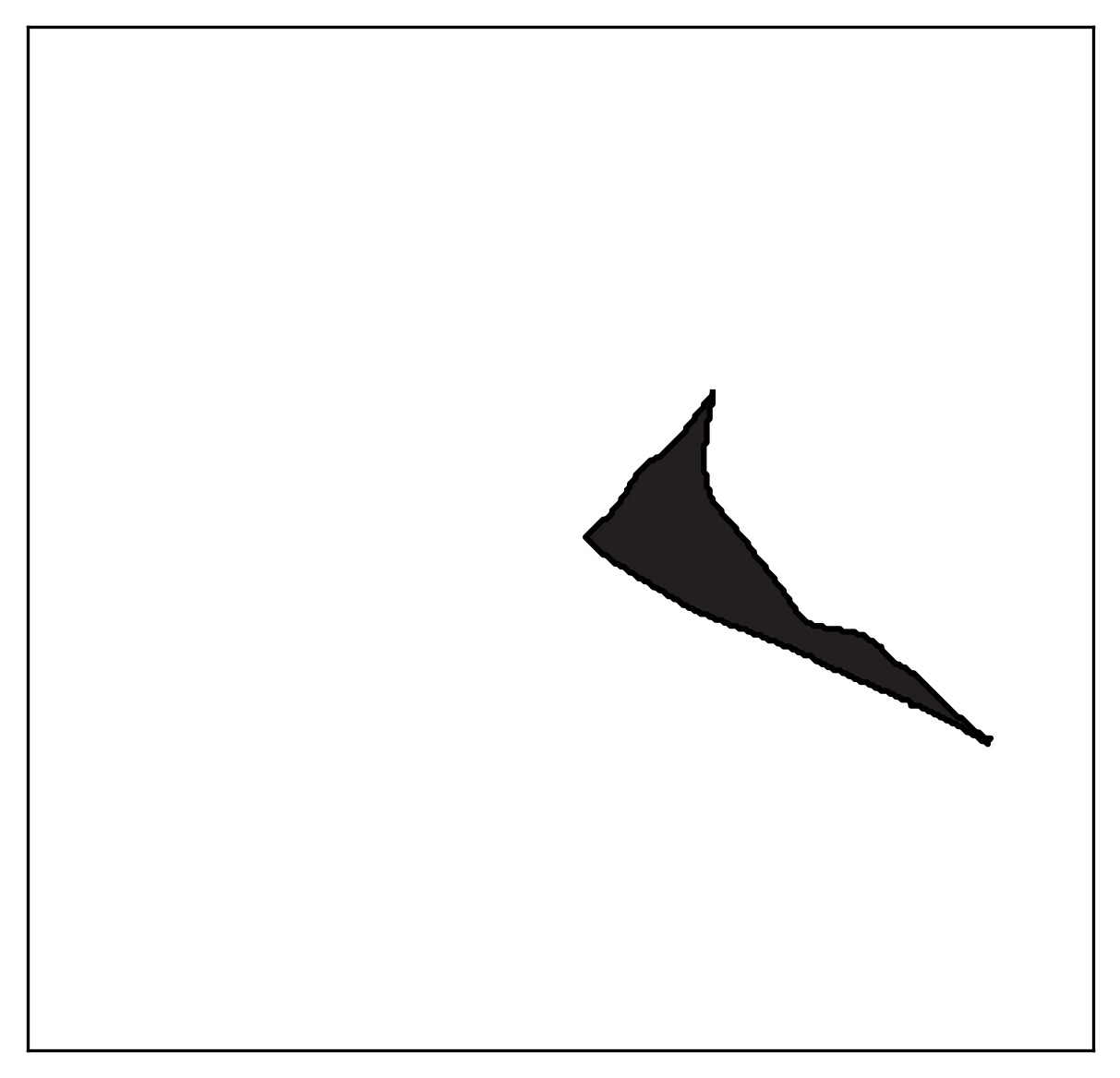} &
\includegraphics[width=0.12\textwidth]{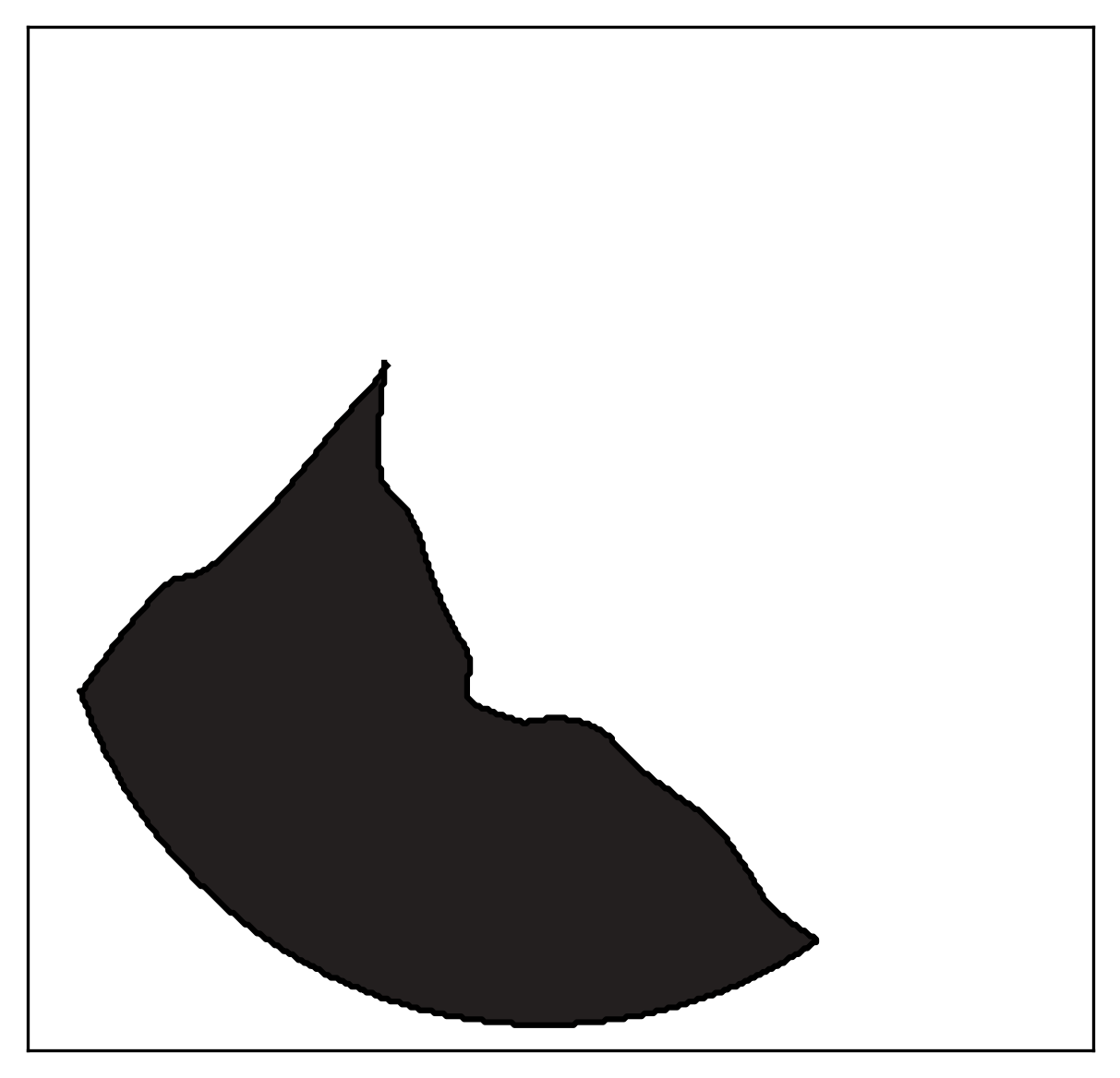} &
\includegraphics[width=0.12\textwidth]{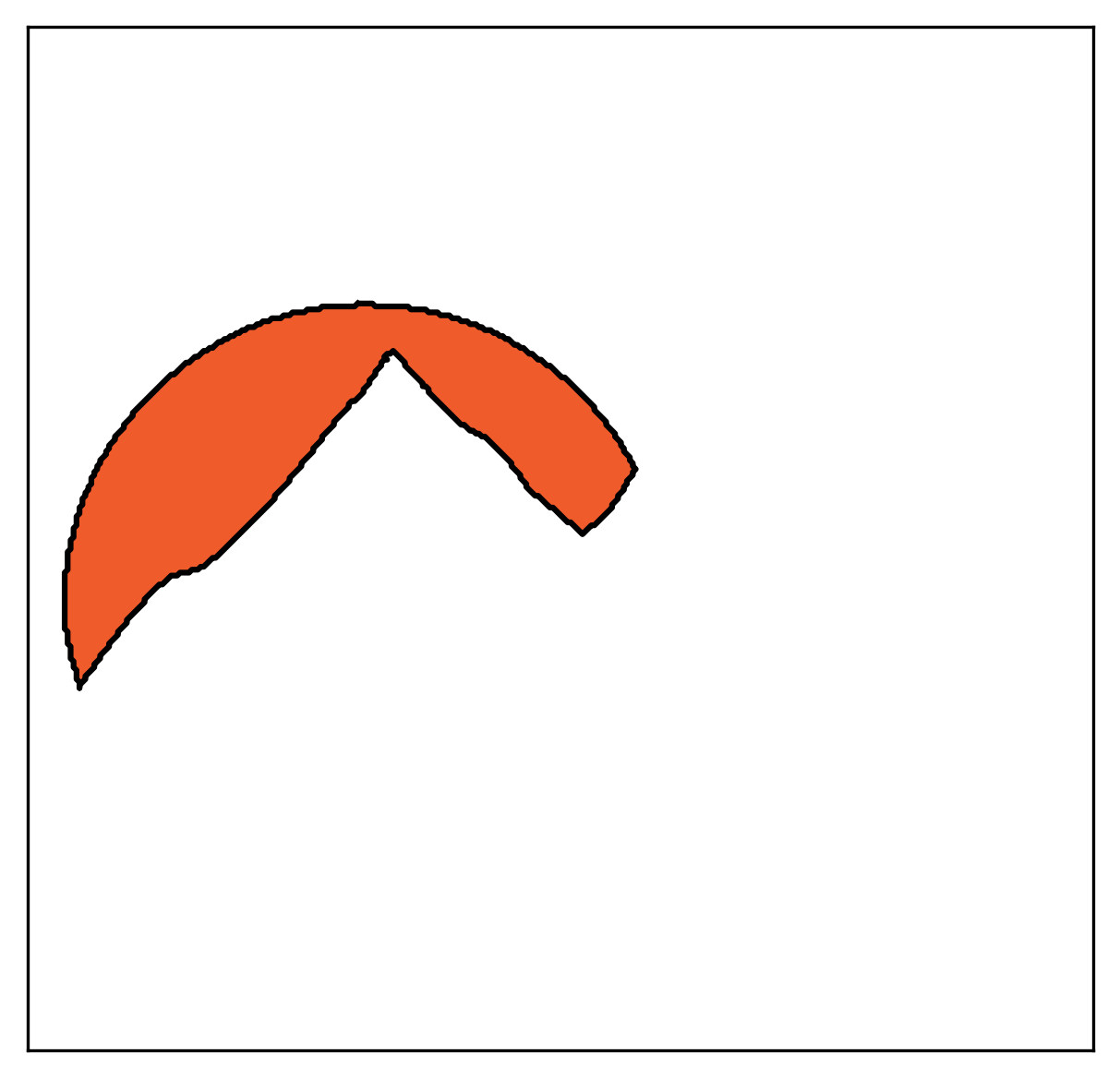}&
\includegraphics[width=0.12\textwidth]{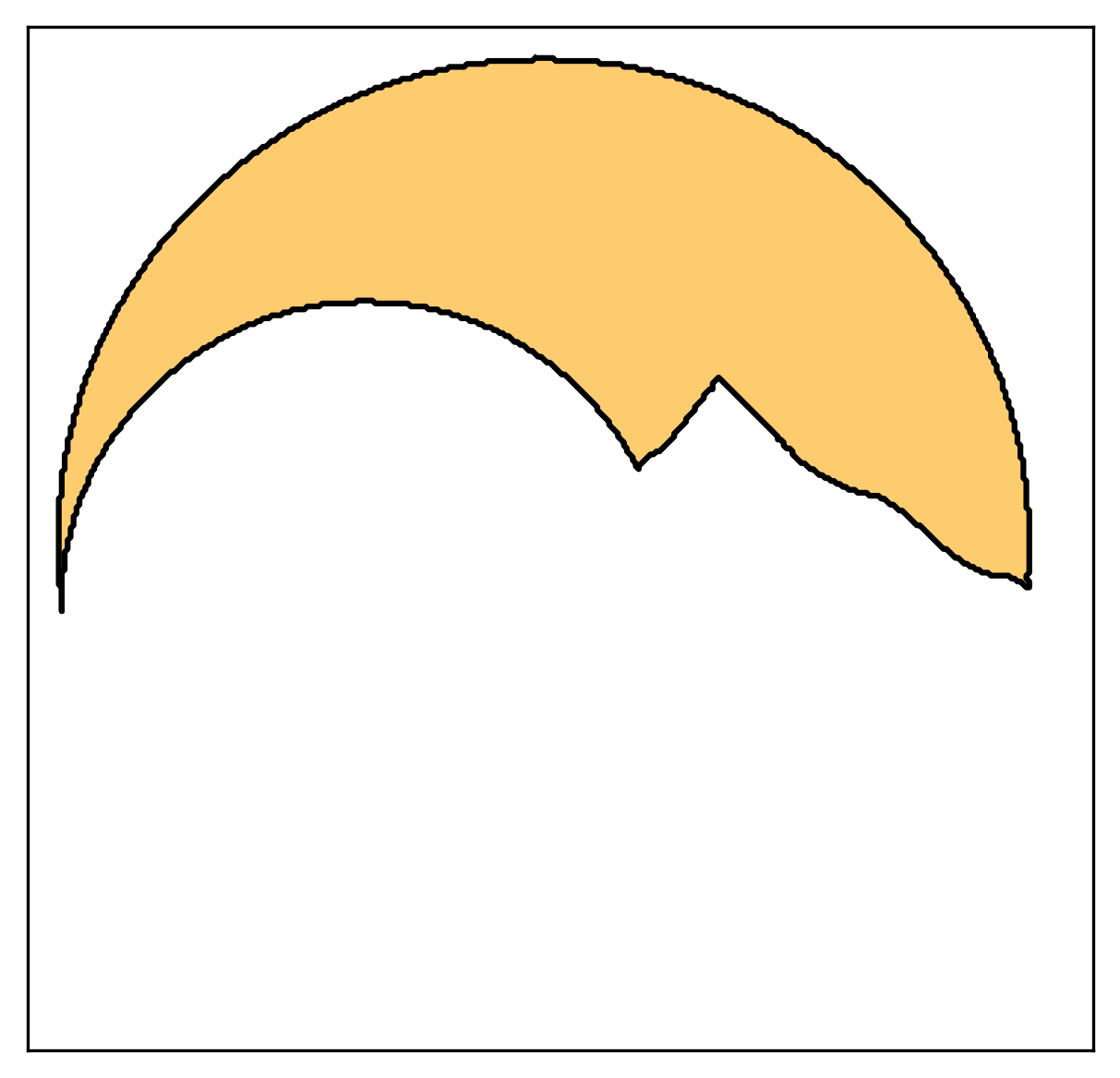} \\
& $S_5$ & $S_6$ &  $S_7$\\
& \includegraphics[width=0.12\textwidth]{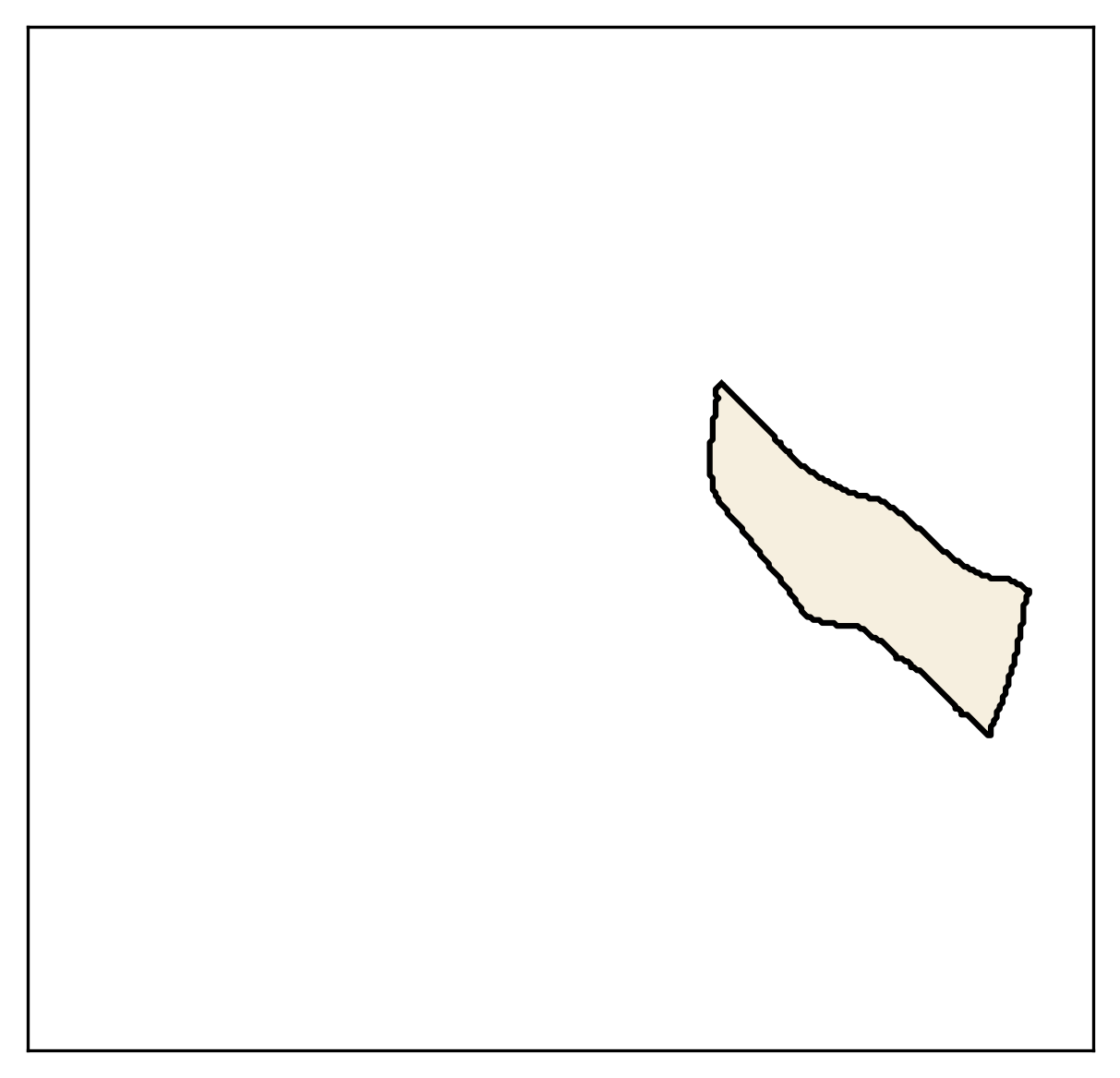} &
\includegraphics[width=0.12\textwidth]{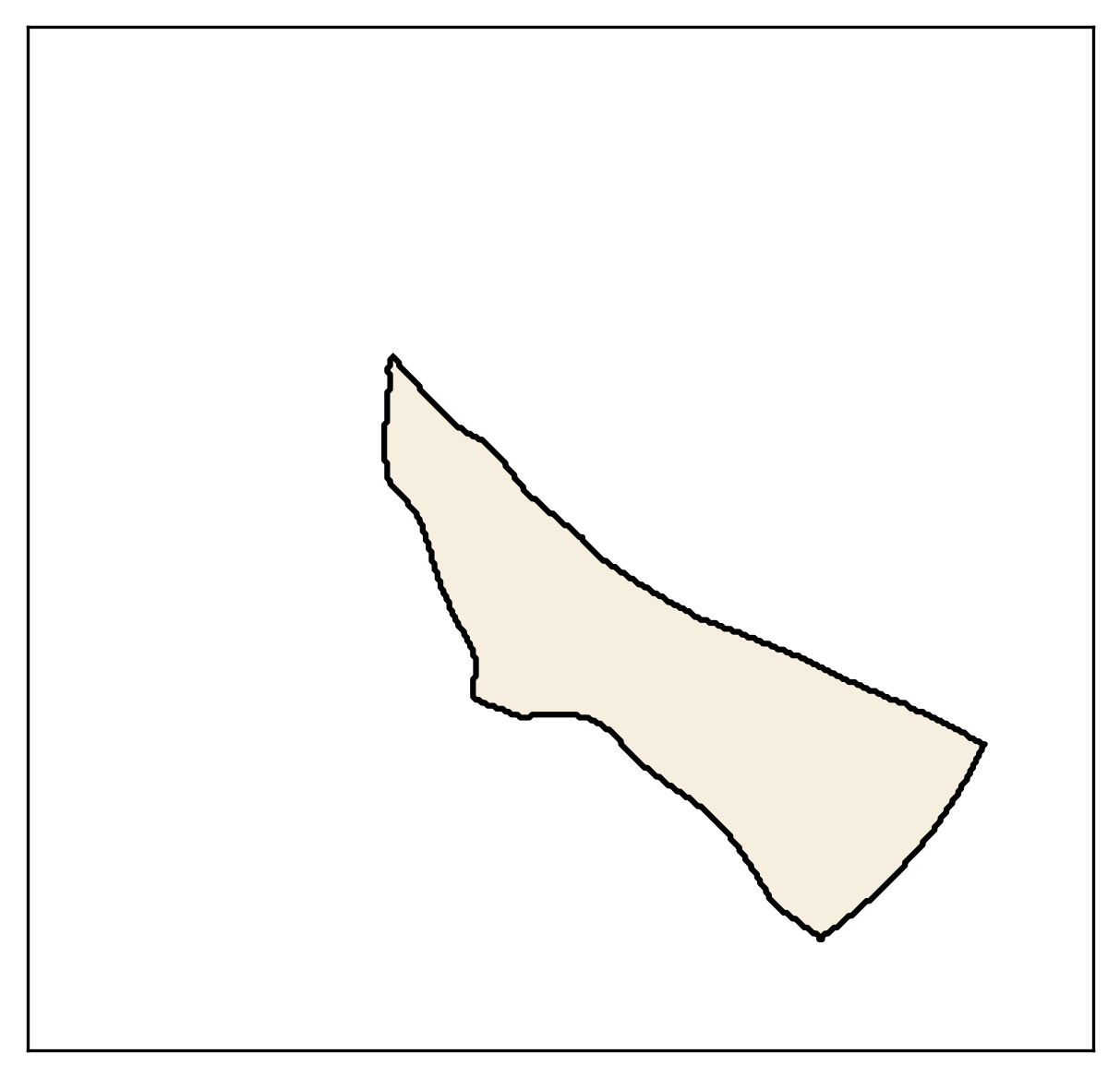} &
\includegraphics[width=0.12\textwidth]{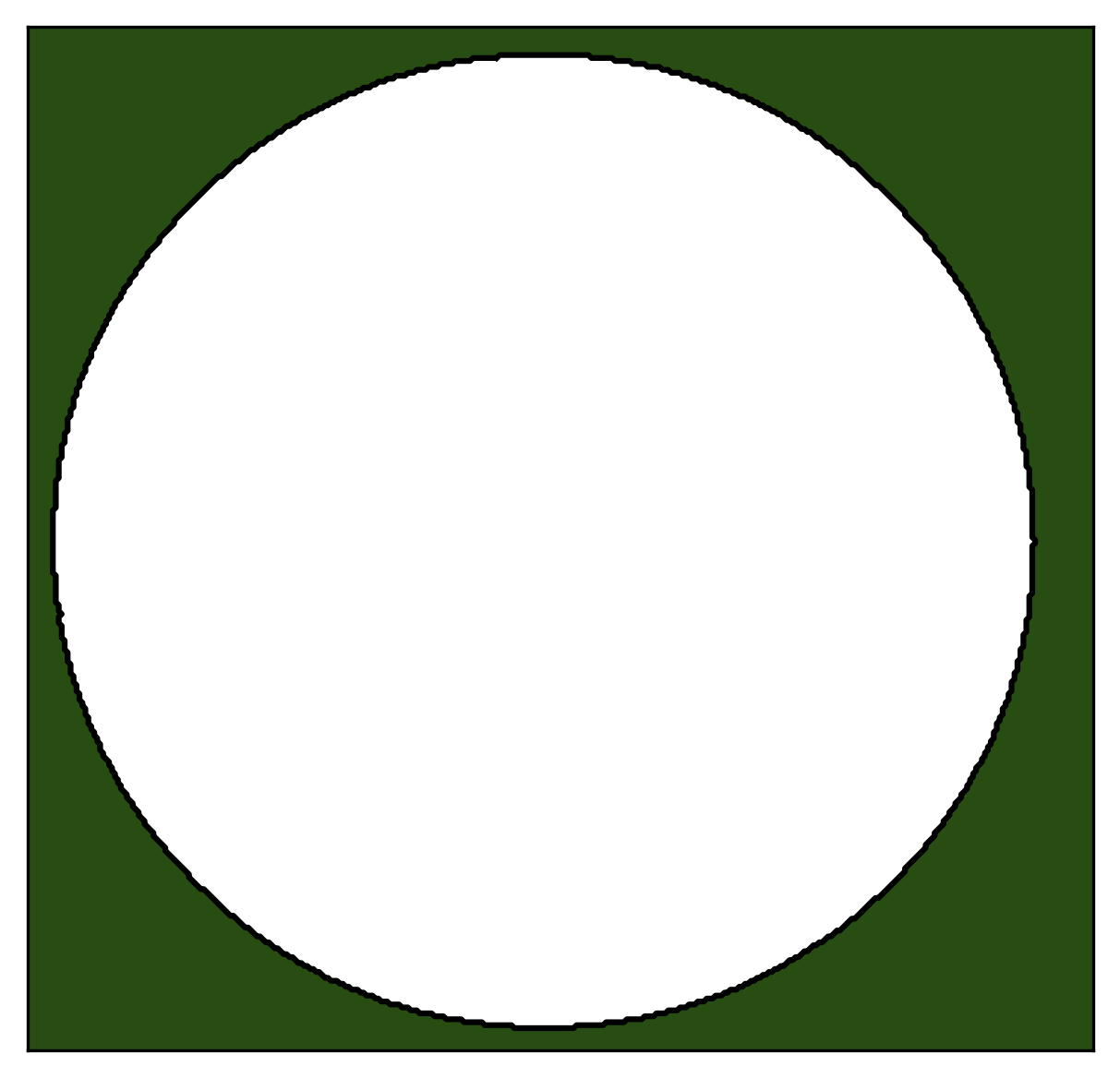}
\end{tabular}
\caption{[Shape Layer $S_i$] (a) The given color quantized raster image $f$.  (b) Seven shape layers $S_i$ $i=1,2,\dots, 7$ are colored by their associated colors $c_l$, with $l=1,2, \dots,5$ (black, orange, yellow, white and green). }\label{fig:shapelayer} 
\end{figure}

\subsection{Depth ordering energy for pairs of shape layers} 
\label{subsec: depth ordering}

In order to determine the depth ordering between any two shape layers from $\mathcal{S}$, we first follow studies of human vision perception to build simple rules to give depth ordering. In  \cite{kelly2000perception}, authors explored how human perception incline to straightening occluded objects based on FACADE model~\cite{grossberg1994perception}. In~\cite{rock1990legacy}, Rock et al. points to Pr\"{a}gnanz's idea of how human usually perceive simpler, smoother and more convex shape behind another when the depth ordering is ambiguous. In~\cite{vision2010}, convex prior in visual perception is also discussed. We give the following assumptions built up on these simple humen perception rules: 
\begin{enumerate}[label = \textbf{A\arabic*}]
\item{\label{assumption: convex object} Objects tend to be convex. }
\item{\label{assumption: top}     
Objects with less occluded region tend to be on top.}
\item{\label{assumption: smooth}
Object boundaries tend to be smooth, i.e., tangential directions on the boundaries change smoothly. }
\end{enumerate}

We introduce a pairwise area measure, which estimates occluded regions between each pair of shape layers $S_{i}$ and $S_{j}$. This is based on our convexity assumption \ref{assumption: convex object}, and \ref{assumption: top} assuming smaller occluded objects are on top.  
\begin{definition}\label{def: covered area measure }
Let $\chi_i: \Omega \to \{0, 1 \}$ be the characteristic function of $S_i,$ and $\chi^{\mathrm{conv}}_{i}$ denote characteristic function of the convex hull of $S_{i}$.  
For two distinct shape layers $S_i$ and $S_j$, we define \textbf{covered area measure} of $S_{j}$ by $S_{i}$ as
    \begin{equation}\label{eq:covered area measure}
        A(i, j) = \frac{\int_{\Omega} \chi^{\mathrm{Conv}}_{j} \chi_{i} \mathrm{d}x}{\int_{\Omega} \chi_{i} \mathrm{d}x}.
    \end{equation}
\end{definition}
This covered area measure $A(i, j)$ approximates the area of $S_j$ possibly occluded by $S_i$, by finding the area of convex hull of $S_j$ intersecting $S_i$.  Comparing against the total area of $S_i$, this ratio shows how much  $S_i$ is occluding shape $S_j$. When $A(i, j)$ is small, this implies shape $S_j$ almost has no overlap with $S_i$, since convex hull of $S_j$ is barely intersecting with $S_i$, and $S_i$ is minimally covering $S_j$. When $A(i, j)$ is close to 1, $S_i$ lies completely inside the convex hull of $S_j$. This covered area measure shows how much of $S_i$ is covering $S_j$, but does not considers how much $S_j$ is occluding $S_i$: it is non-commutative, i.e. $A(i,j) \neq A(j,i)$ in general. 
To properly determine the depth ordering between $S_i$ and $S_j$, we compare two covered area measures between a pair of shape layers.

\begin{definition}
We define  \textbf{depth ordering energy}  between two adjacent shape layers $S_i$ and $S_j$ to be 
\begin{equation}\label{eq: ordering}
D(i, j) = A(i,j) - A(j,i) = 
\frac{\int_{\Omega} \chi^{\mathrm{Conv}}_{j}\chi_{i} \mathrm{d}x}{\int_{\Omega} \chi_{i} \mathrm{d}x} -
\frac{\int_{\Omega} \chi^{\mathrm{Conv}}_{i}\chi_{j} \mathrm{d}x}{\int_{\Omega} \chi_{j} \mathrm{d}x}
\end{equation}
\end{definition}
If $D(i, j) > 0$, this implies $A(i,j)$ is bigger than $A(j,i)$, that shape layer $S_i$ covers $S_j$ more according $S_i$'s own size. This means more portion of $S_i$'s area is in front of $S_j$. This measure is independent of the size of $S_i$ and $S_j$, that even if the area of $S_i$ is small compared to that of $S_j$, if more portion of $S_i$ is covering $S_j$, it is still determined to be in front of $S_j$. If $D(i, j) < 0$, $S_j$ is determined to be in front of $S_i$. This measure expresses the assumption \ref{assumption: top}. 
To address numerical error and small perturbations in practice, we allow a small variation and use the following:
\begin{equation}\label{eq: ordering result}
     \begin{cases}
    D(i,j)> \delta & \implies \text{$S_i$ is in front of  $S_j$, and set } D(i,j)> 0 \\
    D(i,j)< -\delta & \implies \text{$S_j$ is in front of  $S_i$, and set } D(i,j)< 0 \\
    D(i,j) \in [-\delta, \delta] & \implies \text{$S_i$ and $S_j$ are on the same depth level, and set } D(i,j) =0. \\
    \end{cases}
\end{equation}
We use $\delta \in [0.01, 0.1]$ depending on how refine one wants the depth ordering to be, and  how many objects are in the input image (see Section \ref{sec: experiments}). 
Figure \ref{fig: depth conv} illustrates depth ordering energy using the shape layers in Figure \ref{fig:shapelayer}.  Figure \ref{fig: depth conv} (a) shows the orange sun $S_3$ and (b) shows the light yellow sky $S_4$ from the input image.  The red closed contours in (a) and (b) show the convex hull of each shapes, which are presented as yellow regions in (c) and (d).  (c) and (d) describe $A(3,4)$ and $A(4,3)$ that green area over green and blue areas gives the ratio.   $A(3,4)$ in (c) is a positive value (close to 1), while  $A(4,3)$ in (d) is almost zero, thus the depth ordering energy $D(3,4)>0$, and $S_3$ is above $S_4$.
\begin{figure}
\centering
\begin{tabular}{ccccc}
(a) $S_3$ & (b) $S_4$ & (c) $A(3,4)$ & (d) $A(4,3)$\\
\includegraphics[width=3.6cm]{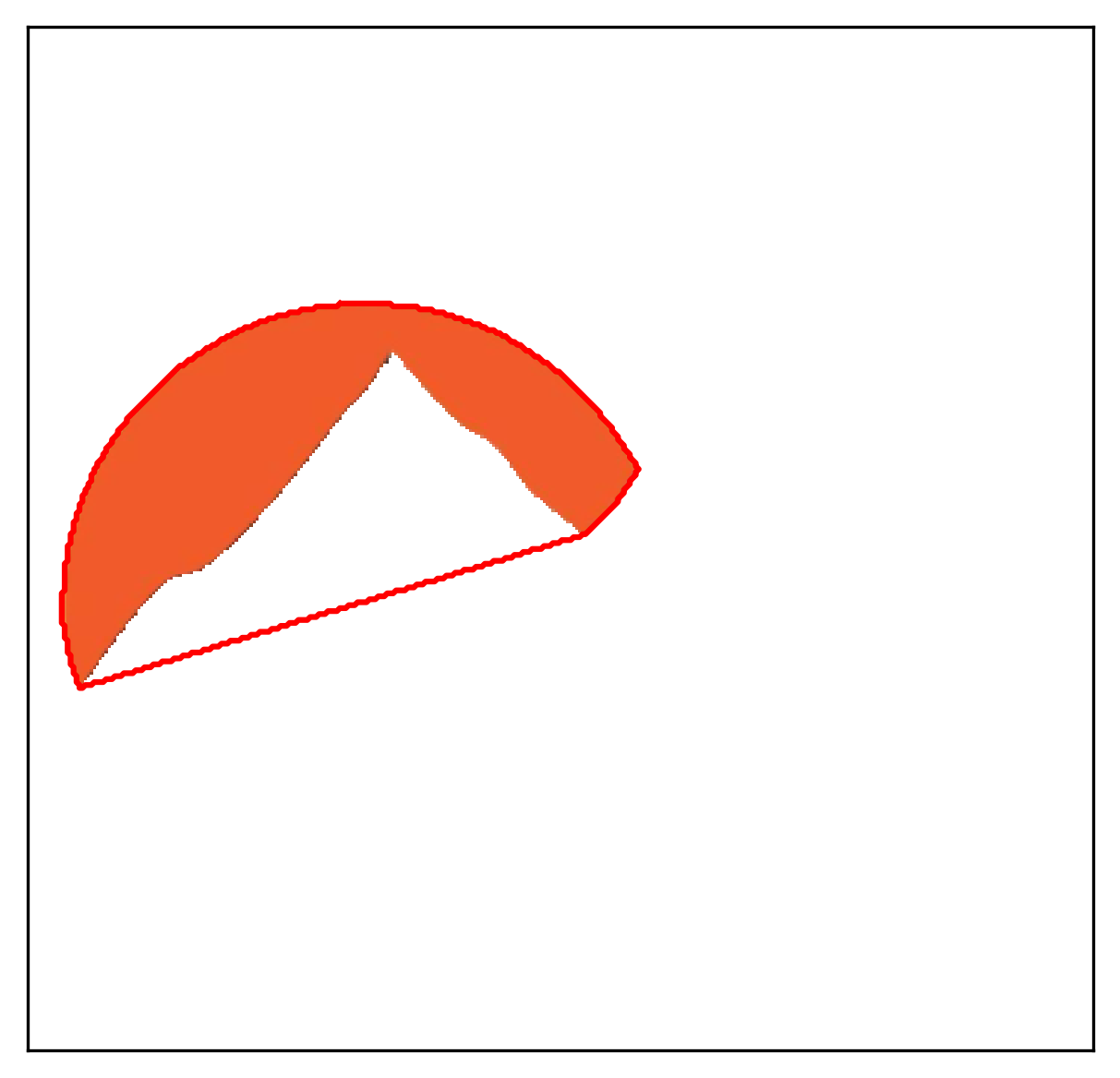} &
\includegraphics[width=3.6cm]{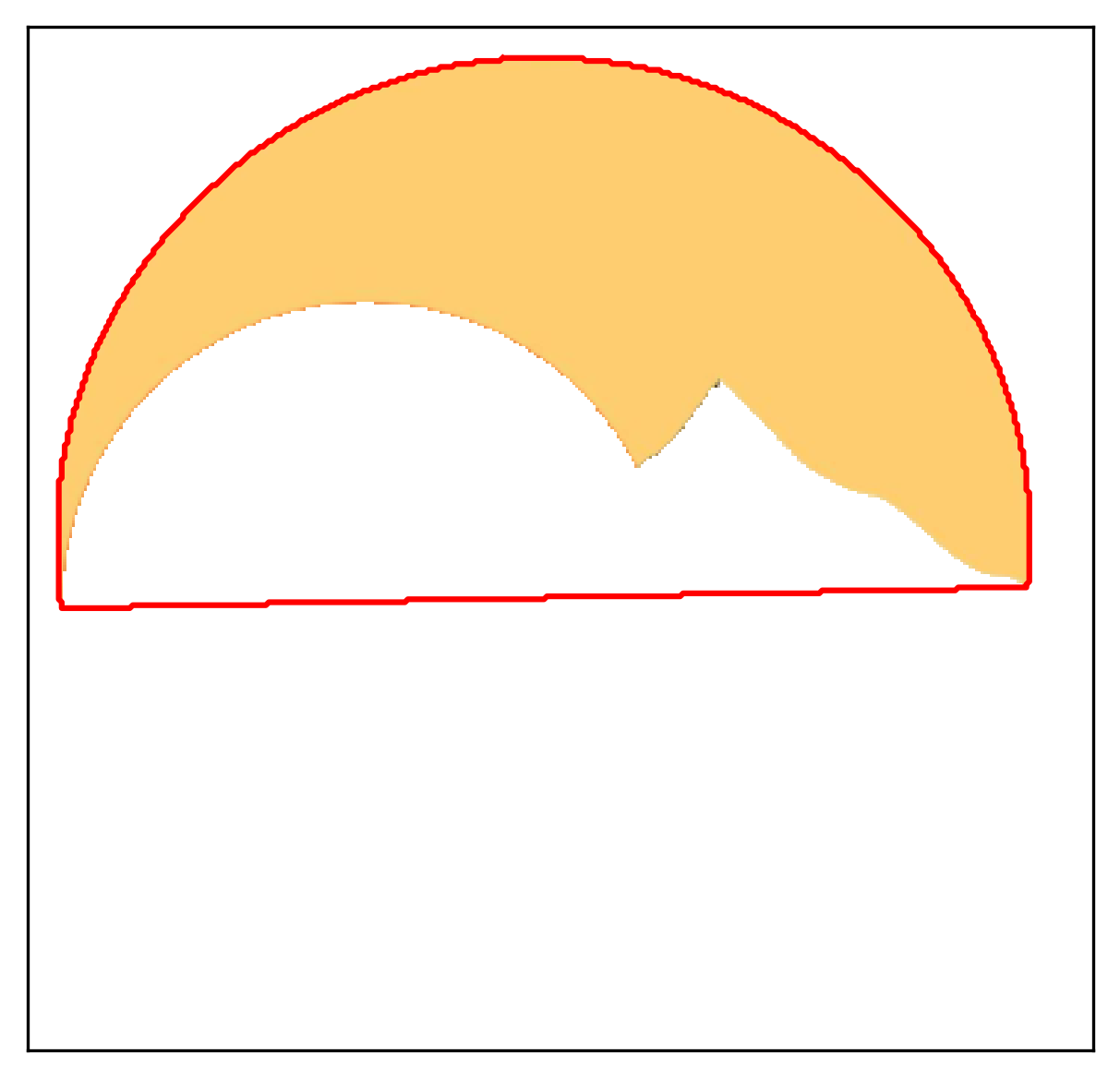}&
\includegraphics[width=3.6cm]{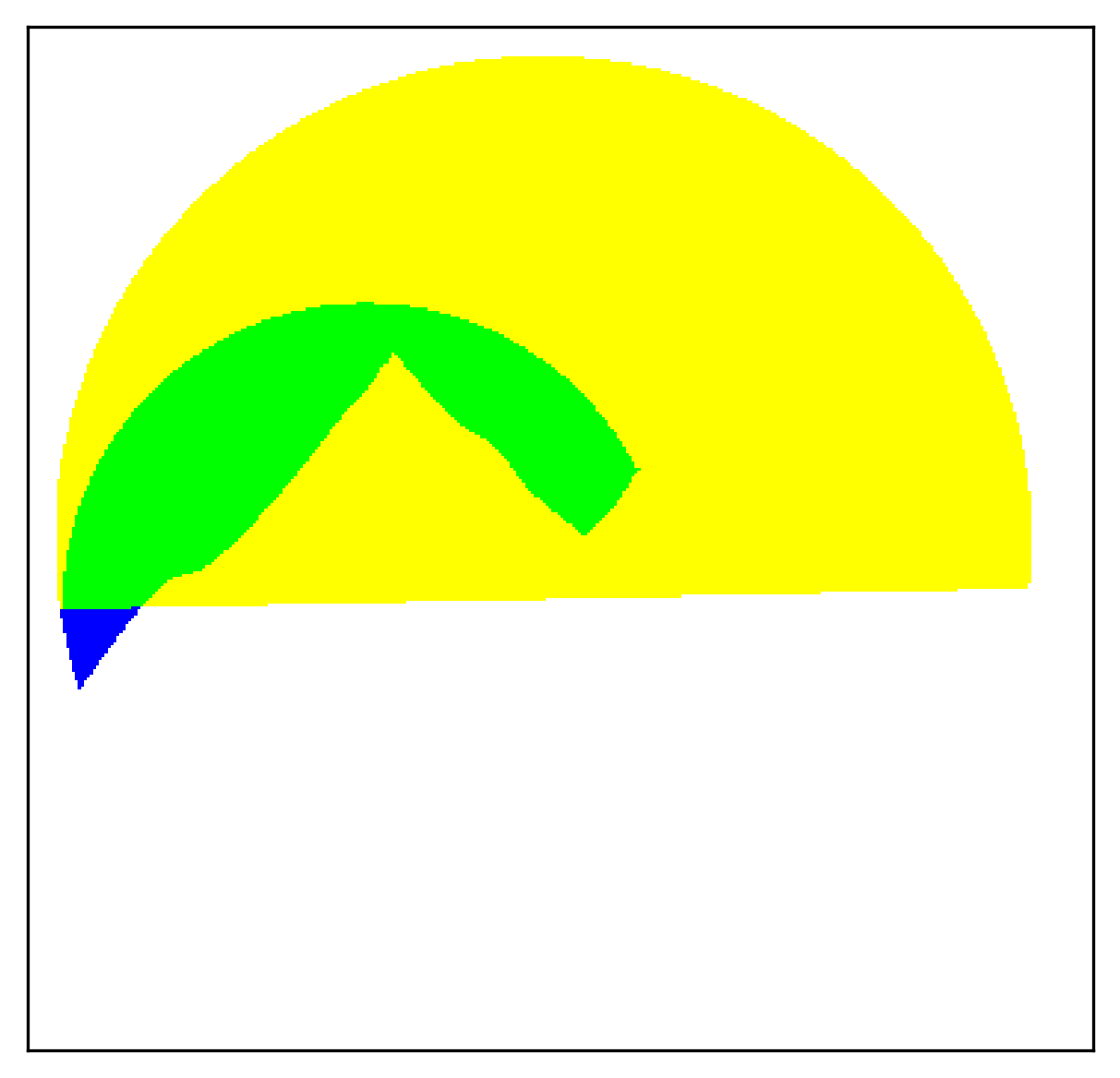} &
\includegraphics[width=3.6cm]{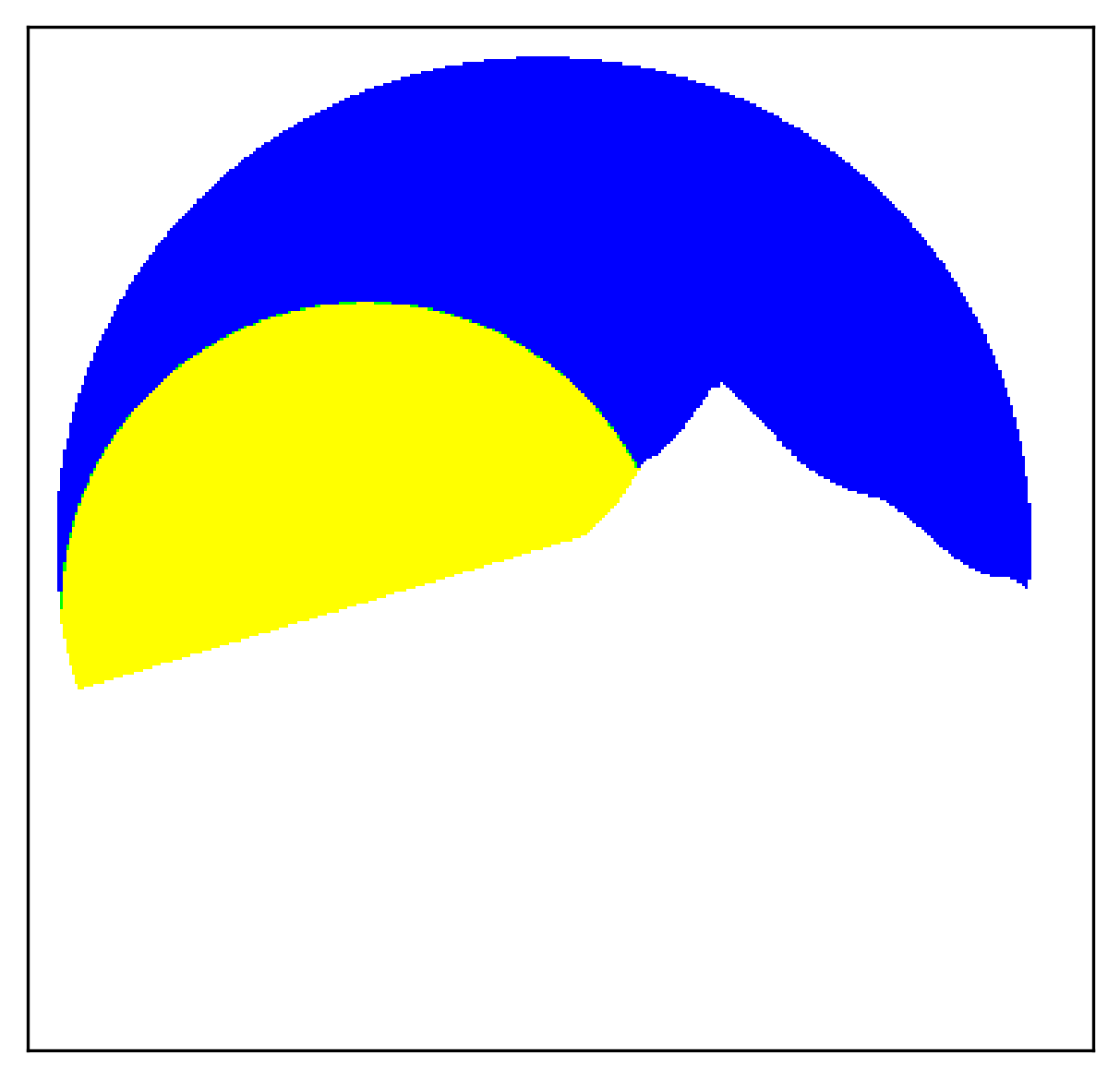}
\end{tabular}
\caption{[Depth ordering $D(i,j)$] Consider $f$ from Figure \ref{fig:shapelayer}(a), (a) shows the orange sun shape layer $S_3$, and (b) shows the light yellow sky $S_4$. The red closed contours in (a) and (b), and yellow regions in (c) and (d) show the convex hull of each shapes.  In (c) and (d), green area represents the numerators of $A(3,4)$ and $A(4,3)$ respectively.  (c) $A(3,4)$ is close to 1, while  $A(4,3)$ in (d) is near zero, thus $D(3,4)>0$, and $S_3$ is determined to be in front of $S_4$.  }
\label{fig: depth conv}
\end{figure}

We found that the tangential direction computation or concavity computation, especially for small images, to be unstable and quite noisy in many cases for raster images. 
The propose covered area measure as well as the depth ordering energy using area comparison give more stable results. 
We use convex hull for depth ordering measure for faster and simpler computation. However, to reconstruct convexified shape layers, we use Euler's elastica curvature based model to satisfy the assumption \ref{assumption: smooth}.   We analyze this difference in Section \ref{sec:analysis}.

\subsection{Global depth ordering via a directed graph}\label{subsec: depth graph}

To determine the global depth ordering, we build a directed graph $G(M,E)$ using the pairwise depth ordering energy $D(i,j)$ in \eqref{eq: ordering}, where $M$ is the set of all shape layers $S_i$ as nodes and $E$ is the set of directed edges $E_{i,j}$ with direction determined by the sign of $D(i,j)$. A directed edge $E_{i,j}>0$ from node $i$ to $j$, indicates that shape layer $S_{i}$ is above $S_{j}$, we also denote as $S_{i} \rightarrow S_{j}$. Every shape layer is compared to every other one, and there is no edge between two nodes if they are identified on the same depth level. This graph helps to find a linear global ordering of all shape layers when it is acyclic, after performing topological sort~\cite{cormen01introduction}.

In real images, there is a large number of shape layers even after $K$-mean clustering of colors, i.e. $N_{\mathcal{S}}$ is large.  In such cases, directed cycles can be common in this directed graph.  Each cycle implies shapes are on top of each other in a loop, which is non-physical. We propose the following energy to break these cycles. Let the set of nodes in a cycle be $M_c \subset M$.   If there are multiple cycles, we consider each cycle separately as $N_{c_{1}}$, $N_{c_{2}}$, \dots $N_{c_{n}}$, and break only one edge from each of them. 
\begin{definition}
We define \textbf{convex hull symmetric difference} $V(i,j)$ for each $i, j \in M_c$ as 
\begin{equation}\label{eq: CSdifference}
V(i,j) := \int_{\Omega} \chi_{i} + \chi_{j} - \chi^{\mathrm{conv}}_{j}\chi_{i} \mathrm{d}x.
\end{equation}
This is a symmetric difference for sets $\chi^{\mathrm{conv}}_{j}$, the convex hull of $S_j$, and $\chi_{i}$.  This $V(i,j)$ is not commutative. 
\end{definition}

The main motivation of this convex hull symmetric difference is to remove the edge which is the least noticeable.  For a cycle with length $m$, $i_1 \rightarrow i_2 \rightarrow \cdots \rightarrow i_m  \rightarrow 
 i_1$, we compute all $V(i_a, i_b)$, for $a<b$ and $a, b \in M_c$ and $V(i_m, i_1)$, then find the maximum $V$ to break the cycle: Find the edge with 
 \begin{equation} \label{eq: edge largest CS difference}
    (i^{\ast}, j^{\ast}) = \arg \max_{(i,j) \in M_c} V(i,j) 
 \end{equation}
and remove the edge weight by setting $E^\ast_{i^\ast,j^\ast} =0$. Once $E^\ast_{i^{\ast},j^{\ast}}$ is set to zero, $S_{i^\ast}$ is no longer on top of $S_{j^{\ast}}$, and there is no cycle. The graph $G$ is reduced to a linear ordering with $S_{j^{\ast}}$ being the source node and $S_{i^{\ast}}$ the sink.  This shape layer $S_{j^\ast}$ is violating the convexity assumption \ref{assumption: convex object}, since it may be occluded by $S_{i^{\ast}}$; yet within the cycle, this represents the least area of convexity violation, i.e. least noticeable to remove this edge. 

This is illustrated in Figure \ref{fig: cyclic}, with three overlapping disks in (a). The depth ordering energy (\ref{eq: ordering result}) gives a cycle in (b) that the red shape layer $S_1$ is above the green shape layer $S_2$ ($E_{1,2}>0$),  the green shape layer $S_2$ above the blue one $S_3$ ($E_{2,3}>0$), and the blue above the red ($E_{3,1}>0$).  (c) shows possible occluded regions $\chi_j^{conv}\chi_i$ in different colors: The yellow region is the convex hull of the green shape layer over the red shape layer $\chi_2^{conv}\chi_1$, thus $V(1,2)$ is the union of red and green shape layers minus the yellow.  The cyan region is $\chi_3^{conv}\chi_2$ convex hull of the blue shape $S_3$ layer over the blue shape layer $S_2$, and the 
magenta region is $\chi_1^{conv}\chi_3$.    Assuming areas of the three circles are similar, red and green shape layer have the smallest area (the yellow region) being subtracted, thus $V(1,2)$ in  \eqref{eq: CSdifference} is the biggest among $V(1,2)$, $V(2,3)$ and $V(3,1)$, and the edge $E_{1,2}$ is set to zero.   As a result, shown in (d), the global ordering is given as $S_2$ green shape layer above $S_3$ blue shape layer, which is above $S_1$ red shape layer.   
\begin{figure}
\centering
\begin{tabular}{cccc}
(a) & (b) & (c) & (d) \\
\includegraphics[width=4cm]{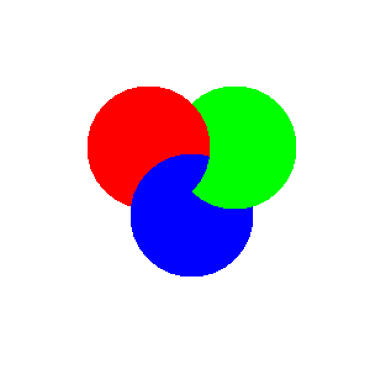}&
\raisebox{1cm}{\begin{tikzpicture}
\draw
(0.0:0) node (green) {$S_2$ Green }
(-120.0:2) node (blue) {$S_3$ Blue }
(-180.0:2) node (red){$S_1$ Red };
 \begin{scope}[->]
 \draw (green) to (blue);
 \draw (blue) to (red);
\draw (red) to (green);
\end{scope}
\end{tikzpicture}
}&
\includegraphics[width=4cm]{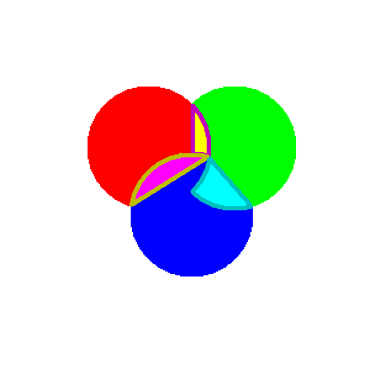}
&
\raisebox{1cm}{
\begin{tikzpicture}
\draw
(0.0:0) node (green) {$S_2$ Green }
(-120.0:2) node (blue) {$S_3$ Blue }
(-180.0:2) node (red){$S_1$ Red };
 \begin{scope}[->]
 \draw (green) to (blue);
 \draw (blue) to (red);
\end{scope}
\end{tikzpicture}
}
\end{tabular}    \vspace{-1cm}
\caption{[The graph $G(M,E)$ with a cyclic.]  The given image $f$ in (a) gives a directed graph in (b).  Considering the convex hull symmetric differences $V(1,2)$ (yellow is $\chi_2^{conv}\chi_1$), $V(2,3)$ (cyan is $\chi_3^{conv}\chi_2$) and $V(3,1)$ (magenta is $\chi_1^{conv}\chi_2$), $V(1,2)$ is the maximum and we set $E_{1,2}=0$.  (d) shows the linear directed graphs which gives the global depth ordering. }
\label{fig: cyclic}
\end{figure}
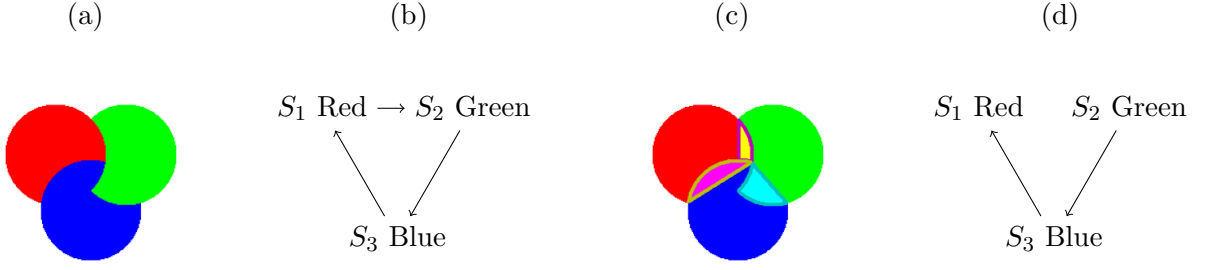

In implementation, we find all cycles in the directed graph, and perform this action until no more cycles are found. Once this directed acyclic graph is obtained, we use topological sort to find a linear depth ordering of shapes, details are presented in Section \ref{subsec: convexification and topological sort}.  The final ordering result is hereafter denoted as a permutation $\mathcal{D}: \mathbb{Z}_{N_\mathcal{S}} \to \mathbb{Z}_{N_\mathcal{S}}$, where $N_\mathcal{S}$ is the total number of shape layers. The full algorithm is described in Algorithm \ref{alg: depth ordering}.

\textbf{Remark}: The effect of using the maximum of convex hull symmetric difference $V(i,j)$ differs from possibly re-using the minimum value of $D(i,j)$ for removing a cycle.  When $|D(i,j)|$ is small, i.e., $A(i,j) \approx A(j,i)$, the difference in the ratio of occluded areas is small regardless of the size of the shapes.  Even if one shape is very large and another is very small, this $|D(i,j)|$ does not distinguish them.  If a cycle exists, this already represents inconsistency of comparing $A(i,j)$'s,  following an example such as Figure \ref{fig: cyclic}, we propose $V(i,j)$ to ignore smallest possible inpainting region.

\subsection{Euler's elastica curvature-based inpainting for the occluded region}\label{subsec: shapes convexifiation}

Once we have the full depth ordering for all shape layers $\mathcal{D}: \mathbb{Z}_{N_\mathcal{S}} \to \mathbb{Z}_{N_\mathcal{S}}$, where $N_\mathcal{S}$ is the total number of shape layers, we convexify each shape layer by the Euler's elastica curvature-based inpainting model considering the occluded regions given by the depth ordering.  Inpainting shape layers is not only for reducing the possibility of forming gaps between two adjacent shapes, but also aiding possible post-vectorization edit process. 
We allow each shape layer $S_i$ to extend following the curvature direction, as our assumption \ref{assumption: smooth}, as long as it is covered by shapes on top of the current layer $S_i$. 
The region  $O_i$ is the occluded region where the inpainting is allowed (i.e. inpaintable domain) for $S_i$ and is defined as follows:
\begin{definition}\label{def: regularization region}
Let the given depth ordering be $\mathcal{D}: \mathbb{Z}_{N_\mathcal{S}} \to \mathbb{Z}_{N_\mathcal{S}}$, where the smaller number represents the shape being on the top, closer to observer.
We define the \textbf{shape-covered region} of shape layer $S_{i}$ to be 
\begin{equation}\label{eq:O_i}
    O_i = \left( \bigcup_{j : \mathcal{D}(j) \leq \mathcal{D}(i)} S_{j} \right) \cup S_{noise}.
\end{equation}
This is the union of all shape layers $S_{j}$ that are on top of shape $S_{i}$, including $S_{i}$, and the noise layer will be defined later in  \eqref{eq:noiselayer}.
\end{definition}

We find the optimal region, the \textbf{inpainted shape layer} of $S_i$, as $C_i \subset \Omega$ by minimizing the Euler's elastica energy within the occluded region $O_i$:
\begin{equation}\label{eq: curvature based energy}
     E(C_i) = \int_{\partial C_i} \left( a + b \kappa^{2} \right) \mathrm{d}s \quad \text{ such that }  \quad S_{i} \subseteq C_i \subseteq O_i
\end{equation}
where $a, b$ are some positive constants, $\partial C_i$ denotes the boundary of $C_i$, and $\kappa$ is the curvature of the boundary.  Using the shape-covered region $O_i$ constraint in  (\ref{eq: curvature based energy}) ensures the final collection of shape layers to be close to the raster image, while curvature-based inpainting inpaints and regularizes each shape's boundary. The inpainted region $C_i$ is not shown from the top, since it is occluded by shape layers above $S_i$. 
We use phase transition function $u_i$ to find the inpainted shape $C_i$, and detail of the modified model and implementation of  \eqref{eq: curvature based energy} is discussed in Section \ref{sec:elastica}.

\subsection{Vectorization: \bezier curve fitting }\label{subsec: bezier curve fitting}

The convexified shape layer $C_i$ is represented by a phase transition functions $u_{i}: \Omega \to [-1, 1]$. By considering the characteristic function of $u_{i} \geq 0$, i.e., $\chi(C_i)$, this becomes equivalent to silhouette image vectorization in \cite{he2021silhouette}.  We briefly outline this process here.  

We extract the phase transition  $\lambda_{i} = \{ x \mid u^{\ast}_{i} = 0\}$ as a set of discrete points.  We pick the curvature extrema \cite{morel2017} from this boundary to capture the geometry of boundary accurately. 
Given $C_i$, oriented clockwise and discretized as a set of points $\{ p_{k} \}_{k=1}^{K}$, for some small positive integer $h$, the curvature at $p_{k}$ is given by 
\begin{equation}\label{eq: discretized curvature}
    \kappa(p_{k}) = \frac{-2 \mathrm{det}\begin{pmatrix} \overrightarrow{p_{k}p_{k - h}} & \overrightarrow{p_{k}p_{k+h}} \end{pmatrix}}{\|\overrightarrow{p_{k}p_{k-h}}\| \|\overrightarrow{p_{k}p_{k+h}}\|\|\overrightarrow{p_{k+h}p_{k-h}}\|}. 
\end{equation}
Since the set of points $\{ p_{k}\}_{k=0}^{K-1}$ are sampled from a shape layer which has closed boundary curve $C_{i}$, $k-h$ and $k+h$ are computed modulo $K$. We use the notation $\overrightarrow{p_{i}p_{j}}$ to denote the vector from $p_{i}$ to $p_{j}$.  We identify local extrema if the curvature in  \eqref{eq: discretized curvature} is larger than a threshold $\mathcal{T}$.  In case no curvature extreme is found, we randomly choose a point to be both the starting and ending point, and consider  the boundary as a single segment. 

We find the cubic \bezier curves fitting these points: 
Given four vectors $\mathbf{P}_{0}, \mathbf{P}_{1},\mathbf{P}_{2},\mathbf{P}_{3}$ in $\mathbb{R}^{2}$, a cubic \bezier curve $\mathbf{B}(t): [0,1] \to \mathbb{R}^{2}$ can be defined as:
\[ 
        \mathbf{B}(t) = (1-t)^{3}\mathbf{P}_{0} + 3(1-t)^{2}t\mathbf{P}_{1} + 3(1-t)t^{2}\mathbf{P}_{2} + t^{3} \mathbf{P}_{3}
\]
Since each cubic \bezier curve is determined by four vectors, this process is commonly called vectorization.
We partition the boundary of inpainted shape layer $C_i$ into segments $\{ \mathcal{P}_{ij} \}_{j = 1}^{\ell_{i}}$, which all begin and end points $\mathcal{P}_{i}$ are at local curvature extrema $\{ p_{ijq}\}_{q = 1}^{Q_{j}}$ for fixed $j$ and $i$. 
For each $\mathcal{P}_{ij}$, we solve a least square problem to fit cubic \bezier curves to a given set of points with orientation as in \cite{schneider1990algorithm}:
\begin{equation}\label{eq: wls fitting}
    \min_{\mathbf{P}_{0}, \mathbf{P}_{1},\mathbf{P}_{2},\mathbf{P}_{3}}\sum_{q = 1}^{Q_{j}} \left\| \sum_{i=0}^{3} \binom{3}{i} (1 - t)^{3 - i} t^{i} \mathbf{P}_{i} - p_{ijq} \right\|^{2}_{2} .
\end{equation}
If the Hausdorff distance between the fitted \bezier curve and the points is too large, we recursively partition $\mathcal{P}_{ij}$ at the point that gives the greatest error, and solve the above least-squares problem \eqref{eq: wls fitting} until the distance is smaller than a prescribed tolerance  \cite{he2021silhouette,schneider1990algorithm}.  We refer to this fitting parameter as $\tau$.

After each shape layer is represented by \bezier curves, following the depth ordering, we write them to a SVG file, starting from the bottom layer to the top layer, in reserve depth order, to superpose the shape layers.

\subsection{Outline of the proposed method}
\begin{figure}[th]
    \centering
\begin{tikzpicture}[node distance=5.3cm, every edge/.style={draw=black, ->}]
  \node (quantized) [label=below: Color Quantized Input $f$] {\includegraphics[width=3.2cm, height=3.2cm]{pic/mountain_example/quantized_mountain.png}};
  \node (stack) [right of = quantized, xshift = 0.8cm] {\begin{tabular}{c@{\hskip .1cm}c@{\hskip .1cm}c}
      \includegraphics[width = 1.2cm]{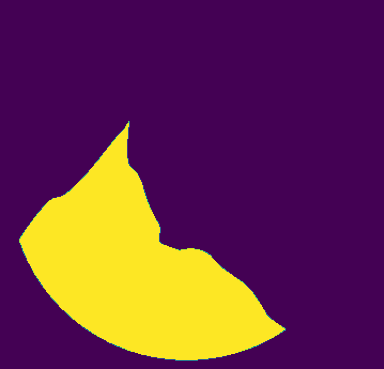} &  \includegraphics[width = 1.2cm]{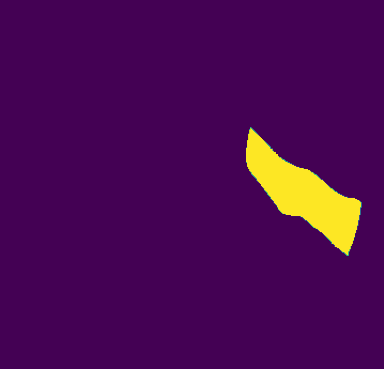} & \includegraphics[width = 1.2cm]{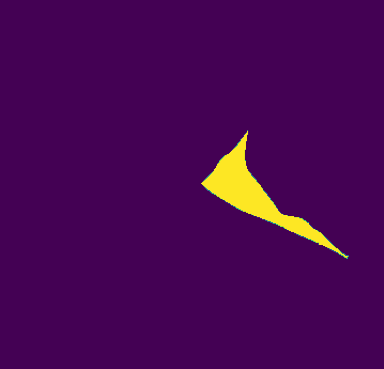} \\
      \includegraphics[width = 1.2cm]{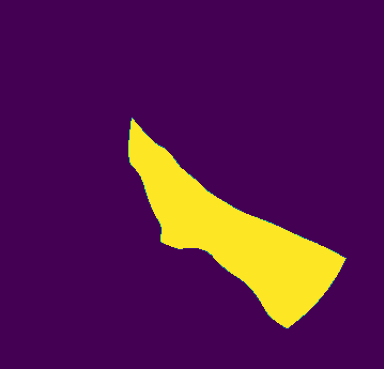} &  \includegraphics[width = 1.2cm]{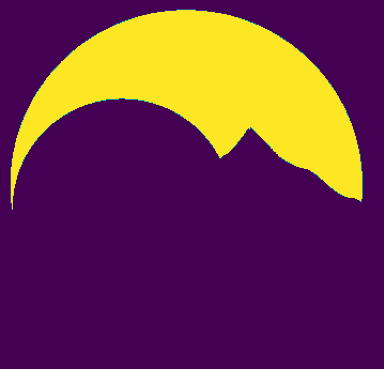} & \includegraphics[width = 1.2cm]{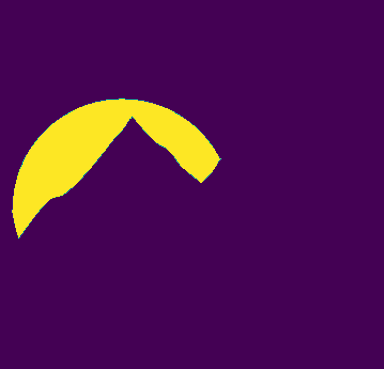}  \\
        & \includegraphics[width = 1.2cm]{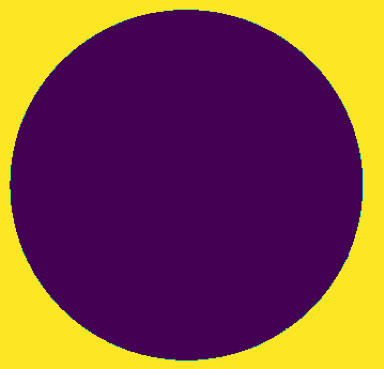} & \\
  \end{tabular}};
  \foreach \X [count=\Z]in {level_3.png, level_5.png,level_6.png,level_0.png,level_4.png,level_2.png, level_1.png}
  \node (ordered stack) [right of = stack, xshift = .5cm, yshift = .1cm] at (7.5,1.2, \Z /2) { \includegraphics[width = 2.5cm, height = 2.5cm]{pic/mountain_example/layers/\X} };
  \node (inpainted stack) [below of = ordered stack, xshift =0cm, yshift = 1.6cm] at (12.2,-1){
  \begin{tabular}{c@{\hskip .1cm}c@{\hskip .1cm}c}
      \includegraphics[width=1.2cm]{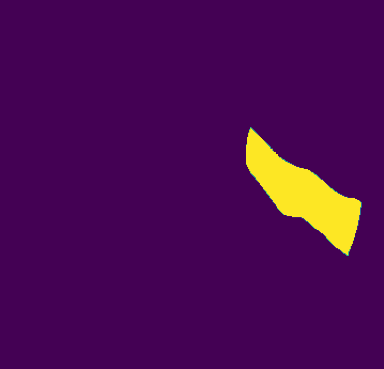} & \includegraphics[width=1.2cm]{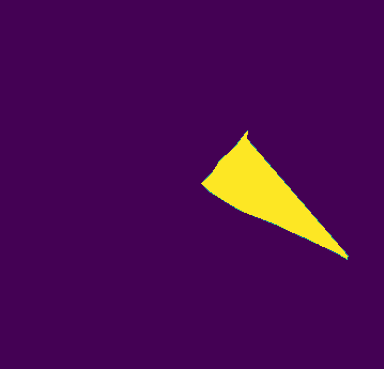} & \includegraphics[width=1.2cm]{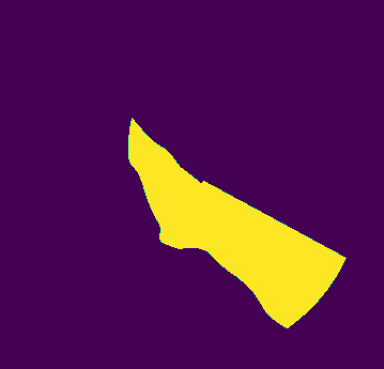} \\
      \includegraphics[width=1.2cm]{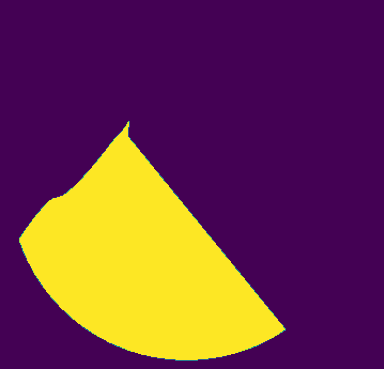} & \includegraphics[width=1.2cm]{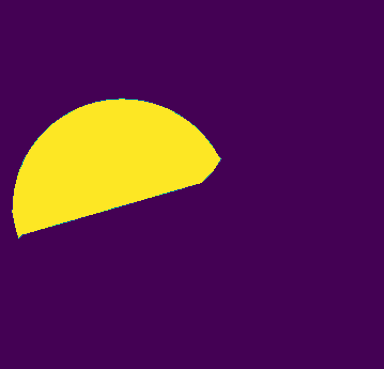} & \includegraphics[width=1.2cm]{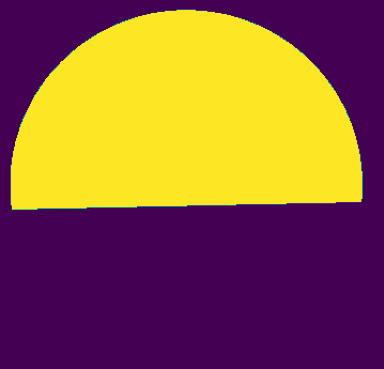} \\
      & \includegraphics[width=1.2cm]{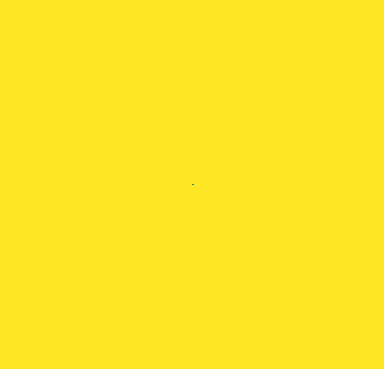} & \\
  \end{tabular}
  };
  \node (bezier stack) [left of = inpainted stack, xshift=-.8cm] {
  \begin{tabular}{c@{\hskip .1cm}c@{\hskip .1cm}c}
      \includegraphics[width = 1.2cm]{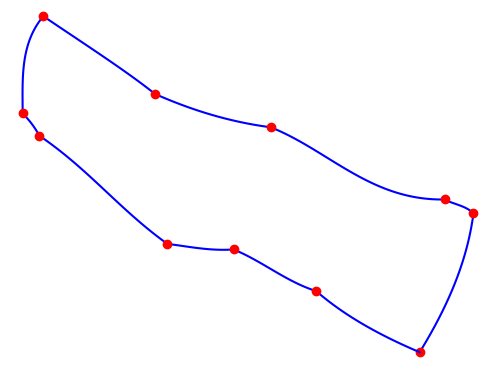} & \includegraphics[width = 1.2cm]{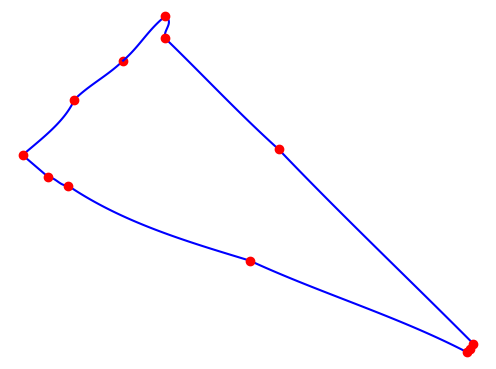} & \includegraphics[width = 1.2cm]{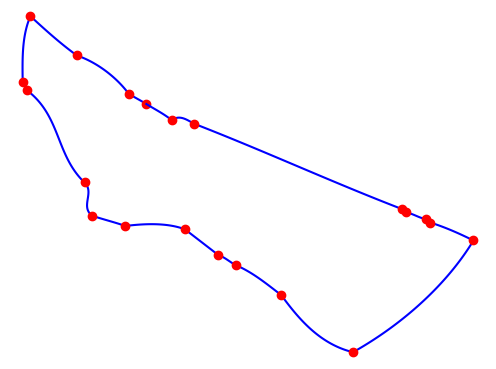}\\
      \includegraphics[width = 1.2cm]{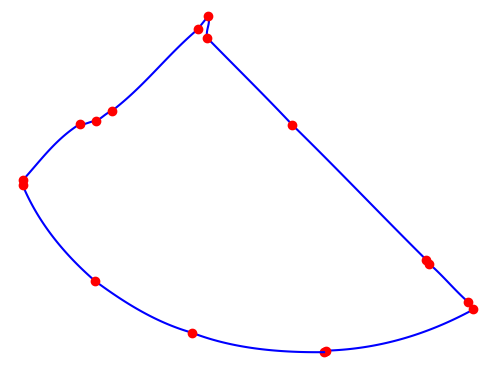} & \includegraphics[width = 1.2cm]{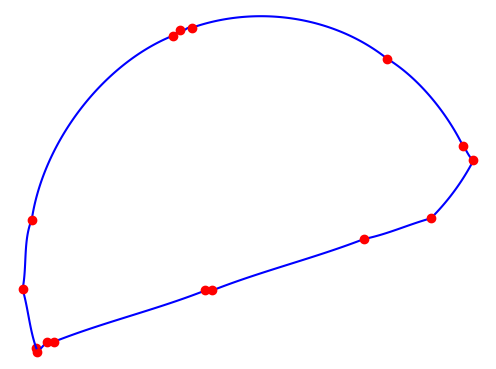} & \includegraphics[width = 1.2cm]{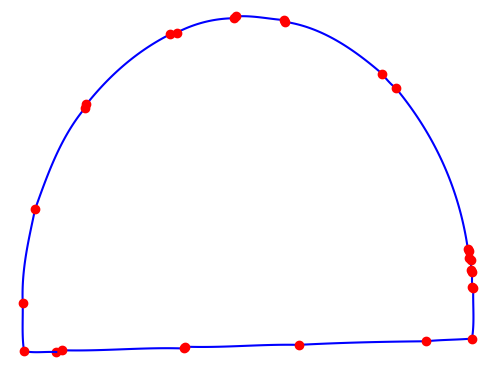}\\
  \end{tabular}  
  };
  \node (output image) [left of = bezier stack, label=below: SVG Output, xshift=-0.8cm]
  {\includegraphics[width = 3.2cm, height = 3.2cm]{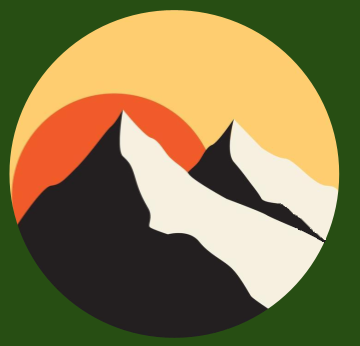}};
  \draw [->] (quantized) -- node[anchor=south] {Shape layers}(stack);
  \draw [->] (stack) -- node[anchor=south, xshift = -.1cm] {Depth ordering}(ordered stack);
  \draw [->] (ordered stack) -- node[anchor=east] {Curvature-based inpainting} (inpainted stack);
  \draw [->] (inpainted stack) -- node[anchor = south]{Curve fitting}
  (bezier stack);
  \draw [->] (bezier stack) -- node[anchor = south] {Vectorize} 
  (output image);
\end{tikzpicture}
\caption{[Image Vectorization with depth flowchart] From the given color quantized image $f$, shape layers $S_i$s are defined, and depth ordering is determined. Euler's Elastica curvature-based inpainting is used to convexify shape layers considering the occluded region $O_i$ given by the depth ordering.  Finally, each convexified layers $C_i$s are vectorized and stacked in SVG file format. }
\label{fig: flow chart}
\end{figure}

The outline of the proposed image vectorization with depth is illustrated in Figure \ref{fig: flow chart}.  From the color quantized input image $f$, shape layers are formed based on colors and connectedness in $\Omega$.  We consider the depth ordering energy  \eqref{eq: ordering} between two adjacent shapes to build a directed graph of depth ordering among all shapes.  If there are cycles in the graph, we remove one edge which has the maximum convex hull symmetric difference  \eqref{eq: CSdifference} and obtain a linear global depth ordering. 
We convexify each shape layer $S_i$ by minimizing Euler's elastica energy, with constraints on the shape-covered region $O_i$ of $S_i$ given by the depth ordering.  Then, we find the \bezier curves to vectorize each convexified region $C_i$, and stack them according to the depth ordering in a SVG file format.  This SVG file gives image vectorization with depth ordering and each shape layer is convexified as $C_i$.

\section{Analytical properties of depth ordering}\label{sec:analysis}
 
We explore some analytical properties of the proposed model, such as some properties of depth ordering energy, and the difference between using convex hull and curvature-based inpainting method when estimating occluded area. 

The covered-area measure $A(i,j)$ in  \eqref{eq:covered area measure} and convex hull symmetric difference $V(i,j)$ in  \eqref{eq: CSdifference} are non-symmetric measures, while $D(i,j)$ is skew-symmetric which is important for stability and consistency of depth-ordering computation. 

\begin{proposition}
The depth ordering energy $D(i,j)$ in  (\ref{eq: ordering}) has the following properties:
\begin{enumerate}
    \item 
    $D(i,j)$ is skew-symmetric: $D(i, j) = -D(j, i)$.
    \item
    $D(i, j) \in [-1, 1]$.
\end{enumerate}
\end{proposition}

\begin{proof}
The first statement follows from the definition of $D(i,j)$ in  \eqref{eq: ordering}, and the second statement is true, since $A(i,j)$ and $A(j,i)$ are both in $[0,1]$.
\end{proof}

These are simple properties, and yet, the skew-symmetry reduces the comparisons by a factor of $2$, since only $D(i,j)$ needs to be computed but not both $D(i,j)$ and $D(j,i)$, and this helps to create less cycles in the graph.  
We use $S_{i}^{\mathrm{Conv}}$ to denote the convex hull of $S_{i}$ and $\chi_{i}^{\mathrm{conv}}: \Omega \to \{ 0, 1 \}$ is the characteristic function of $S_{i}^{\mathrm{Conv}}$. 
In the following, we present a few more properties of the shape layer ordering. 

\begin{proposition} \label{prop: subset}
If shape layer $S_{i}$ is a subset of $S^{\mathrm{Conv}}_{j}$, then $S_{i} \rightarrow S_{j}$.
\end{proposition}

\begin{proof}
Since $S^{\mathrm{Conv}}_{i} \cap S_{j} \subset S_{j}$, 
\[D(i,j) = \frac{\int_{\Omega} \chi^{\mathrm{Conv}}_{j}\chi_{i} \mathrm{d}x}{\int_{\Omega} \chi_{i} \mathrm{d}x} - \frac{\int_{\Omega} \chi^{\mathrm{Conv}}_{i}\chi_{j} \mathrm{d}x}{\int_{\Omega} \chi_{j} \mathrm{d}x} = 1 - \frac{\int_{\Omega} \chi^{\mathrm{Conv}}_{i}\chi_{j} \mathrm{d}x}{\int_{\Omega} \chi_{j} \mathrm{d}x} > 0.\]
\end{proof}

This proposition is useful especially when $|D(i,j)|$ is very small, e.g.,$ |D(i,j)| <\delta$.  For example, consider a configuration where a thin doughnut-shape $S_j$ is surrounding the outer boundary of another convex shape layer $S_i$ such that area of $S^{\mathrm{Conv}}_j$ is close to that of $S^{\mathrm{Conv}}_i$.  In this case, both terms in $D(i,j)$ are similar, thus $|D(i,j)| < \delta$.  Using Proposition \ref{prop: subset}, once $S_i \subset S^{Conv}_j$ is confirmed, one does not need to compute $D(i,j)$ directly, but use $S_{i} \rightarrow S_{j}$. 
 
\begin{proposition}\label{prop: local convex regions}
Suppose two adjacent shape layers $S_{i}$ and $S_{j}$ share one boundary segment $\Gamma_{ij}$, and let $L_{ij}$ be the straight line connecting the two endpoints of $\Gamma_{ij}$.  Let the region bounded by $\Gamma_{ij}$ and $L_{ij}$ to be  $A_{ij}$.  If  each connected component of $A_{ij}$ are convex and $A_{ij}$ is  a subset of $S_{i}$, then $S_{i} \rightarrow S_{j}$, and if $A_{ij}$ is a subset of $S_{j}$, then $S_{j} \rightarrow S_{i}$.
\end{proposition}

\begin{proof}
Since each connected component of $A_{ij}$ are  convex, if $A_{ij} \subset S_{i}$,  $\int_{\Omega} \chi_{i}^{\mathrm{Conv}} \chi_{j} \mathrm{d} x = 0$. In $S_j$ point of view $\int_{\Omega} \chi_{j}^{\mathrm{Conv}} \chi_{i} \mathrm{d} x > 0$ since $A_{ij}$ in $S_i$ is convex, so $D(i,j) > 0$.  If $A_{ij} \subset S_{j}$, the same argument is true changing  $S_{i}$ to  $S_{j}$.
\end{proof}

Proposition \ref{prop: local convex regions} is not limited to a pair of shapes that share only one boundary segment; if they share multiple segments, this proposition can be applied to each connected boundary segment one by one. The depth ordering between these two shapes considers the sum of all boundary segments.

\begin{proposition}\label{prop: transitivity}
    If $S_{i} \rightarrow S_{j}$, $S_{j} \rightarrow S_{k}$ and $S_{k} \nrightarrow S_{i}$, then our depth ordering algorithm identifies $S_{i} \rightarrow S_{k}$.
\end{proposition}
\begin{proof}
    Given that $S_{k} \nrightarrow S_{i}$, we have either $S_{i} \rightarrow S_{k}$ or there is no depth ordering between $S_{i}$ and $S_{k}$ by direct computation of $D(i,k)$. The first case is exactly the result, and in the second, by the transitive property of directed graph, it gives $S_{i}$ is above $S_{k}$.
\end{proof}

Proposition \ref{prop: transitivity} identifies the condition of the transitive property of our depth ordering.  When extending Proposition \ref{prop: transitivity} to $n$ shape layers, i.e. given $D(i_{1}, i_{2}), D(i_{2}, i_{3}), \cdots, D(i_{n-1}, i_{n}) > 0$, to identify if a shape layer $S_{i_{0}}$ satisfies $S_{i_{0}} \rightarrow S_{i_{1}} \rightarrow \cdots \rightarrow S_{i_{n}}$, one need to verify $D(i_{0}, i_{j}) \leq 0$ for all $j = 1, \cdots , n$. Proposition \ref{prop: transitivity} gives theoretical guarantee of the natural transitive property of depth ordering in our model, thus once $S_{i_{0}} \rightarrow S_{i_{1}}$ is verified other shapes' depth information follows.

\begin{figure}
\centering
\begin{subfigure}[b]{.34\textwidth}
    \centering
    \subcaption{}
    \includegraphics[width = \textwidth]{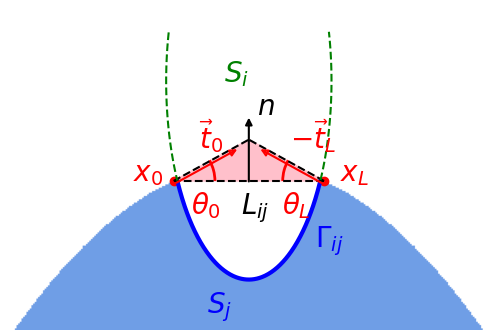}
\end{subfigure}
\begin{subfigure}[b]{.34\textwidth}
    \centering
    \subcaption{}
    \includegraphics[width = \textwidth]{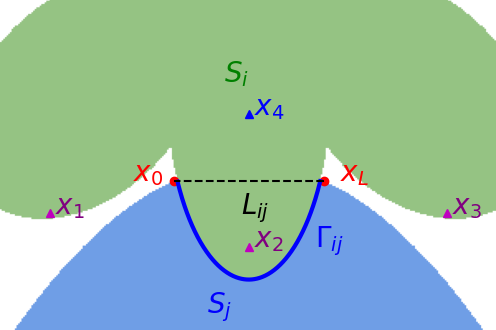}
\end{subfigure}
\begin{subfigure}[b]{.3\textwidth}
    \centering
    \subcaption{}
    \includegraphics[width = \textwidth]{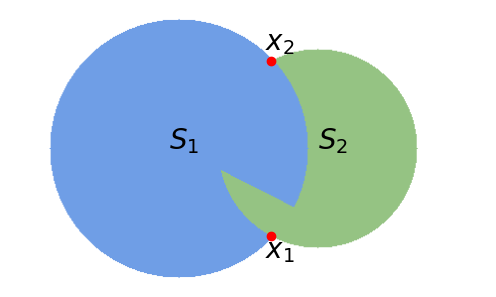}
\end{subfigure}
\caption{(a) The straight line $L_{ij}$ and boundary segment $\Gamma_{ij}$, the two endpoints $x_{0}$ and $x_{L}$, the tangent directions $\vec{t}_{0}$ and $-\vec{t}_{L}$, and the angles they make with $L_{ij}$ as $\theta_{0}$ and $\theta_{L}$.  The pink region is the bounding triangle $T^{j}$.  
(b) An example where $S_{j}$ is one-sided to $L_{ij}$, while the adjacent shape $S_{i}$ is not.  (c) The T-junctions $x_1$ and $x_2$ give conflicting information about the depth ordering between $S_1$ and $S_2$, while the proposed depth ordering gives $S_{2} \rightarrow S_{1}$ showing robustness of the proposed method.}
\label{fig: non example of ordering}
\end{figure}

For further analysis, we use the following notations as shown in Figure \ref{fig: non example of ordering}(a).  Let $S_{i}$ and $S_{j}$ be two adjacent shapes sharing $\Gamma_{ij}$ which is one connected mutual boundary segment. Let the two endpoints of $\Gamma_{ij}$ be $x_{0}$ and $x_{L}$, and $L_{ij}$ be the straight line connecting $x_{0}$ and $x_{L}$.  Let $\mathbf{r}: [0,1] \to \partial S_{j}$ be a piecewise differentiable parameterization of the closed boundary of $S_{j}$ such that for some $0 < s_{0} < s_{L} < 1$, we have $\mathbf{r}(s_{0}) = x_{0}$ and $\mathbf{r}(s_{L}) = x_{L}$.  We denote two tangent vectors to be \[ \vec{t}_{0} = \lim_{s \to s_{0}^{-}}\frac{\mathrm{d}\mathbf{r}}{\mathrm{d}s}\;\; \text{and} \;\;  \vec{t}_{L} = \lim_{s \to s_{L}^{+}}\frac{\mathrm{d}\mathbf{r}}{\mathrm{d}s}.\]
We denote $\theta_{0}$ be the angle between the vectors $\vec{t}_{0}$ and $x_{L} - x_{0}$, $\theta_{L}$ be the angle between $-\vec{t}_{L}$ and $x_{0} - x_{L}$, and  $\vec{n}$ be a vector orthogonal to $x_{L} - x_{0}$.

\begin{definition} \label{def: concave in shape}
Let $S_j$ be a shape, $L_{ij}$ be a straight line segment with two endpoints $x_{0}$ and $x_{L}$ on $\partial S_j$, and $\vec{n}$ be a vector orthogonal to $L{ij}$.  We denote a shape $S_j$ is \textbf{one-sided to $L_{ij}$} if for any $x \in S_j$, $\vec{n} \cdot (x - x_{0})$ (or $\vec{n} \cdot (x - x_{L})$) is either non-positive or non-negative.  
\end{definition}

Definition \ref{def: concave in shape} describes shapes that completely lies on one side of the line segment $L_{ij}$. In Figure \ref{fig: non example of ordering}(b), $S_{j}$ is a shape that is one-sided to $L_{ij}$; but $S_{i}$ is not.  Considering four points $x_1, x_2, x_3$ and $x_4$ as examples, the signs of $\vec{n} \cdot (x_{1} - x_{0}), \vec{n} \cdot (x_{2} - x_{0})$ and $ \vec{n} \cdot (x_{3} - x_{0})$ are opposite sign of $\vec{n} \cdot (x_{4} - x_{0})$.  While for $S_j$, for any point $x \subset S_j$, $\vec{n} \cdot (x - x_{0})$ will be the same sign.

While we prefer smooth boundaries and leverage Euler's elastica curvature-based inpainting model when inpainting shape layers, we use convex hull in \eqref{eq: ordering} to estimate occluded area for computational efficiency in the depth ordering step. In the following, we show the error of convex hull estimation compared against curvature based inpainting to explore when convex hull is a reasonable approximation for depth ordering computation. In \cite{kang2006error}, the authors explored various error analysis results for image inpainting, including Total Variation (TV) inpainting model for piecewise constant images. TV inpainting has great similarity with convex hull estimation that some results from~\cite{kang2006error} is transferable upon certain conditions. We approach the error analysis as region area comparison, since the shape layer $S_i$ is a region and the main error comes from the region difference, since color is not considered for convexification.   

\begin{definition}\label{def: bounding triangle}
Suppose $S_{j}$ is one-sided to $L_{ij}$.  Consider a straight line $L_0$, extending from $x_0$ following the direction of $\vec{t}_{0}$ which forms an angle $\theta_{0}$, and another straight line $L_L$, extending from $x_L$ with direction  $- \vec{t}_{L}$ and an angle $\theta_{L}$.  If both angles are less than $\frac{\pi}{2}$ and at least one of them is strictly less $\frac{\pi}{2}$, we define the \textbf{bounding triangle} $T^{j}$ of $S_{j}$ with respect to $L_{ij}$ as the triangle formed by $L_{ij}$, $L_0$ and $L_L$.  
\end{definition}

When a shape $S_j$ is one-sided to $L_{ij}$, the bounding triangle is formed on the other side of $L_{ij}$ as in  Figure \ref{fig: non example of ordering}(a).   
The area of this bounding triangle can be computed by considering $\frac{1}{2} |L_{ij}| a \sin \theta_0$ where $a$ is the length of side opposite to angle $\theta_L$, and  $|\cdot|$ representing the length.  Using the law of sines, $ \frac{a}{\sin \theta_L} = \frac{|L_{ij}|}{\sin (\pi - (\theta_0+\theta_L))}$ which gives $a = \frac{|L_{ij}| \sin \theta_L}{\sin(\theta_0+\theta_L)}$, this gives the area of bounding triangle to be
\begin{equation}\label{eq:triangleArea}
\text{Area of } T^{j} = |T^j| := \frac{| L_{ij}|^{2}_{2}}{2} \frac{\sin \theta_{0} \sin \theta_{L}}{\sin(\theta_{0} + \theta_{L})}. 
\end{equation}

Next proposition shows that any smooth curve connecting $x_0$ and $x_L$ smoothly with certain condition is bounded within the bounding triangle.

\begin{proposition}\label{prop: smooth connection within bounding triangle}
Suppose $S_{j}$ is one-sided to $L_{ij}$ and $T^j$ be the bounding triangle.  We define a local coordinate system where $L_{ij}$ is $x$-axis with $x_{0}$ at the origin and $x_{L}$ on the positive side.  We let $\vec{e}_{1}$ be the first standard basis vector $\begin{bmatrix} 1 \\ 0 \end{bmatrix}$ in this local coordinate system. 
Consider a smooth parameterized curve $\gamma(t) : [0,1] \to \mathbb{R}^{2}$ which connects $\gamma(0)=x_{0}$ and $\gamma(1)=x_{L}$.  If the angle between $\gamma^{\prime}(t)$ and  $\vec{e}_{1}$ monotonically decreases from $\theta_{0}$ to $-\theta_{L}$, then this smooth curve $\gamma(t)$ is within the bounding triangle $T^{j}$ for any $t \in [0,1]$.
\end{proposition}

\begin{proof}
Suppose there is $t_{0} \in [0,1]$ such that $\gamma(t_{0})$ is outside of the bounding triangle $T^{j}$. Then, the angle between $\gamma(t_{0}) - x_{0}$ and $\vec{e}_{1}$ is either larger than $\theta_{0}$ or not monotonically decreasing.  This is a contradiction. 
\end{proof}

We note that when we convexify the shape layer $S_j$ in the occluded region using  Euler's elastica curvature-based inpainting model \eqref{eq: curvature based energy}, the convexified result is also a curve satisfying the condition in Proposition \ref{prop: smooth connection within bounding triangle}. Since if the angle between the tangent vector and $\vec{e}_{1}$ increases, the elastica energy $\int_{\gamma} (a + b \kappa^{2}) \mathrm{d}s$ also increases.
Euler's elastica curvature-based inpainting model produces a natural inpainting result through the minimization of a combination of arc length and curvature. By establishing this upper bound on error, we demonstrate the consistency of our vectorization approach.

When the angles $\theta_0$ and $\theta_L$ are small, the area of bounding triangle \eqref{eq:triangleArea} is small, i.e., convex hull and the tangent direction extension are similar to each other.  While when both angles $\theta_0$ and $\theta_L$ are near $\frac{\pi}{2}$, then the denominator $\sin(\theta_0+\theta_L)$ becomes near zero, and this triangle will be very large.  In the following, we investigate conditions on this bounding triangle's angles for having a consistent depth ordering between using $\gamma(t)$ as defined in Proposition \ref{prop: smooth connection within bounding triangle}, which includes Euler's elastica model, and convex hull to estimate occluded area.

\begin{proposition}\label{prop: euler vs convex hull}
Suppose that $S_{i}$ is convex and $S_{j}$ is one-sided to $L_{ij}$. Let $\chi^{\gamma}_{j}$ be the characteristic function of shape of $S_{j}$ constructed by $\gamma$ which satisfies the condition in Proposition \ref{prop: smooth connection within bounding triangle} and connects $x_{0}$ and $x_{L}$. We denote the depth ordering given by $\gamma$ to be $D_\gamma(i,j)$ and suppose $D_{\gamma}(i,j) > 0$. If the two angles $\theta_{0}$ and $\theta_{L}$ satisfy
\begin{equation}\label{eq:angles_small}
|T^j| = \frac{| L_{ij}|^{2}_{2}}{2} \frac{\sin \theta_{0} \sin \theta_{L}}{\sin(\theta_{0} + \theta_{L})} < D_\gamma(i,j) \int_{\Omega} \chi_{i} \mathrm{d}x,
\end{equation}
then $D(i,j) > 0$, i.e. the depth ordering $D_\gamma(i,j)$ given by $\gamma$ is the same as that given by convex hull in $D(i,j)$.
\end{proposition}

\begin{proof}
We consider the difference $\epsilon = \int_{\Omega}(\chi^{\gamma}_{j} - \chi^{\mathrm{Conv}}_{j})\chi_{i} \mathrm{d}x$ to represent the area difference between the region constructed by smooth curve $\gamma$ and convex hull in $D(i,j)$. 
Since $S_{i}$ is convex, $A(j,i) = 0$, and $D(i,j) = A(i,j) = \frac{\int_{\Omega} \chi_{j}^{\mathrm{Conv}}\chi_{i} \mathrm{d}x}{\int_{\Omega} \chi_{i} \mathrm{d}x}$, and 
\begin{equation*}
D_\gamma (i,j) = \frac{\int_{\Omega} \chi^{\gamma}_{j}\chi_{i}\mathrm{d}x}{\int_{\Omega}\chi_{i}\mathrm{d}x} 
= \frac{\int_{\Omega} \chi^{\mathrm{Conv}}_{j}\chi_{1}\mathrm{d}x}{\int_{\Omega}\chi_{i}\mathrm{d}x} + \frac{\epsilon}{\int_{\Omega}\chi_{i}\mathrm{d}x} 
= D(i,j) + \frac{\epsilon}{\int_{\Omega}\chi_{i}\mathrm{d}x}. 
\end{equation*}
The difference area $\epsilon$ is within the bounding triangle $T^j$ in  \eqref{eq:triangleArea},  and from condition \eqref{eq:angles_small},
\[ D_\gamma(i,j) - D(i,j) = \frac{\epsilon}{\int_{\Omega}\chi_{i}\mathrm{d}x} < \frac{ |T^j|} {\int_{\Omega}\chi_{i}\mathrm{d}x} < \frac{D_{\gamma}(i,j) \int_{\Omega}\chi_{i}\mathrm{d}x}{\int_{\Omega}\chi_{i}\mathrm{d}x} = D_{\gamma}(i,j)\]
Thus if $D_{\gamma}(i,j) > 0$, then $D(i,j) > 0$.
\end{proof}

Proposition \ref{prop: euler vs convex hull} supports the use of convex hull for estimating depth ordering, since it shows using convex hull gives the same depth ordering as any curvature-based inpainting as long as the constructed result $\gamma(t)$ and the shape layer boundary satisfy the regularity given in Proposition \ref{prop: smooth connection within bounding triangle} and \ref{prop: euler vs convex hull}.  
This can be generalized that if the given color quantized raster image $f$ has all the shape layers $S_i$ satisfying the regularity condition \eqref{eq:angles_small} in Proposition \ref{prop: euler vs convex hull}, that any shape layer is either convex or one-sided to some straight line segment, then the proposed depth ordering given by $D(i,j)$ is consistent with the depth ordering given by $D_\gamma(i,j)$ using $\gamma$ defined in Proposition \ref{prop: smooth connection within bounding triangle}. 
Since convex hull is more computationally efficient compared to curvature-based inpainting for depth ordering, we leverage on convex hull method's efficiency. 

\textbf{Remark}  We note that the proposed depth ordering does not use T-junction information, while using T-junctions to determine depth ordering has certain advantages as shown in \cite{palou2013depth}.  There are three major reasons for using area based depth ordering $D(i,j)$ in this paper. First, from the given color quantized image $f$ it is not easy to compute accurate T-junction due to staircase effect, especially for small regions. Second, the number of T-junctions in a real image could be quite large, possibly larger than the number of connected components. Last but not least, we found that using area based measure is more stable.
Figure \ref{fig: non example of ordering} (c) presents such an example, where T-junction at $x_1$ and $x_2$ give conflicting information on which shape is above which.  However, using the depth ordering \eqref{eq: ordering}, our method computes $D(1,2) = A(1,2) - A(2,1) \approx 0.02514-0.03538 < 0$ and gives $S_{2} \rightarrow S_{1}$.  

\section{Euler's elastica based model for inpainting shape layers }\label{sec:elastica}

To convexify each shape layer's occluded region, we use Euler's elastica model, which dates back to 1744, when Euler~\cite{euler1744methodus} solved the well-known elastica problem, which is to find a curve that minimises a linear combination of arc-length and squared curvature term. Natural extension along curvature direction in imaging is a very attractive feature that these models are explored in image segmentation as well as in image inpainting. Mumford explored visual perception and construction of elastica model for computer vision~\cite{mumford1994elastica}, Masnou et al.~\cite{morel1998elastica} presented a framework for level line structure to achieve disocclusion. Shen et al.~\cite{inpaint2003} studied the mathematical perspective and numerically computed Euler's elastica based inpainting model. Chan et al.~\cite{chan2001nontexture} explored a curvature-based inpainting model.  Ballester et al.~\cite{sapiro2001inpaint} explored joint interpolation of vector fields and gray level to incorporate curvature directions, and  Chan et al.~\cite{chan2001nontexture} explored curvature-driven diffusion. Bredies et al.~\cite{bredies2015euler} suggested a convex, lower semi-continuous modification to the model. 
Despite the versatility of Euler's elastica model, difficulty still lies in its high non-linearity and non-convex property which makes computation slow and difficult. A fast algorithm based on Augmented Lagrangian Method is explored in~\cite{tai2011alm}. In~\cite{Yashtini2016AFR}, authors suggested two numerical schemes for the Euler's elastica problem that are based on operator splitting and alternating direction method of multipliers, and in~\cite{zhu2017}, authors used Augmented Lagrangian Method for elastica based segmentation model.  
For large domain inpainting problem, which is the case considered in this paper, we take an alternative direction for more stable large region computation: de Giorgi~\cite{DeGiorgi1991} studied a $\Gamma$-convergence approach, and in~\cite{kang2014}, authors adopted this approach for illusory shapes construction using large domain curvature-based inpainting.

\subsection{Corner phase function and inpainting}
\label{subsec: phase and support}

One of the unique idea of illusory shapes construction in~\cite{kang2014} is to get a clue of illusory shape from convex corners of the given shapes. These convex corners are assumed to be generated from occlusion by an illusory shape, and we also extend such ideas for convexifying shape layers. In particular, we utilize shape-covered regions $\mathcal{O}_{i}$ in  \eqref{eq:O_i} to make sure inpainting of $S_{i}$ happens inside and only inside $\mathcal{O}_{i}$, such that after stacking vectorized shape layers, the inpainted regions would not be seen without moving shape layers on top away. We first introduce the following definitions:
 
\begin{definition}\label{def:endpoint}
Let $\partial S_{i}$ be the boundary of shape $S_{i}$, and $\Gamma_{i}$ be a part of  $\partial S_{i}$ touching $O_i$, i.e., $\Gamma_{i}:=\partial S_{i} \bigcap \partial O_i$. Let $N_\Gamma$ be the total number of connected curves $\gamma_{il}$ in $\Gamma_{i}$, thus $\Gamma_i=\bigcup_{l=1}^{N_\Gamma} \{ \gamma_{il} \}$ and each $\gamma_{ij}$ and $\gamma_{ik}$ are disjoint curves for $j \neq k$. We define the two endpoints of $\gamma_{il}$ as the \textbf{inpainting endpoints}, $b_{il0}$ and $b_{il1}$, and let the set of all inpatining endpoints to be $\mathcal{B}_{i} = \bigcup_{l=1}^{N_\Gamma} \{b_{ilj} | j =0, 1\}$.
\end{definition}

In Figure \ref{fig: sun euler result} (a), the orange region has the blue line as $\Gamma_{i}$ with two red endpoints.

\begin{definition} 
Let $\partial S_{i}$ be the boundary of $S_{i}$, and $\mathcal{B}_{i} = \bigcup_{l=1}^{N_\Gamma} \{b_{ilj} | j =0, 1\}$ be the set of inpainting endpoints for $S_i$.
Since $S_i$ is a closed region, there exists a parameterization of $\partial S_{i}$ such that $s_i(t) : [0,1] \to \partial S_{i}$, and $s_i(0) = s_i(1)$. For an inpainting endpoint $b_{ilj}$, we let $s_i(t_{ilj}) = b_{ilj}$, and let $\mathbf{n}(t_{ilj}^{-})$ and $\mathbf{n}(t_{ilj}^{+})$ be the pre and post-normal vector of $s_i(t)$ at an inpainting endpoint $b_{ilj}$.  Let $B(b_{ilj},r)$ be a small disk centered at $b_{ilj}$ with a small fixed radius $r$.  
We define \textbf{inpainting corner phase function} $\psi_{b_{ilj}}: B(b_{ilj},r) \to \{-1, 0, 1\}$ as 
\begin{equation}\label{eq: phase def}
    \psi_{b_{ilj}}(x) = \left\{\begin{array}{rl}
        -1, & \text{if $\mathbf{n}(t_{ilj}^{-}) \cdot (x - b_{ilj}) \geq 0$ and $\mathbf{n}(t_{ilj}^{+}) \cdot (x - b_{ilj}) \geq 0$}, \\
    0,  & \text{if $x \in S_i$, and}  \\
    1, & \text{otherwise}.
        \end{array} \right. 
\end{equation}
\end{definition}
Here the region with $\psi_{b_{ilj}}=1$, the purple region in Figure \ref{fig: sun euler result}(a), is the inpainting domain where the region $S_i$ can be extended to, and the region with $\psi_{b_{ilj}}=-1$, the yellow region in Figure \ref{fig: sun euler result}(a), is where $S_{i}$ should not extend to in order to stay faithful to input raster image.  
Using Figure \ref{fig:shapelayer} as an example, if all the black and white regions, $S_1, S_2, S_5$ and $S_6$, are identified to be above the orange sun $S_3$, this makes the light green region $O_3\backslash S_3$ in Figure \ref{fig: sun euler result} (a). Since the orange sun $S_3$ only has one boundary segment (the blue curve) touching its shape-covered region, we have two inpainting corners (red dots).  At these two inpainting endpoints, we define the inpainting corner phase functions. The phase function is equal to $1$ on purple regions, where inpainting is desired, and is equal to $-1$ on yellow region where inpainting is not allowed.

\begin{figure}
\centering
\begin{tabular}{cc}
(a) & (b) \\
\begin{tikzpicture} [spy using outlines={rectangle,red,magnification=5,size=2.25cm, connect spies}]
\node[anchor=south west,inner sep=0] (image) at (1.5,0) {\includegraphics[width = 4cm]{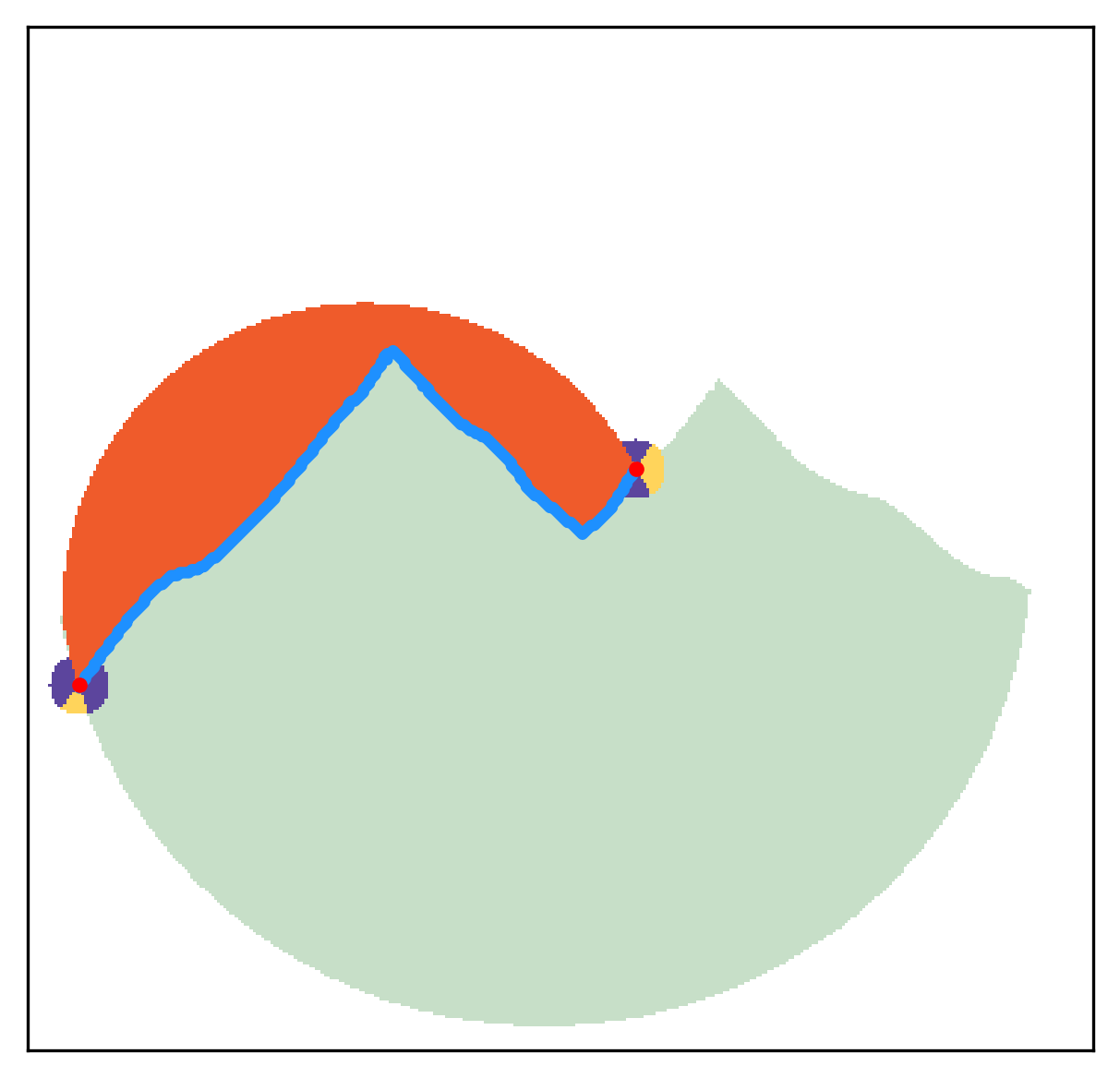}};
\spy[red, connect spies] on (1.85,1.4) in node at (0, 1.8) ;
\spy[red, connect spies] on (3.80,2.2) in node at (7, 1.8) ;
\end{tikzpicture}
&
\includegraphics[width = 4cm]{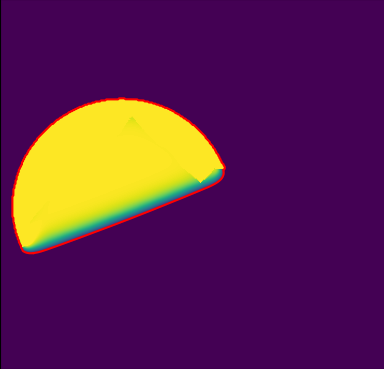}
\end{tabular}
\caption{[Inpainting corner phase function] From the shape layers in Figure \ref{fig:shapelayer}, (a) shows the orange sun $S_3$ and the blue curve $\Gamma_3$.
Two zoomed areas show the inpainting endpoint phase functions, where the yellow region is $\psi(b_{ilj})= -1$ (not to diffuse), and the purple region is  $\psi(b_{ilj})=1$ (to inpaint). (b) shows the curvature inpainted region represented by a phase transition function $u_i >-0.5$, and the red boundary is used to fit \bezier curve for vectorization. Notices the orange region $S_3$ in (a) is convexified to $C_3$ in (b) as a red curve. } 
\label{fig: sun euler result}
\end{figure}

\subsection{Double-well potential model for Euler's elastica curvature-based functional }\label{subsec: gamma euler elastica}

We convexify each shape layer $S_i$ by a phase transition function $u_{i} : \Omega \to [-1,1]$ by minimizing the following \emph{Euler's elastica curvature-based inpainting with shape-covered region $O_i$ constraint}:
\begin{align}\label{eq: energy with phase and support}
    E(u_{i}) &= \int_{O_i} \left( a + b \left( \nabla \cdot \frac{\nabla u_{i}}{|\nabla u_{i}|} \right)^{2} \right) |\nabla u_{i}| \mathrm{d}x + \sum_{l=1}^{N_\Gamma} \sum_{j=0}^{1}\int_{B(b_{ilj},r)\cap O_i} (u_{i} - \psi_{b_{ilj}})^{2} \mathrm{d}x  \\ 
    &\text{subject to } S_{i} \subseteq \{x \mid u_{i}(x) > 0 \} \subseteq O_{i} \nonumber
 \end{align}
where $N_\Gamma$ is the total number of connected curves in $\Gamma_i$ as in Definition \ref{def:endpoint}.  In the first term, the integral is over occluded region $O_i$ which is given by depth ordering, and the second term considers all disks centered at the inpainting corners $\mathcal{B}_{i} = \bigcup_{l=1}^{N_\Gamma} \{b_{ilj} | j =0, 1\}$  of $S_i$. The second term is fitting the phase information given by inpainting corner phase functions, while the first term extends the boundary following the curvature direction.  This model has a constraint that $S_i \subseteq C_i \subseteq O_i$(where $C_{i}$ is the output inpainted shape), and using phase transition function representation, we have an equivalent constraint $S_{i} \subseteq \{x \mid u_{i}(x) > 0 \} \subseteq O_{i}$.

We consider the $\Gamma$-convergence approximated energy proposed by de Giorgi~\cite{DeGiorgi1991}, and use corner-based large domain inpainting as in~\cite{kang2014}.  
We look for inpainted shape layer $C_i$, given the inpainting corner phase function $\psi_{b_j}$ and shape-covered region $O_i$, by minimizing
\begin{equation}
 \label{eq: euler elastica}
E_{\epsilon}(u^{\epsilon}_{i})  =  \displaystyle{\int_{O_i} a \left( \frac{\epsilon}{2} | \nabla u^{\epsilon}_{i} |^{2} + \frac{W(u^{\epsilon}_{i})}{2 \epsilon} \right) 
+ \frac{b}{\epsilon} \left( \epsilon \Delta u^{\epsilon}_{i} - \frac{W^{\prime}(u^{\epsilon}_{i})}{2 \epsilon}\right)^{2} \mathrm{d}x 
+ \sum_{l=1}^{N_\Gamma} \sum_{j=0}^{1} \int_{B(b_{ilj},r) \cap O_i} (u^{\epsilon}_{i} - \psi_{b_{ilj}})^{2} \mathrm{d}x} 
\end{equation}
subject to $S_{i} \subseteq \Omega^{\epsilon}_{i} = \{ x \mid u_{i}^{\epsilon}(x) > 0 \} \subseteq O_i$, where $\epsilon > 0$ is a small constant and $W(x)$ is the double-well potential 
\[ 
    W(x) = (x-1)^{2}(x+1)^{2}.
\] 
It is proved in~\cite{modica1987gradient} that with this choice of double-well potential energy, the solution to  \eqref{eq: euler elastica} $\Gamma$-converges to that of  \eqref{eq: energy with phase and support} as $\epsilon \to 0^{+}$. One advantage of solving  \eqref{eq: euler elastica} is that it produces  a second order PDE using Euler-Lagrange, whereas solving  \eqref{eq: energy with phase and support} directly would produce a higher order PDE which leads to higher unstability.
In~\cite{bellettini1993minimizer}, the authors considered an extended lower semi-continuous envelope of \eqref{eq: curvature based energy} and determined the domain where this envelope is bounded.

Figure \ref{fig: sun euler result} (a) shows the input to the Euler's elastica model \eqref{eq: euler elastica} with two inpainting corners, and the light green region represents the shape-covered region $O_i$.  The orange region is the sun's original area.  Figure \ref{fig: sun euler result} (b) shows the inpainted shape layer $C_i$ bounder by a red closed curve. The image value ranges from $-1$(solid purple) to $1$(solid yellow). Taking the phase transition  $C_i =\{x | u_i(x) > 0.5 \}$ gives a smooth contour that is depicted in red.  The proposed model \eqref{eq: euler elastica} successfully forces the inpainting corners information to diffuse to the direction that is the intersection of the shape-covered region and where the phase function is equal to $1$.

\section{Numerical implementation details} \label{sec:numerical}

We provide details of numerical implementation in this section.  We first mention a denoising step of removing small regions after color quantization in subsection \ref{subsec: denoising}.  The computational details of elastica curvature inpainting model is presented in subsection \ref{ssec:curvature}, and we explain more details in section \ref{subsec: convexification and topological sort}. Pseudocode is included in Appendix  \ref{Asec:codes}.

\subsection{Denoising shape layers}
\label{subsec: denoising}

During the color quantization step, it is common that there are many small connected components whose color is quite different from any of adjacent large regions.  In Figure \ref{fig: noisy regions}(a), we show an example: after $K$-mean color quantization step, some pixels between the large black and white shape are mis-identified, giving a color which is neither black or white.  
This effect is one of the challenges in vectorization: how to either properly remove these small regions or correct their colors to one of the adjacent colors, while keeping the correct features of the boundary. 

\begin{figure}
\centering
\begin{tabular}{cc}
(a) & (b) \\
\begin{tikzpicture}[spy using outlines={rectangle,red,magnification=15,size=2.5cm, connect spies}]
\node [anchor=south west,inner sep=0] (image) at (1.5,0) {\includegraphics[width = 4cm]{pic/mountain_example/quantized_mountain.png}};
\spy[red, connect spies] on (1.95,3) in node at (-0.25, 1.75) ;
\spy[red, connect spies] on (5,1.25) in node at (7.25, 1.75) ;
\end{tikzpicture}
&
\includegraphics[width=4cm]{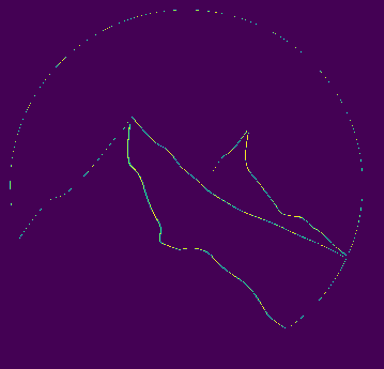}
\end{tabular}
\caption{[Noise layer $S_{noise}$] (a) Two zoomed areas of $f$ in Figure \ref{fig:shapelayer} show different colors (orange) although they are between green/yellow or black/white. (b) The noise layer $S_{noise}$ which is a collection of small noisy regions from image (a). }
\label{fig: noisy regions}
\end{figure}

In this work, we identify such small connected regions as noise and incorporate them into shape-covered regions $O_i$s.
\begin{definition}\label{def: noisy components}    
A small connected region $S_{n_i}$ is defined as a part of the \emph{noise layer}, $S_{noise}$, if it satisfies the following two conditions: 
\begin{enumerate}
\item{The area is small: $|S_{n_i}| \leq \epsilon$, for some small integer $\epsilon$, and }  
\item{$S_{n_i}$ is adjacent to at least two shapes of different colors.}
\end{enumerate}
We define a noise layer 
\begin{equation} \label{eq:noiselayer}
 S_{noise} = \bigsqcup_{\forall n_{i}} S_{n_i}     
\end{equation}
as the union of all noisy small connected regions.
\end{definition} 
Figure \ref{fig: noisy regions}(b) shows the noise layer $S_{noise}$, and this is added to the shape-covered regions, i.e.,  $S_{noise} \subset O_i $ for all $i=1,\dots, N_\mathcal{S}$ where $N_\mathcal{S}$ is the total number of shape layers as in Definition \ref{def:shapelayer}. This helps the boundary of each shape layers to follow the curvature direction and has a less chance to leave a gap between adjacent shapes in vectorization.  The noise layer $S_{noise}$ is not considered to be one of the shape layers, since following Definition \ref{def:shapelayer},  $S_{noise}$ does not have one associated color $c_l$, thus it is not considered for depth ordering nor considered for vectorization in general.

\subsection{Numerical details of Euler's elastica inpainting model} \label{ssec:curvature}

To find a minimum of Euler’s elastica curvature-based double-well potential model with the constraint in  \eqref{eq: euler elastica}, we compute the Euler-Lagrange equation:
\[ 
    a \left( -\epsilon \Delta u^{\epsilon}_{i} + \frac{W^{\prime}(u^{\epsilon}_{i})}{2\epsilon} \right) + 2b \Delta \left( \epsilon \Delta u^{\epsilon}_{i} - \frac{W^{\prime}(u^{\epsilon}_{i})}{2 \epsilon}\right) - \frac{b}{\epsilon^{2}}W^{\prime \prime}(u^{\epsilon}_{i})\left( \epsilon \Delta u^{\epsilon}_{i} - \frac{W^{\prime}(u^{\epsilon}_{i})}{2 \epsilon}\right) = -2\sum_{l=1}^{N_\Gamma} \sum_{j=0}^{1} (u^{\epsilon}_{i} - \psi_{b_{ilj}})\chi_{b_{ilj}}
\]
where $\chi_{b_{ilj}}$ is the characteristic function of the disk $B(b_{ilj},r)$ centered at an inpainting endpoints $b_{ilj}$. 
This is solved on $O_i$ and initialization is given as the characteristic function of $S_i$. 
We introduce an auxiliary function $v^{\epsilon}_{i}$ as in \cite{Ginzburg2009, Cahn1958, kang2014}
\[   
v^{\epsilon}_{i} = \epsilon \Delta u^{\epsilon}_{i} - \frac{W^{\prime}(u^{\epsilon}_{i})}{2 \epsilon},
\]
and solve the following iterative scheme:
\begin{align}
    -av^{\epsilon}_{i} + 2b\Delta v^{\epsilon}_{i} - \frac{b}{\epsilon^{2}}W^{\prime \prime}(u^{\epsilon}_{i})v^{\epsilon}_{i} &= -2 \sum_{l=1}^{N_\Gamma} \sum_{j=0}^{1}(u^{\epsilon}_{i} - \psi_{b_{ilj}})\chi_{b_{ilj}} \label{eq: v subproblem} \\
    \epsilon \Delta u^{\epsilon}_{i} - \frac{W^{\prime}(u^{\epsilon}_{i})}{2 \epsilon} &= v_{i}^{\epsilon} \label{eq: u subproblem} 
\end{align}
with constraints that $ u_{i}^{\epsilon}(x) = 1$ if $x \in S_i$ and $u_{i}^{\epsilon}(x) = -1$ if $x \notin O_i$.  We use the standard forward and backward discrete operators:
\begin{align*} 
    \partial_{x}^{-}v(i, j) &= \begin{cases}
        v(i, j) - v(i-1,j), & 1 < i \leq w \\
        v(1, j) - v(w,j), & i=1 
    \end{cases},
    \partial_{x}^{+}v(i, j) = \begin{cases}
        v(i+1, j) - v(i,j), & 1 \leq i \leq w - 1 \\
        v(1, j) - v(w,j), & i=w 
    \end{cases}\\
    \partial_{y}^{-}v(i, j) &= \begin{cases}
        v(i, j) - v(i,j-1), & 1 < j \leq h \\
        v(i, 1) - v(i,h), & j=1 
    \end{cases},
    \partial_{y}^{+}v(i, j) = \begin{cases}
        v(i, j + 1) - v(i,j), & 1 \leq j \leq h - 1 \\
        v(i, 1) - v(i,h), & j=h 
    \end{cases}. 
\end{align*}

To solve $v$-subproblem  \eqref{eq: v subproblem}, we add an extra Tikhonov type regularization term $cv^{\epsilon}_{i}$ to increase stability for some small positive constant $c$.
\begin{equation}\label{eq: discrete v subproblem}
    (-a + c) v^{\epsilon}_{i} + 2b (\partial_{x}^{-}\partial_{x}^{+} + \partial_{y}^{-}\partial_{y}^{+}) v^{\epsilon}_{i} = -2\sum_{j=1}^{k} (u^{\epsilon}_{i} - \psi_{b_{j}})\chi_{b_{j}} + \frac{b}{\epsilon^{2}}W^{\prime \prime}(u^{\epsilon}_{i})v^{\epsilon}_{i} + c v^{\epsilon}_{i}.
\end{equation}
As in~\cite{tai2011alm, kang2014}, we apply Fast Fourier Transform (FFT) $\mathcal{F}$ to solve the above  \eqref{eq: discrete v subproblem} to utilize the advantage of one-off fast pointwise division and FFT process, without an additional iterative process.  
Thus, for every grid point $(i,j)$, we have the discrete Laplacian operator as: 
\begin{align*}
     \mathcal{F}(\partial_{x}^{-}\partial_{x}^{+} + \partial_{y}^{-}\partial_{y}^{+}) &= -(e^{-\sqrt{-1}2 \pi (i-1) / w} - 1)(e^{\sqrt{-1}2 \pi (i-1) / w} - 1) -(e^{-\sqrt{-1}2 \pi (j-1) / h} - 1)(e^{\sqrt{-1}2 \pi (j-1) / h} - 1) \\
     &= -2\left( 2 - \cos \left( \frac{2\pi(i-1)}{w}\right) - \cos \left( \frac{2\pi(j-1)}{h}\right)\right). 
\end{align*}
Applying discrete Fourier Transform to the both sides of  (\ref{eq: discrete v subproblem}), we have
\begin{equation}\label{eq: fft'ed v subproblem}
\begin{array}{l}
     \left(4b \left( 2 - \cos \left( \frac{2\pi(i-1)}{w}\right)  - \cos \left( \frac{2\pi(j-1)}{h}\right) \right) + a\right)\mathcal{F}v^{\epsilon}_{i}(i,j) \\
     = \mathcal{F} \left( -2\sum_{j=1}^{k} (u^{\epsilon}_{i} - \phi_{p_{j}})\chi_{p_{j}} + \frac{b}{\epsilon^{2}}W^{\prime \prime}(u^{\epsilon}_{i})v^{\epsilon}_{i} + cv^{\epsilon}_{i}\right)
\end{array}
\end{equation}
At each point, $v^{\epsilon}_{i}$ is computed by a division and an inverse FFT.

To solve $u$-subproblem \eqref{eq: u subproblem}, we similarly solve for $u^{\epsilon}_{i}$ with a Tikhonov type  regularization term $cu^{\epsilon}_{i}$, $c > 0$, 
\begin{equation}\label{eq: discrete u subproblem}
    -2\epsilon \left( 2 - \cos \left( \frac{2\pi(i-1)}{w}\right) - \cos \left( \frac{2\pi(j-1)}{h}\right) + \frac{c}{-2\epsilon}\right) \mathcal{F}u^{\epsilon}_{i}(i,j) = \mathcal{F} \left( v^{\epsilon}_{i} +  \frac{W^{\prime}(u^{\epsilon}_{i}(i,j))}{2 \epsilon} + cu^{\epsilon}_{i}(i,j) \right).
\end{equation}

For initialization, $u^{\epsilon}_{i}$ is given as the shape of $S_i$, i.e., $u^{\epsilon}_{i} = \chi_{i}$, and $v^{\epsilon}_{i}$ is initialized to a zero function. For a non-simply-connected $S_{i}$, we fill-in the holes of $u^{\epsilon}_{i}$ as an initialization by~\cite{2020SciPy-NMeth}.   To speed up the inpainting process, for small shape layers $S_{i}$'s with area less than a threshold, e.g., picked to be $30$ pixels in the experiments, we directly obtain $S_{i}^{\mathrm{Conv}} \cap O_{i}$ as output.

\subsection{Convex hull computation and topological sort}\label{subsec: convexification and topological sort}

We briefly introduce standard algorithms to find the convex hull of a binary image and topological sort that are used to estimate a depth ordering \eqref{eq: ordering} of input raster images. 
To compute the convex hull for a 2D black and white image, we use the Graham scan algorithm~\cite{graham1972}. We first identify the set of boundary pixels representing the object which serves as the input points for the algorithm.  The starting point which has the lowest y-coordinate is picked (and in case of ties, the leftmost pixel), and the remaining points are sorted based on the polar angle they form with the starting point. The algorithm constructs the convex hull by iterating through the sorted list and using a stack, a linear data structure that accompanies the Last-In-First-Out principle, to maintain the sequence of points that form the convex boundary, ensuring that only left turns are made to exclude interior points. This method efficiently yields the smallest convex polygon enclosing all the boundary pixels of the object.

To form a linear global depth ordering from the directed graph $G(M,E)$ in Section \ref{subsec: depth graph}, we use topological sorting~\cite{knuth97}. The algorithm starts by identifying the source which is a node with no incoming edges (there may be multiple source and one can be chosen randomly) from the given graph.  We store it to the output list, and remove it from the graph.  Then, find the next source and place it behind the first node in the sorted list.  This process iterates until all nodes are processed, resulting in a valid topological order of the graph. This method ensures that dependencies represented by the edges are respected in the final sequence.

\section{Numerical Experiments} \label{sec: experiments}

We present various numerical results in this section.  First, Figure \ref{fig: mountain shapes} shows the progress of the proposed algorithm for the example in Figure \ref{fig:shapelayer}.  Figure \ref{fig: mountain shapes}(a) shows the results of depth ordering \eqref{eq: ordering} as a table. The depth ordering with $S_7$ is computed, but not shown, since it is obviously at the bottom and gives $D(i,7) > 0$ for all $i=1,2,\dots,6$.  
The pairwise depth ordering yields the global depth ordering graph $G(M,E)$. If there are directed cycles in the graph, we use convex hull symmetric difference   \eqref{eq: CSdifference} to break them and obtain a linear depth ordering. 
Each shape layer is convexified using the Euler's elastia curvature-base model \eqref{eq: euler elastica} and shown in the second row (b).  The noise layer $S_{noise}$, shown in Figure \ref{fig: noisy regions}(b), is added to shape-covered region $O_i$ for all shape layers $S_i$.   The last row (c) shows the stacking of vectorized layers following the depth ordering in the reverse order. Here $R_{7,4}$ represents stacking of $C_7$ and $C_4$, i.e. convexified shape layers of $S_{7}$ and $S_{4}$. Figure \ref{fig:mountainzoom} shows zoom of details of the result from Figure \ref{fig: mountain shapes} $R_{7,4,3,2,6,1,5}$.  Zoomed images show good approximations to the curves, and T-junction is also well-approximated, each curves following each level line direction. 
\begin{figure}
\centering
(a) 
\begin{tabular}{|c|c|c|c|c|c!{\vrule width 2pt}c|c|c|c|c|c|}
\hline
$S_i$ & $S_j$ & $A(i,j)$ & $A(j,i)$ & $D(i,j)$ & ordering & $S_i$ & $S_j$  & $A(i,j)$ & $A(j,i)$ & $D(i,j)$ & ordering \\ \hline
$S_2$ & $S_3$  & 0.1649  & 0.0490 &  0.116 &  $S_2$ $\rightarrow$ $S_3$ & $S_3$  &  $S_4$         & 0.960  & 0.003 &  0.957 &  $S_3$  $\rightarrow$ $S_4$     \\ \hline
$S_3$      &   $S_6$    &  0.0821  & 0.174  & -0.0917  &  $S_6$  $\rightarrow$ $S_3$   &  $S_1$     &    $S_6$   &  0.480  & 0.008  & 0.472  &  $S_1$  $\rightarrow$ $S_6$     \\ \hline
$S_2$    &    $S_6$    & 0.0551  &  0.331 & -0.276  &   $S_6$  $\rightarrow$ $S_2$   &  $S_5$   &     $S_4$     & 0.604  &  0.0190 & 0.585  &   $S_5$  $\rightarrow$ $S_4$  \\
\hline 
\end{tabular}

\vspace{0.3cm}
\begin{tabular}{c}
   (b) \\ 
    \begin{tikzpicture}[node distance=2.45cm, every edge/.style={draw=black, ->}]
    \node (S5) [label=above:$S_{5}$] {\includegraphics[width=0.11\textwidth]{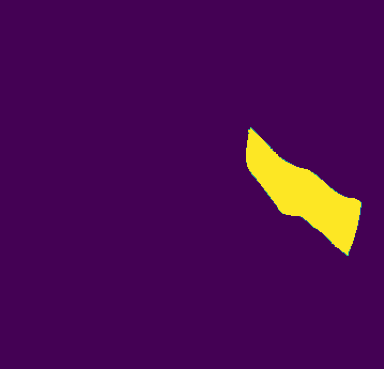}};
    \node (S1) [label=above:$S_{1}$, right of = S5] {\includegraphics[width=0.11\textwidth]{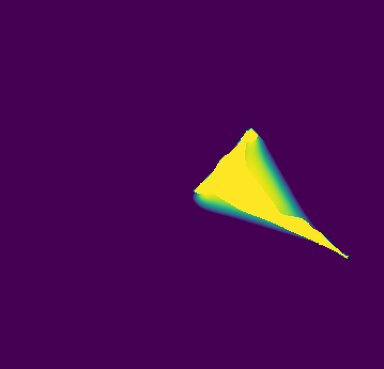}};
    \node (S6) [label=above:$S_{6}$, right of = S1] {\includegraphics[width=0.11\textwidth]{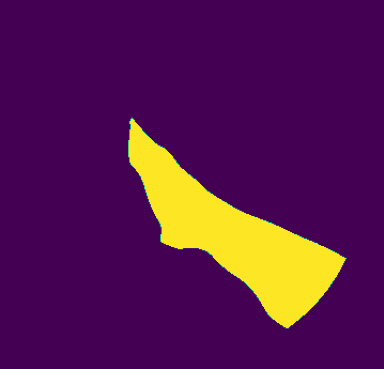}};
    \node (S2) [label=above:$S_{2}$, right of = S6] {\includegraphics[width=0.11\textwidth]{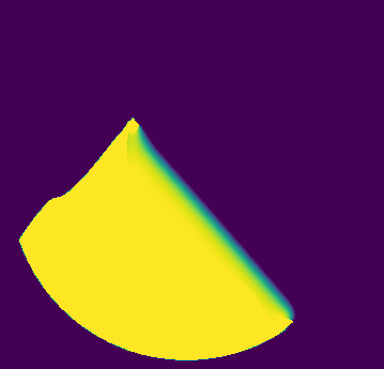}};
    \node (S3) [label=above:$S_{3}$, right of = S2] {\includegraphics[width=0.11\textwidth]{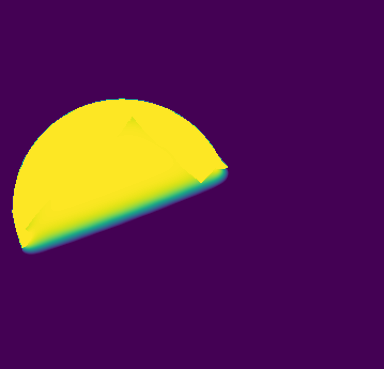}};
    \node (S4) [label=above:$S_{4}$, right of = S3] {\includegraphics[width=0.11\textwidth]{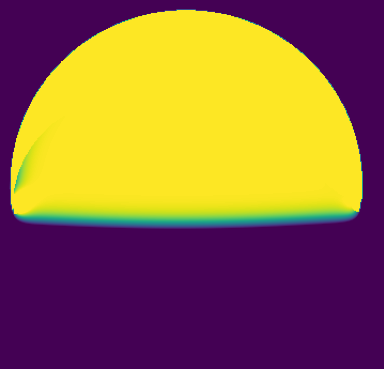}};
    \node (S7) [label=above:$S_{7}$, right of = S4] {\includegraphics[width=0.11\textwidth]{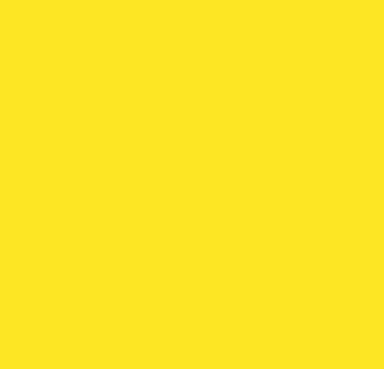}};
    \draw [->] (S5) -- (S1);
    \draw [->] (S1) -- (S6);
    \draw [->] (S6) -- (S2);
    \draw [->] (S2) -- (S3);
    \draw [->] (S3) -- (S4);
    \draw [->] (S4) -- (S7);
    \end{tikzpicture}
  \end{tabular}
\vspace{0.3cm}
(c)
\begin{tabular}{cccccc}
$R_{7,4}$ & $R_{7,4,3}$ & $R_{7,4,3,2}$ & $R_{7,4,3,2,6}$ & $R_{7,4,3,2,6,1}$ & $R_{7,4,3,2,6,1,5}$ \\
\includegraphics[width=0.14\textwidth]{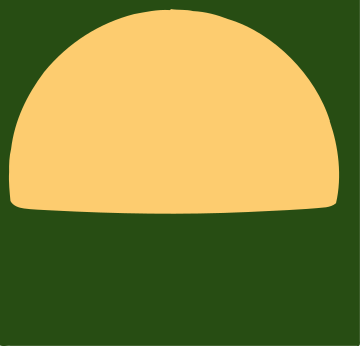} &
\includegraphics[width=0.14\textwidth]{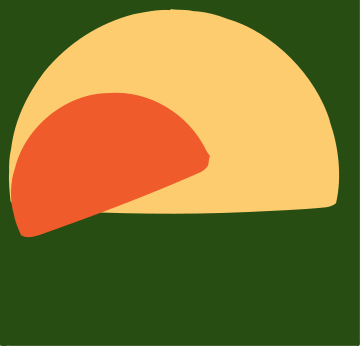} &
\includegraphics[width=0.14\textwidth]{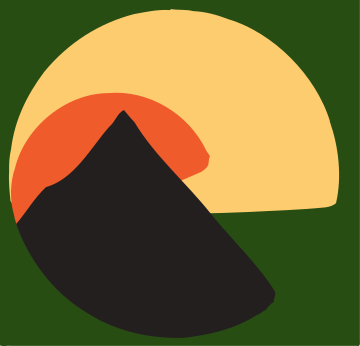}&
\includegraphics[width=0.14\textwidth]{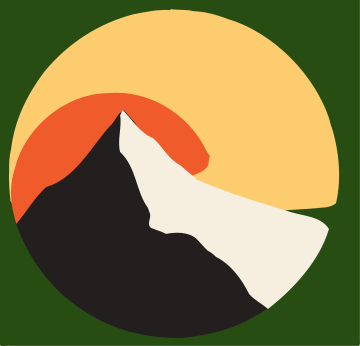} &
\includegraphics[width=0.14\textwidth]{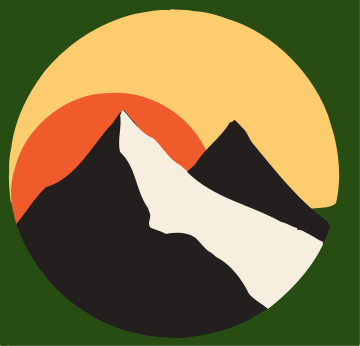} &
\includegraphics[width=0.14\textwidth]{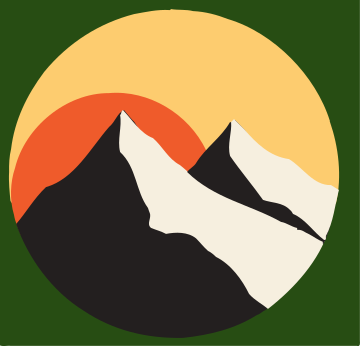}\end{tabular}
\caption{[Outline of image vectorization with depth] From the input image $f$ and shape layers in Figure \ref{fig:shapelayer}, (a) Table of $D(i,j)$ depth ordering.  It determines $S_5 \rightarrow S_1 \rightarrow S_6 \rightarrow S_2 \rightarrow S_3 \rightarrow S_4 \rightarrow S_7$. In (b) each shape layer is convexified and depth ordered from left to right.  The last row shows the stacking of vectorized layers in reverse order. Here $R_{7,4}$ represents stacking of $C_7$ and $C_4$.}  
\label{fig: mountain shapes}
\end{figure}
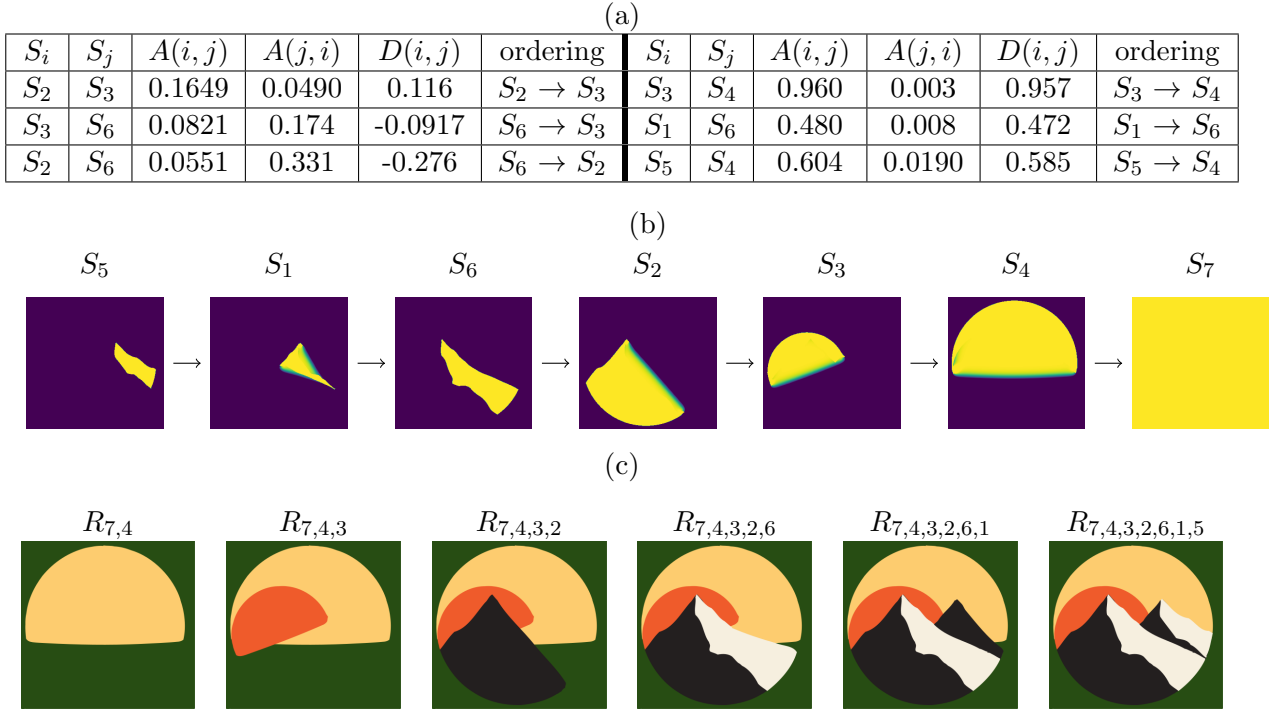

\begin{figure}
\centering
\begin{tikzpicture} [node distance = 7cm, spy using outlines={rectangle,red,magnification=9,size=2cm, connect spies}]
\node[label = above: (a)] (quantized image) {\includegraphics[width = 4cm]{pic/mountain_example/quantized_mountain.png}};
\spy[red, connect spies] on (-.3, -.6) in node at (3.1,-1.1);
\spy[red, connect spies] on (0.3, 0.27) in node at (3.1,1.1);
  \node[right of = quantized image, label = above: (b)] (vectorized image) { \includegraphics[width = 4cm]{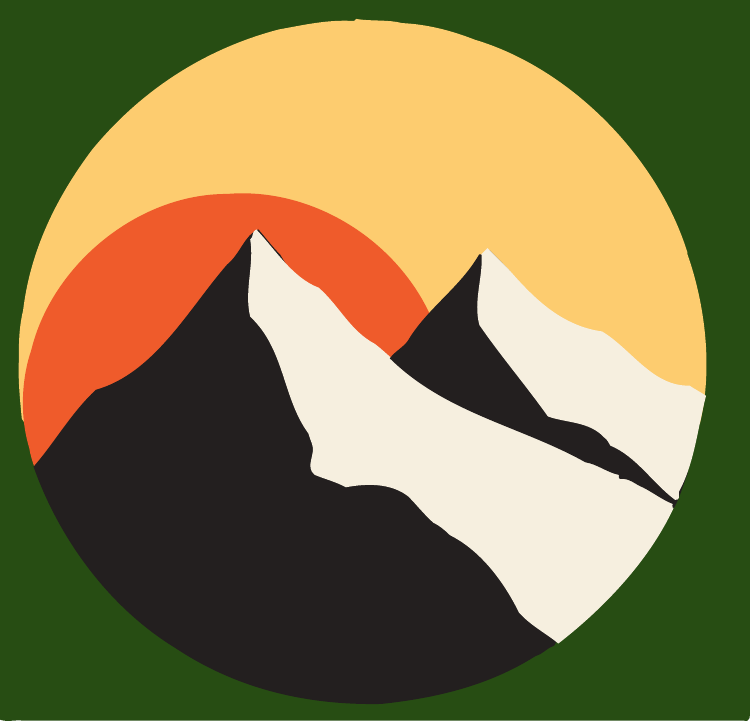}};
  \spy[red, connect spies] on (6.7, -.6) in node at (10.1,-1.1);
  \spy[red, connect spies] on (7.3, 0.25) in node at (10.1,1.1); 
  \end{tikzpicture}
\caption{[Image vectorization with depth result] (a) The given color quantized image $f$ in Figure \ref{fig:shapelayer}. (b) The proposed method's vectorization result in Figure \ref{fig: mountain shapes}. Zoomed images show good approximations to the curves, and T-junction is also well-approximated following the level line. }
    \label{fig:mountainzoom}
\end{figure}

For the experiments, in general we set curvature  functional \eqref{eq: euler elastica} parameters to be $a = 0.1$, $b = 1$ and $\epsilon = 5$, the curvature extrema \eqref{eq: discretized curvature} parameter to be $\mathcal{T}=1.25$, and \bezier curve \eqref{eq: wls fitting} fitting parameter as $\tau=1$. For most of the experiments, we set $\delta=0.05$ for depth ordering in \eqref{eq: ordering result}. If different values are used for experiments, we mention them in each experiment accordingly.  

We present more results in Section \ref{ssec:firstresults}, then show how the proposed model makes image modification easy in Section \ref{ssec:modify}, and how grouping quantization, a proposed pre-processing step, can improve vectorization of real images in Section \ref{ssec:multiquanti}. We present comparisons of our vectorization with depth against other layer-based vectorization approaches in Section \ref{ss:comparison}. In Section \ref{ss:grouping}, we further explore grouping shape layers to perform vectorization.

\subsection{Image vectorization with depth}
\label{ssec:firstresults}

\begin{figure}
    \centering
    \begin{tabular}{c|cccc}
    \multicolumn{1}{c}{(a)} \\
    \multicolumn{1}{c|}{\includegraphics[width=2.7cm]{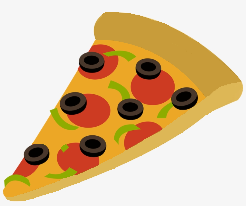}}      &  \includegraphics[width=2.7cm]{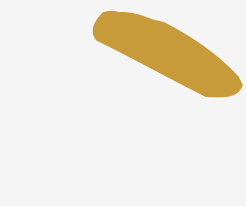} & \includegraphics[width=2.7cm]{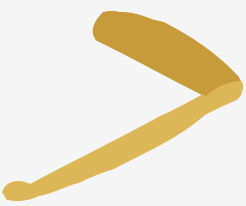} & \includegraphics[width=2.7cm]{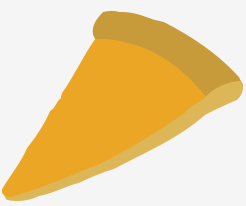} \\
    &\includegraphics[width=2.7cm]{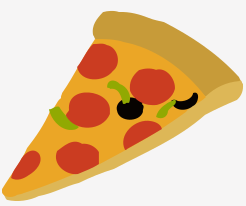}&  \includegraphics[width=2.7cm]{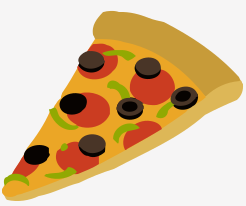}&  \includegraphics[width=2.7cm]{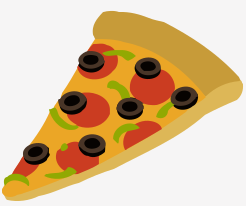}   \\   
    \multicolumn{1}{c}{(b)} \\
    \multicolumn{1}{c|}{\includegraphics[width=2cm]{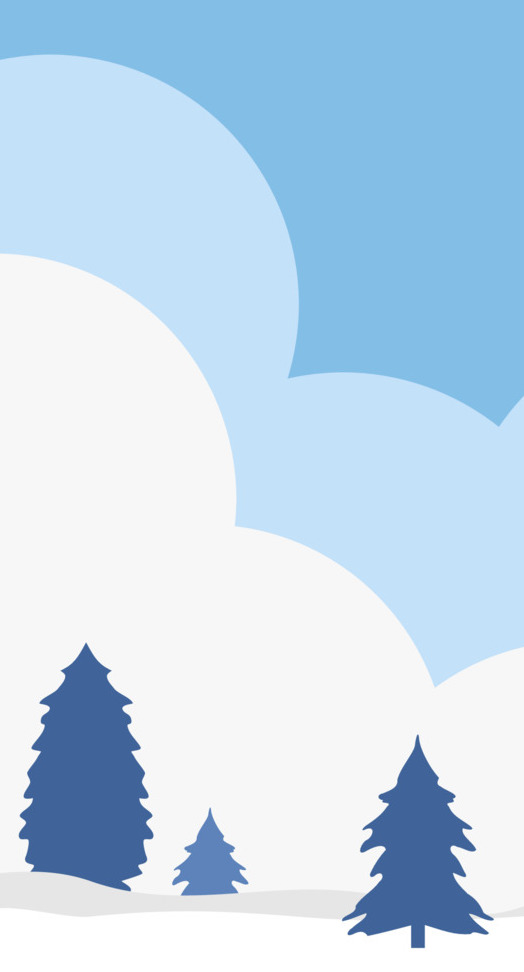}}      &   \includegraphics[width=2cm]{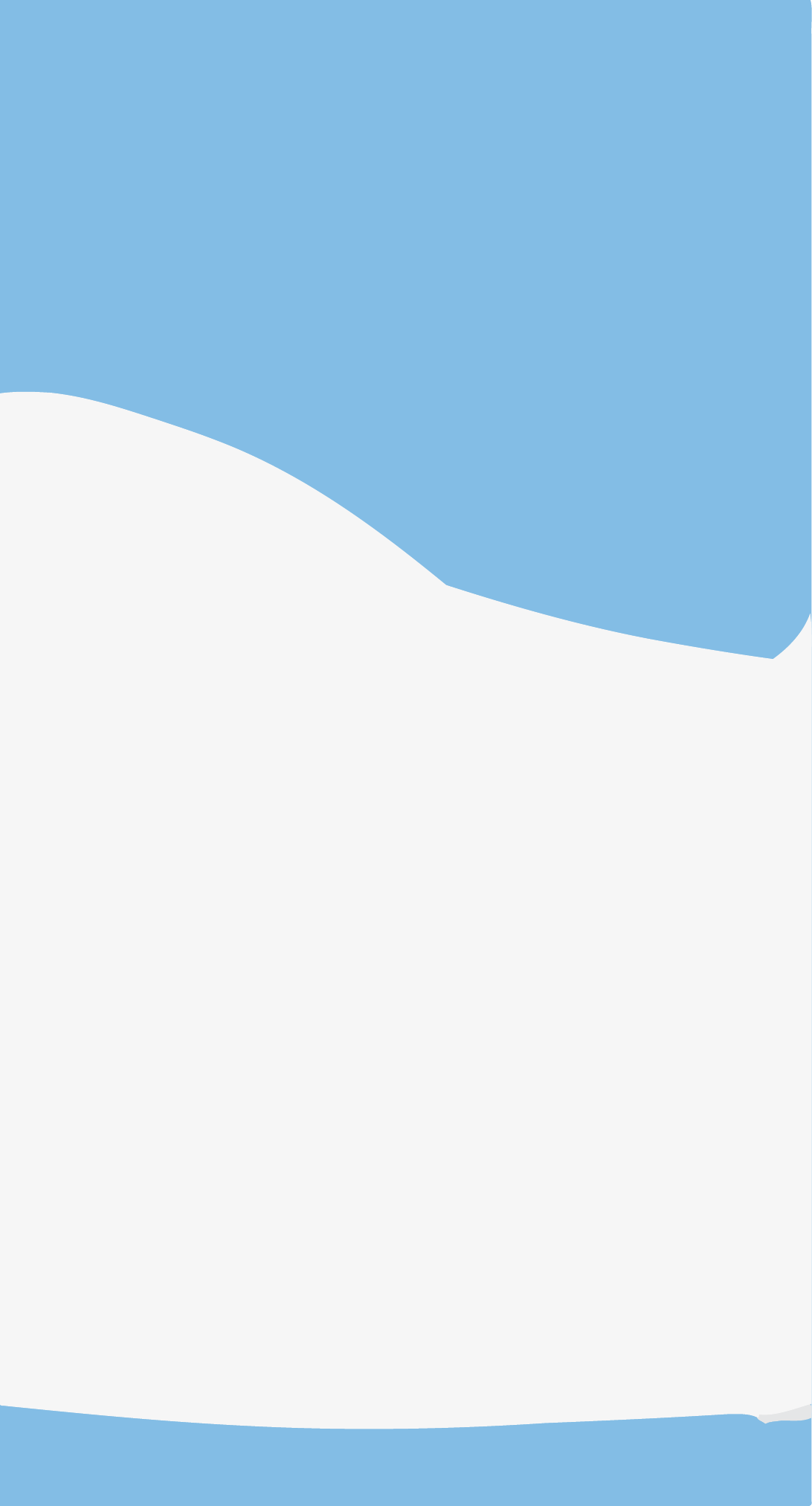} & \includegraphics[width=2cm]{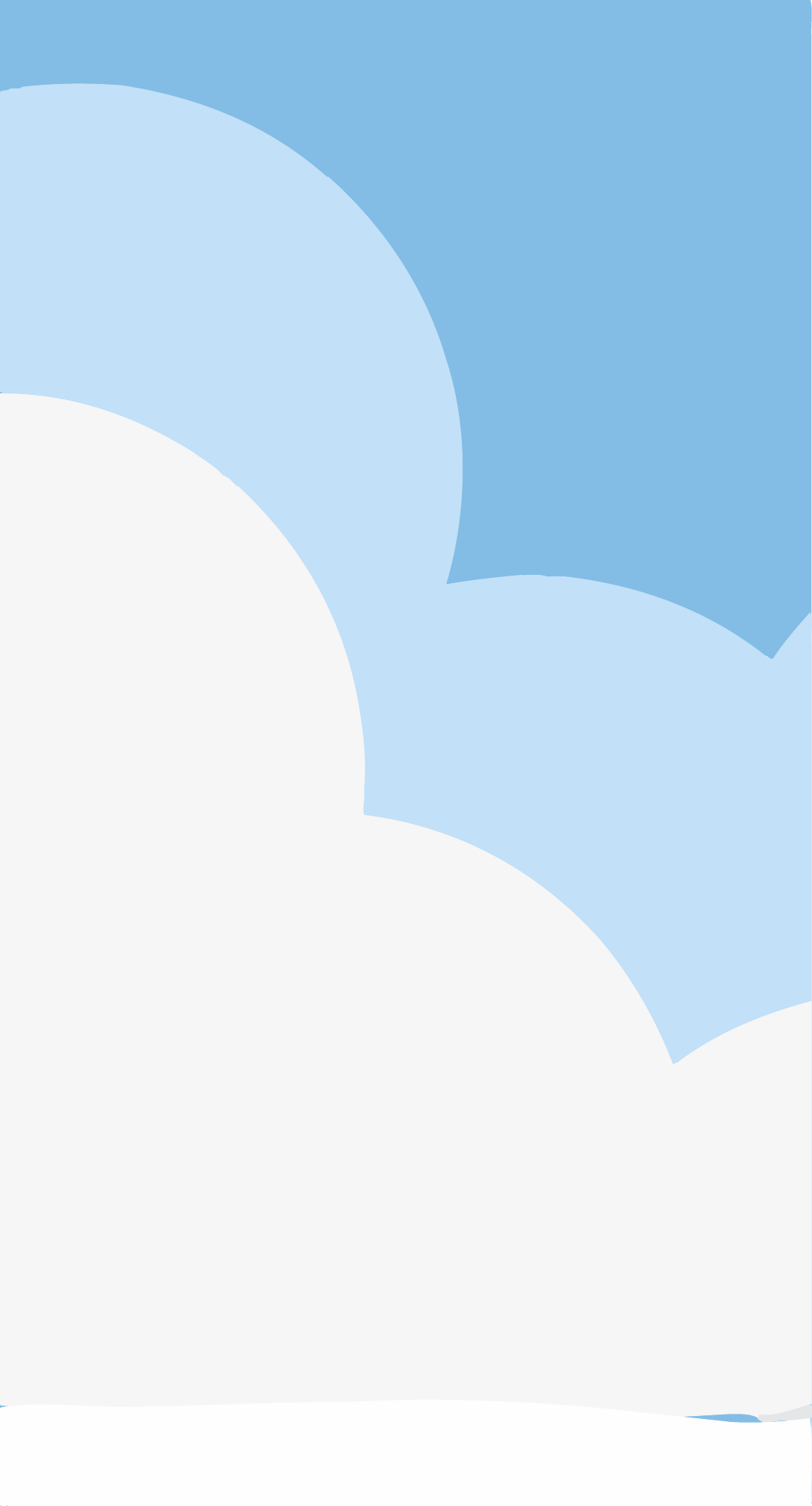} &  \includegraphics[width=2cm]{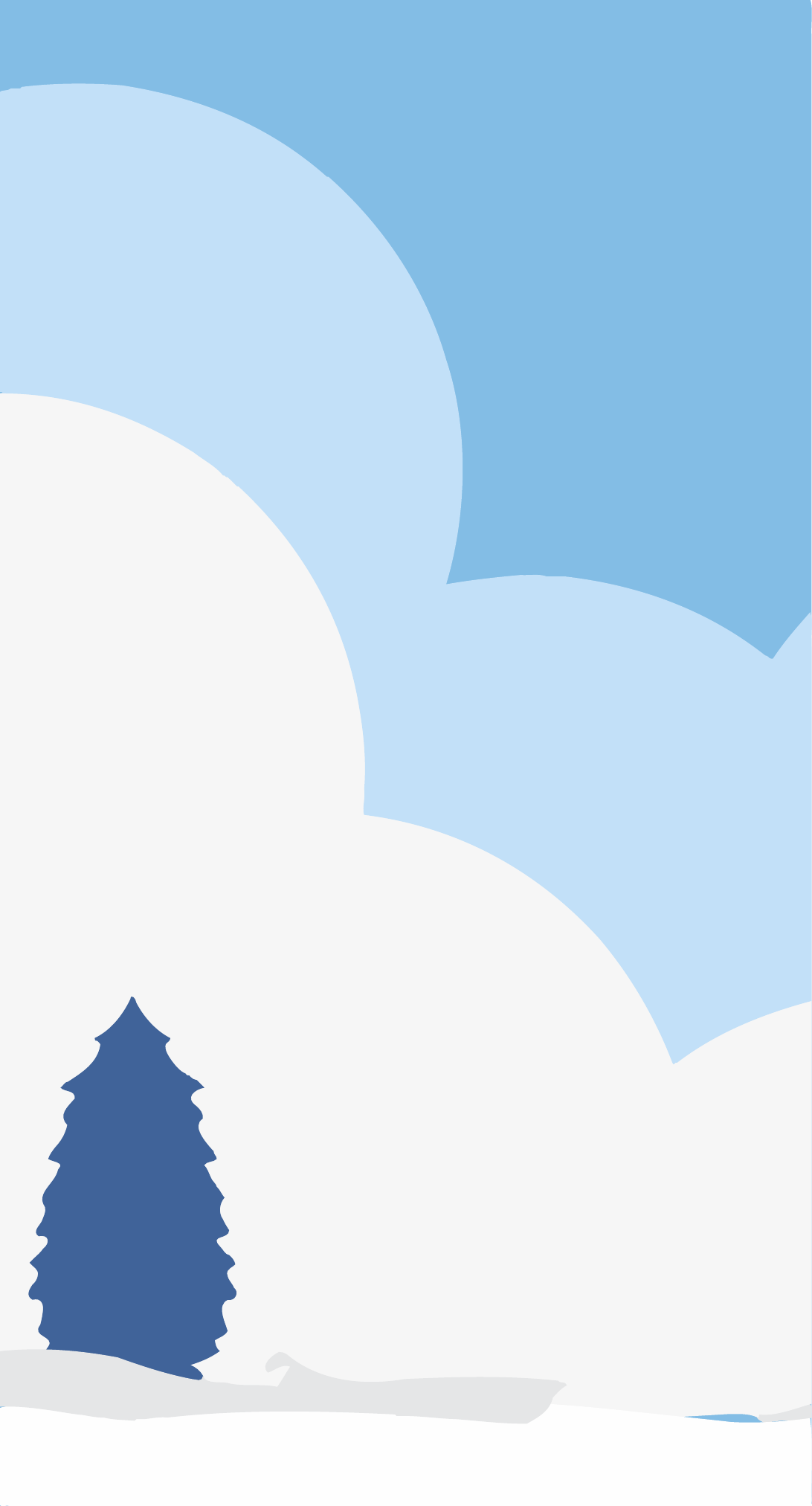}&    \includegraphics[width=2cm]{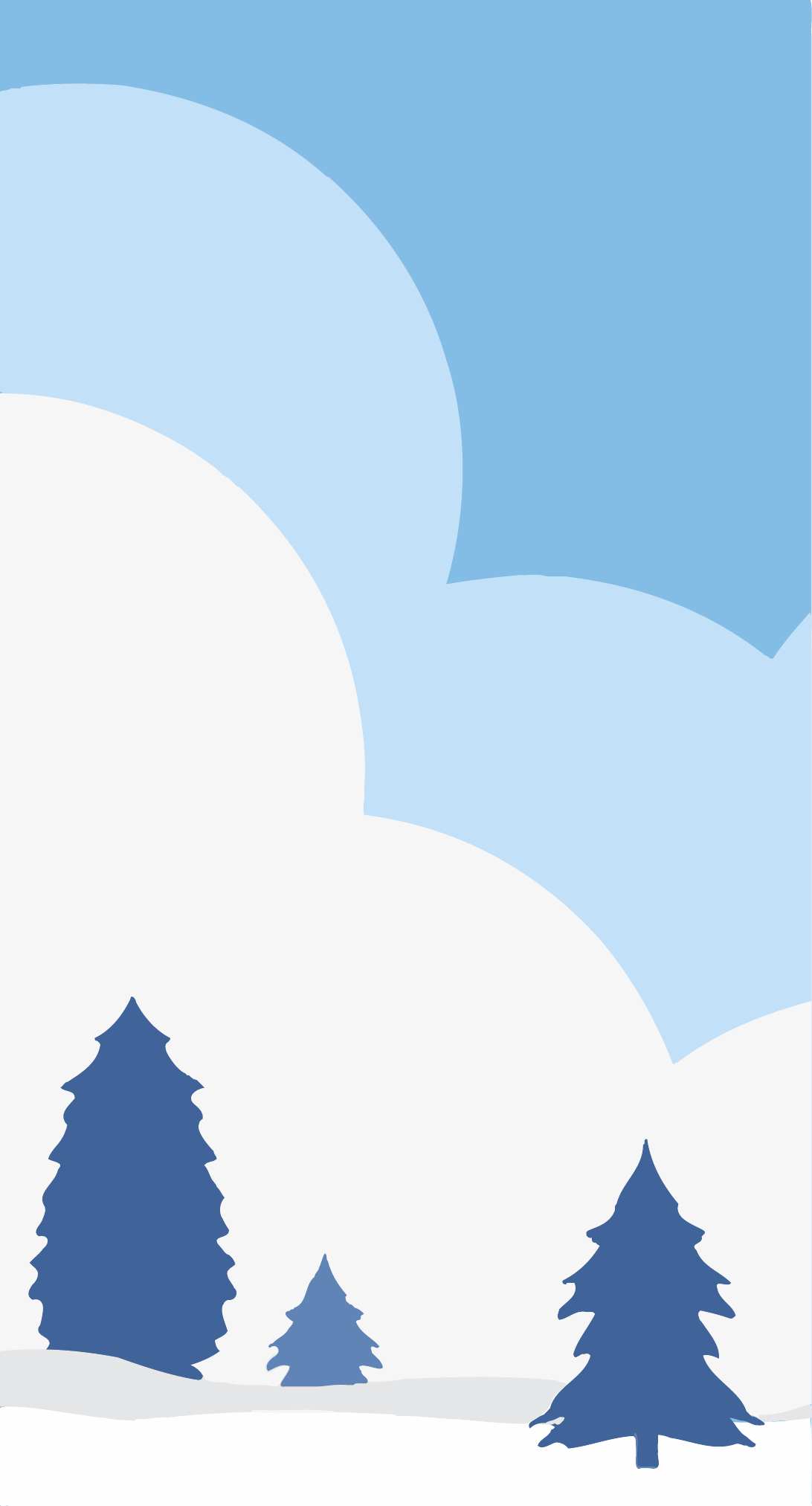}      \\ 
\multicolumn{1}{c}{(c)}\\
\multicolumn{1}{c|}{\includegraphics[width=2.7cm]{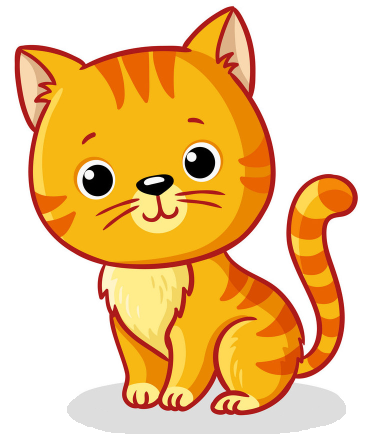}}      &  \includegraphics[width=2.7cm]{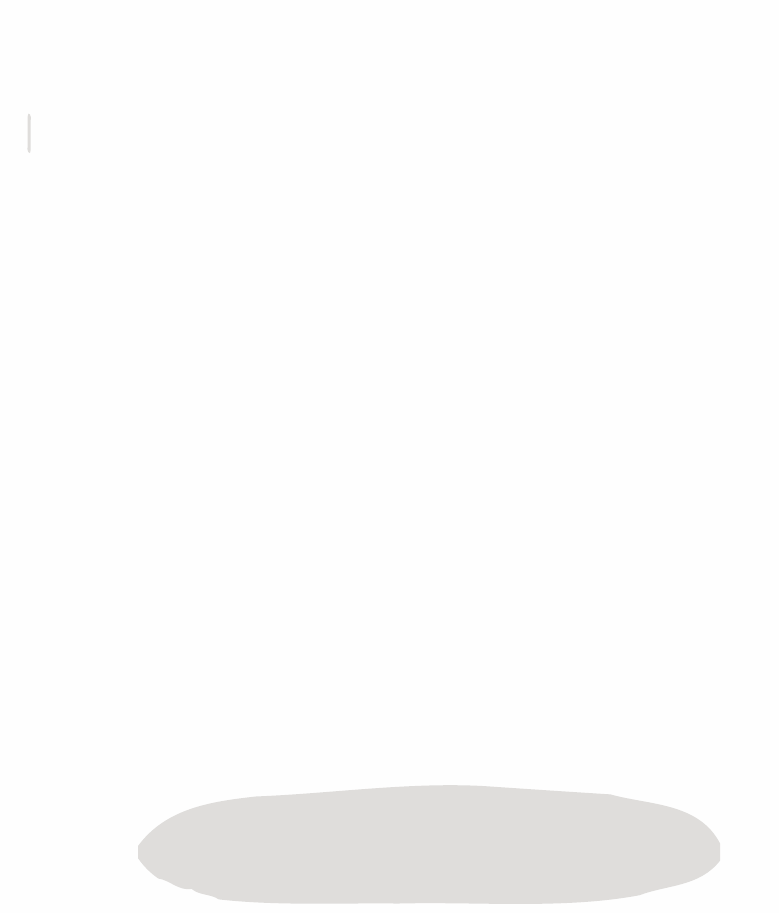} & \includegraphics[width=2.7cm]{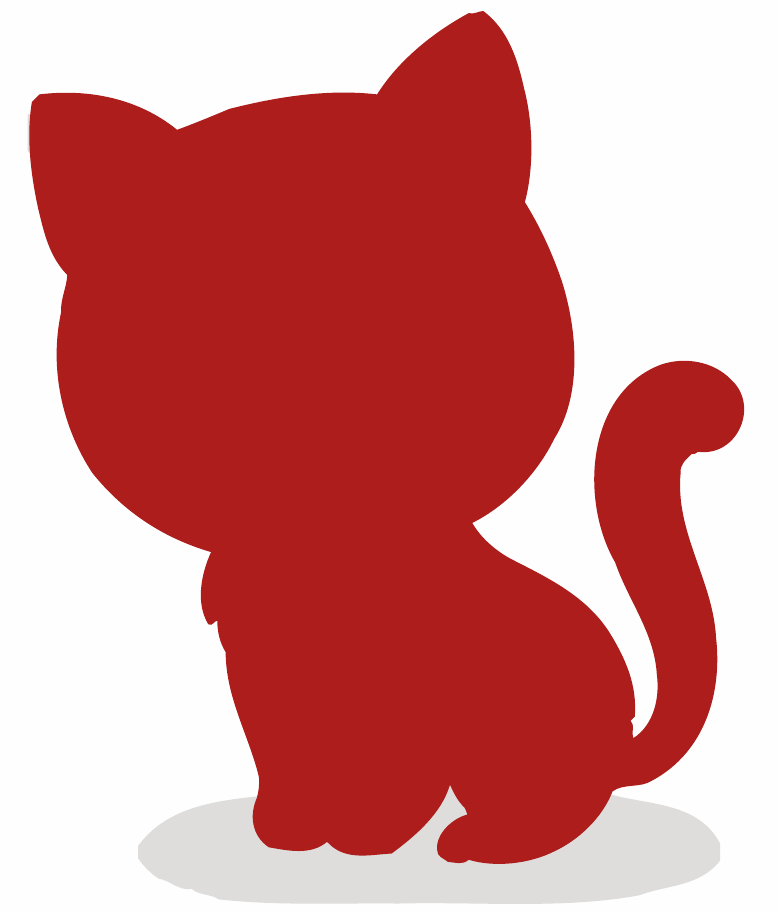} &  
\includegraphics[width=2.7cm]{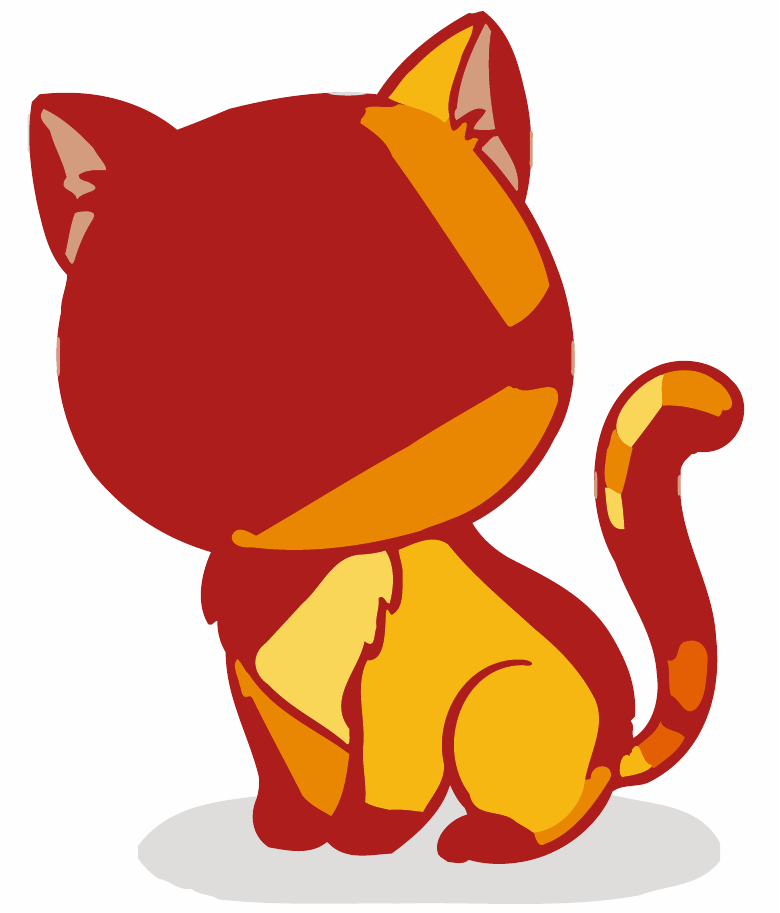}&  \includegraphics[width=2.7cm]{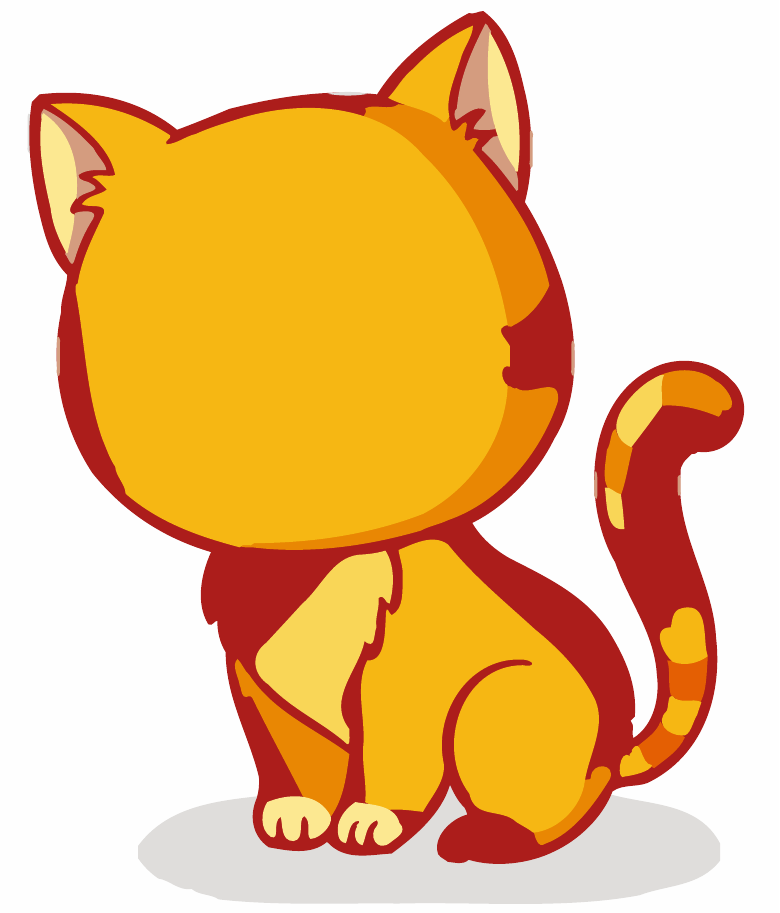}      \\ 
&
\includegraphics[width=2.7cm]{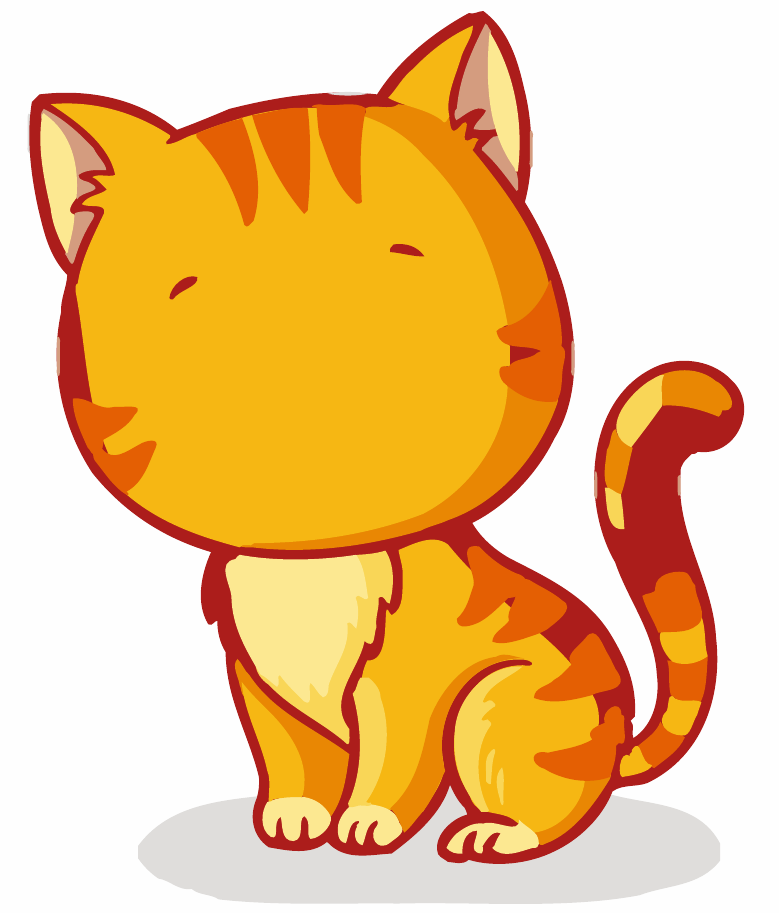} & 
\includegraphics[width=2.7cm]{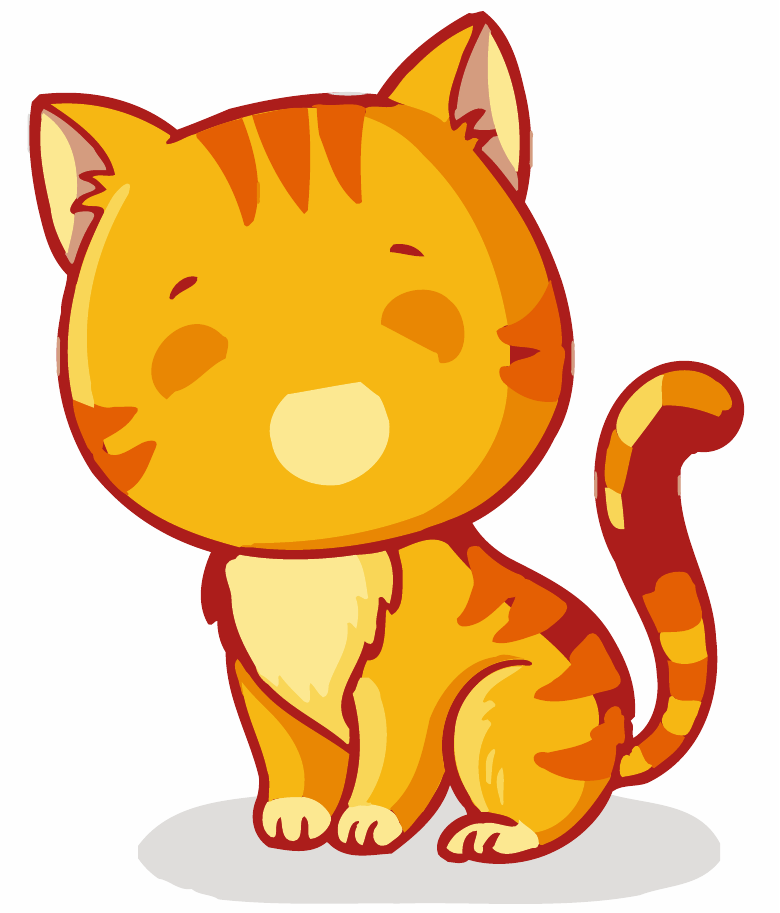} &  \
\includegraphics[width=2.7cm]{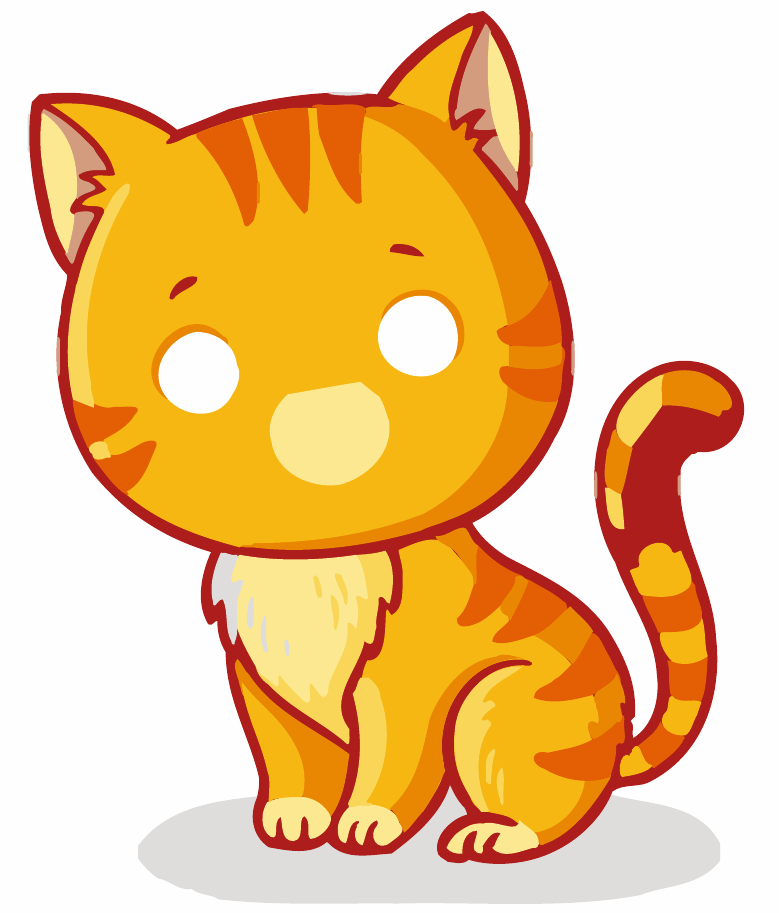}&  
\includegraphics[width=2.7cm]{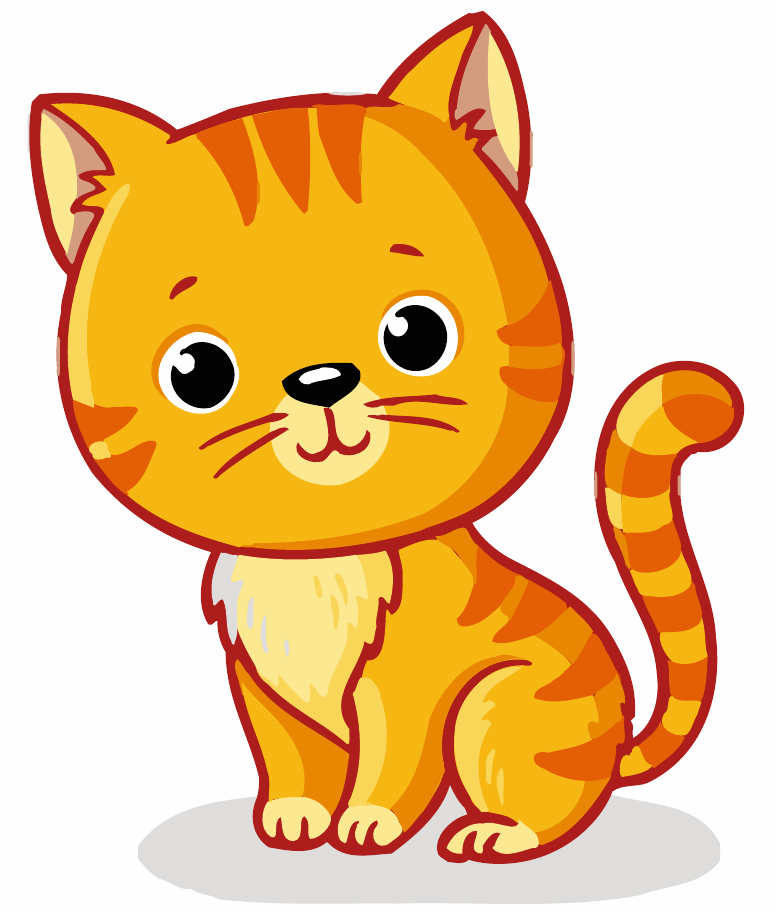}      \\ 
\end{tabular}
\caption{[Depth layering effect of the proposed method]  The first column shows the color quantized raster images $f$.  The second column to the last column present different levels of depth ordered shape layers.  We add 1 to 20 shape layers from one image to the next one. In (a) noticed occluded toppings are reconstructed as convex shapes.  In (b), the sky and the floor are covexified, then two clouds and snowy floor are on top, then trees.  In (c), the shadow is reconstructed as a convex shape, the silhouette of whole cat shown as one layer, more details are added on top, then finally the eyes and nose.  Effects of correct depth ordering and curvature-based inpainting are clear. }
\label{fig: 3results}
\end{figure}

Figure \ref{fig: 3results} shows three results presented as in the bottom row of Figure \ref{fig: mountain shapes}, that from the second column to the last column (from left to right and first row to the second row), it shows the stacking from the bottom shape layer to the top layer. Unlike Figure \ref{fig: mountain shapes} where we add one layer at a time, we add multiple shape layers from one image to the next for more concise presentation. 

The top row (a) shows the pizza image, where layers are added in the order of the crust, the side of pizza, the bottom part of the pizza, then various toppings following the depth ordering. 
In the second row (b), landscape example shows a snowy background with a few trees. The proposed method first finds the blue sky and a blue floor in the first column, and then a couple of clouds, one light blue and one white, shown in the second column. Then a tree and some finer details are on top. The final result on the very right shows the final vectorization of the given image in the left. 
In the third cartoon cat example on row (c), our depth ordering identifies the shade as the one of the bottom layer and constructs it as an ellipse shape by the curvature inpainting convexification step. A thin dark orange stroke around the cat is identified as one shape layer, filled in by the curvature inpainting showing the total silhouette of the cat.  Each additional shape layer adds more details to the cat and the final vectorization is very close to the given image.  Effects of correct depth ordering and curvature-based inpainting is clear. 
For the first row, we used $\delta = 0.01$, $b=5$ and $\epsilon = 10$; for the second row, $\delta = 0.1$ and $\epsilon = 10$; for the last row, $\delta = 0.05$, $\epsilon = 10$ and $b = 0.25$, while we use the general parameters for the rest.

\subsection{Easy editing by SVG file modification}
\label{ssec:modify} 

A proper depth ordering and convexification step give a huge advantage of our model that it makes editing and modification easy and natural.  We present the results of the proposed image vectorization with depth in Figure \ref{fig: modified pizza} (a) pizza, (e) snowy landscape, (g) cartoon cat, and (i) painting. Figure \ref{fig: modified pizza} (a), (e) and (g) are the same vectorized results in Figure \ref{fig: 3results}. In Figure \ref{fig: modified pizza}, to the right of the vertical separators are variations from the vectorization showing various modifications.  In (b), we simply delete all toppings that are all identified to be above the pizza. Then, we add new toppings to (b) which yields (c) and (d).
For the snowy landscape (e), we remove the cloud, and add a green mountain with one of the original tree relocated on top of it in (f). 
For the cartoon cat (g), first we only remove some shapes on the cat's face and body, then craft a bowtie on its neck in (h). Without any other manual modifications, we can easily make drastic changes to the input raster image as from (g) to (h). 
Figure \ref{fig: modified pizza} (i) shows the vectorized output of a famous artwork by Henri Matisse. After vectorization, we can easily reposition certain elements. For instance, we move the guitar-like object and reposition the separated white finger back to the palm to form a complete hand. We rearrange the green patterns on the black body and remove some yellow leaves in (j). 

\begin{figure}
\centering
\begin{tabular}{c|ccc}
(a) & (b) & (c) & (d) \\
\includegraphics[width = 3cm]{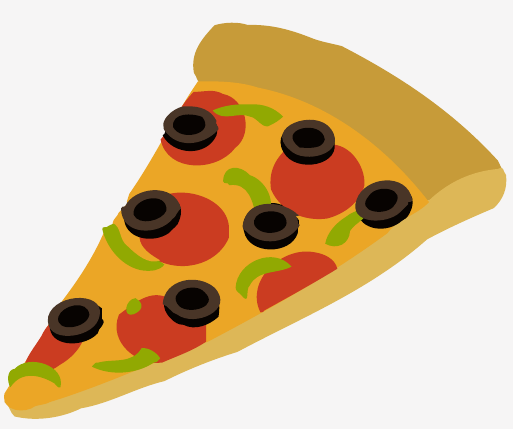} &
\includegraphics[width = 3cm]{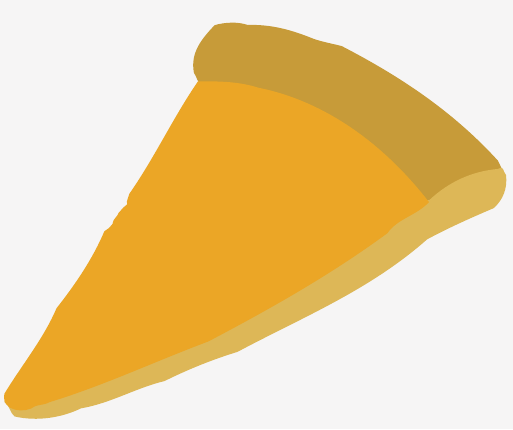} & 
\includegraphics[width = 3cm]{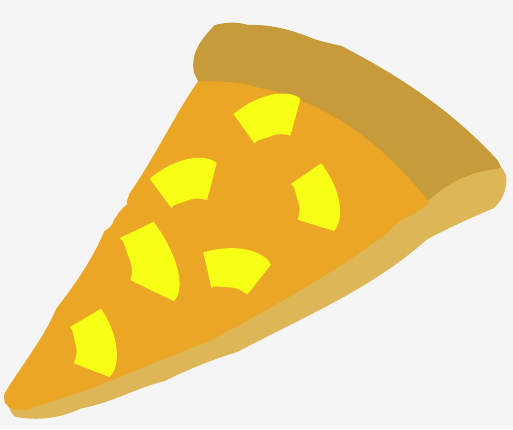} &
\includegraphics[width = 3cm]{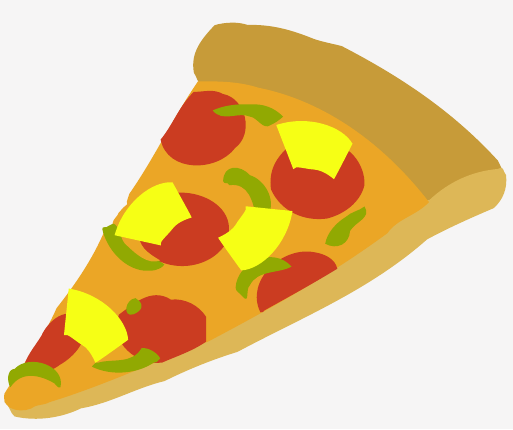} 
\end{tabular}
\begin{tabular}{c|cc|c}
(e) & (f) & (g) & (h) \\
\includegraphics[width=2.5cm]
{pdf/new_winter/interpolation/new_winter5.pdf} & \includegraphics[width=2.5cm]{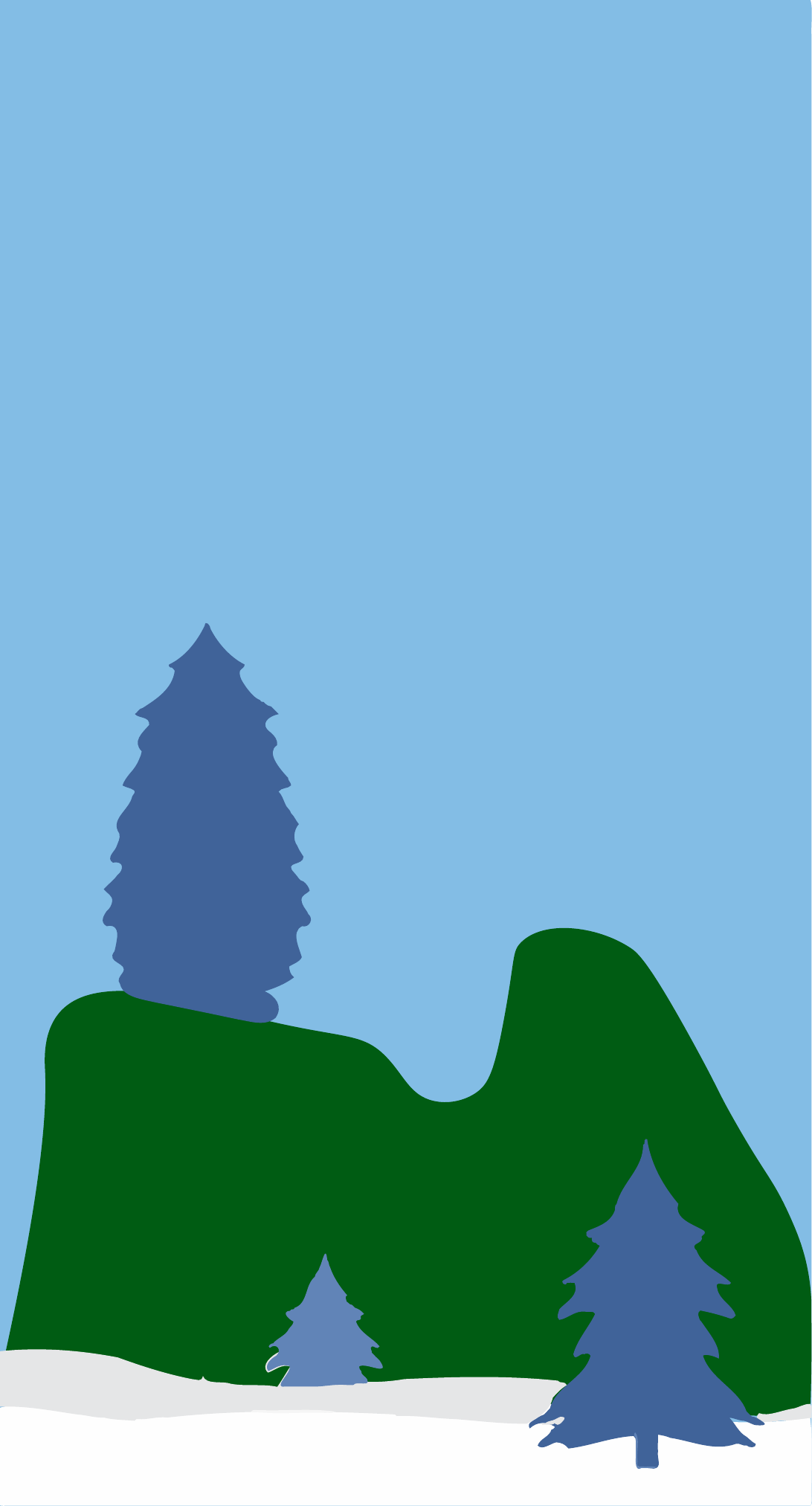} & \includegraphics[width=3cm]{pdf/cat/interpolation/cat11.pdf}& \includegraphics[width=3cm]{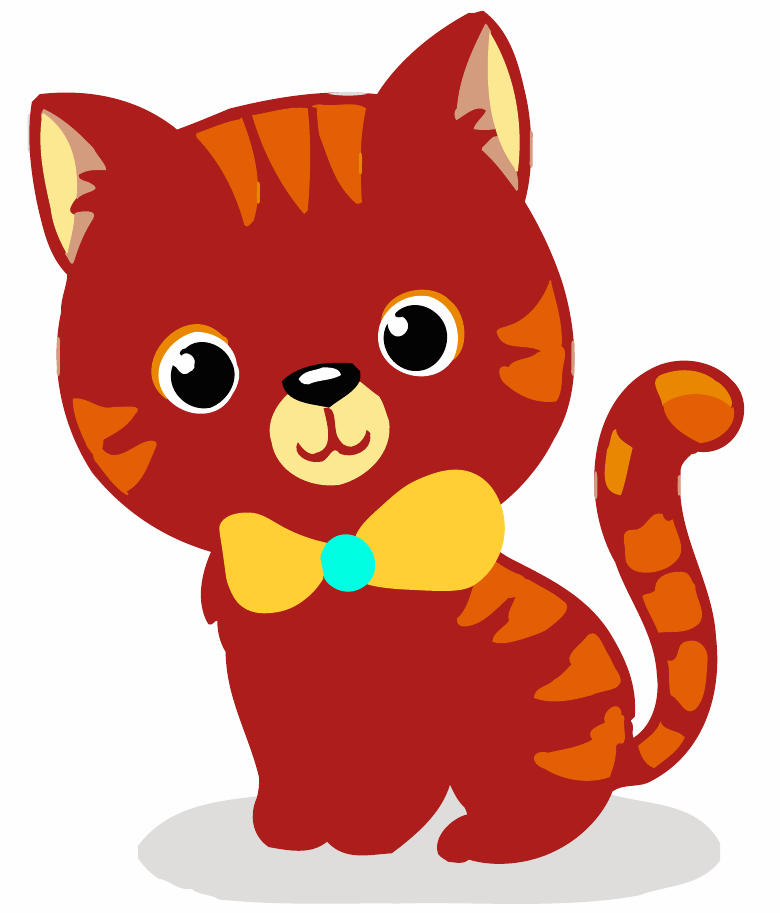}
\end{tabular}
\begin{tabular}{c|c}
(i) & (j)  \\
\includegraphics[width=7cm]{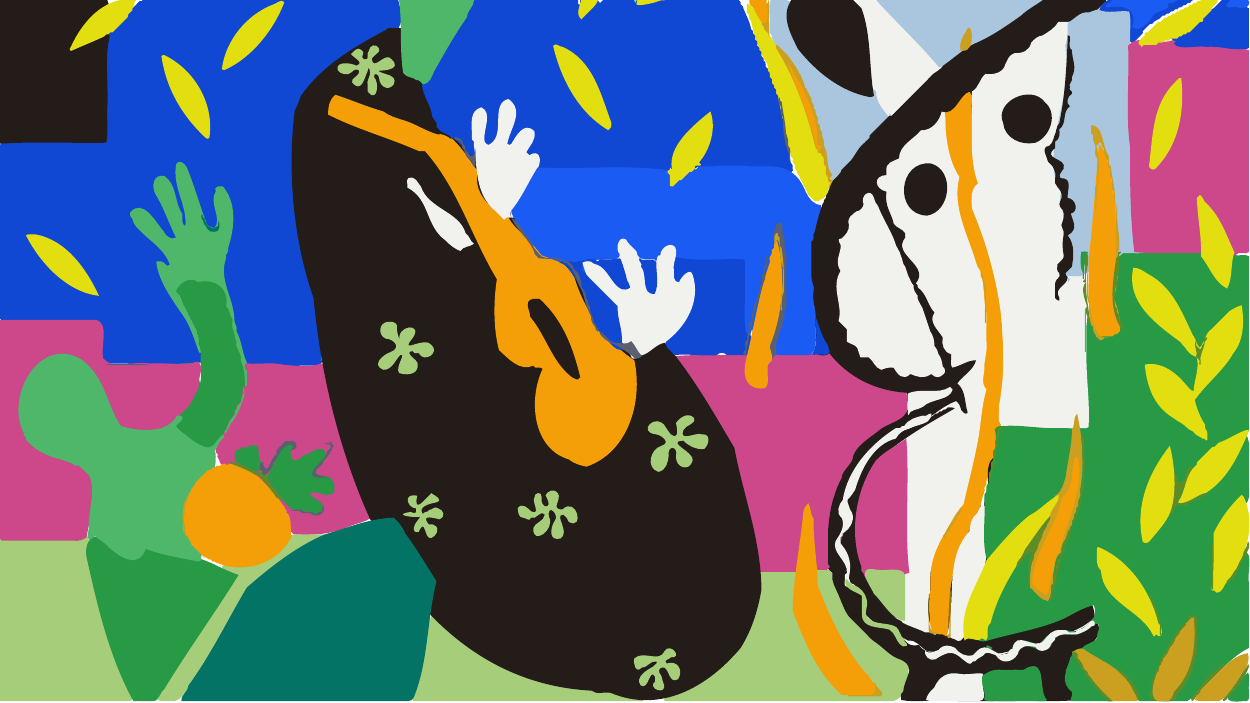} &
\includegraphics[width=7cm]{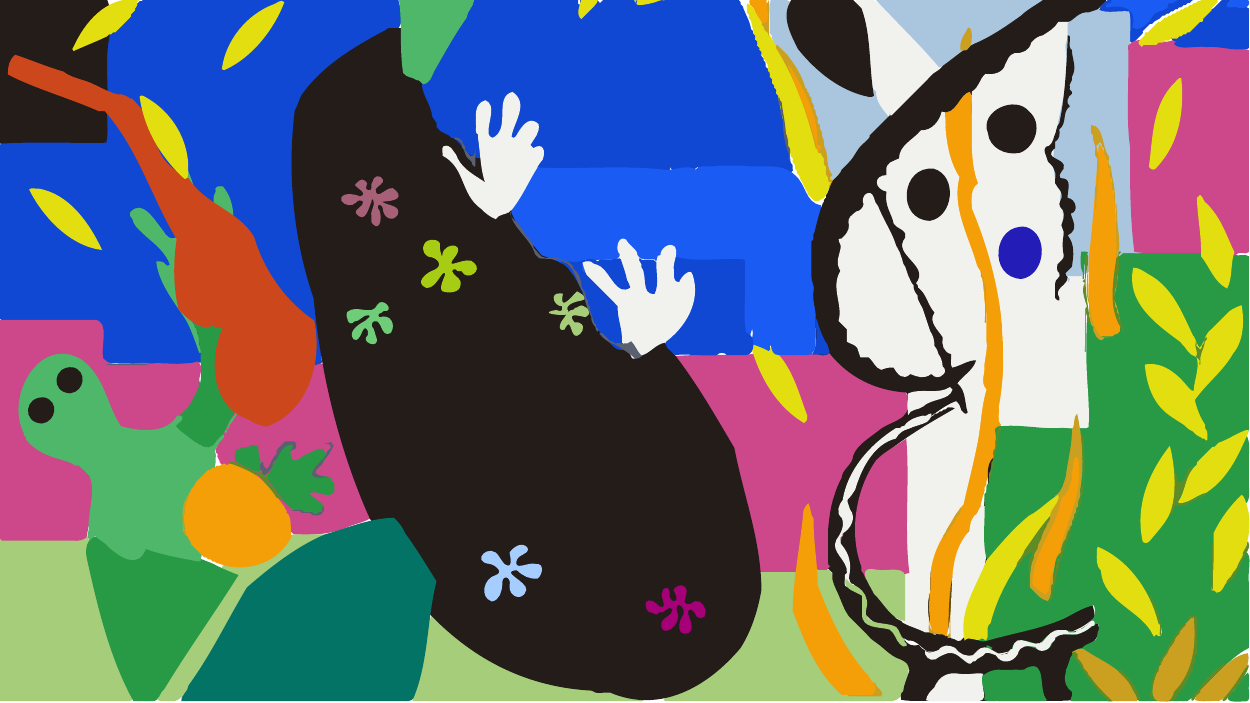}
\end{tabular}
\caption{[Ease of Editing by the proposed method] Image vectorization with depth on the left (a), (e), (g), and (i). To the right of the separators, possible modifications are presented in (b), (c), (d), (f), (h), and (j). Some shape layers are deleted, moved or new layers are added. With convex shape layers, modifications are natural and easy. }
\label{fig: modified pizza}
\end{figure}

Since the shape layers boundaries become more regular after curvature inpainting, editing becomes natural and easy. We demonstrate this by comparing with a typical vectorization in Figure \ref{fig: transformation}.  We present the proposed method's vectorization result in the first row, and Adobe Illustrator's~\cite{adobeillustrator} in the second. 
We apply identical deletion, translation and rotation to shape layers for both methods. In the first column, small mountain is removed, and for the proposed method, it shows the green background and the yellow sky are convexified, while Adobe Illustrator~\cite{adobeillustrator} leaves a blank region of the mountain shape.
In the second column, the orange sun is translated and rotated from the first column results, and in the third column, the yellow sky is removed from the second column image.  Notice that typical vectorization gives many blank regions, while the proposed method gives convexified shape layers in the back. In addition, typical vectorization does not have depth information.  This shows the convenience of using our model to modify vectorized outputs. 

\begin{figure}
\centering
\begin{tabular}{ccc}
   (a)  & (b) & (c) \\
    \includegraphics[width = .25\textwidth]{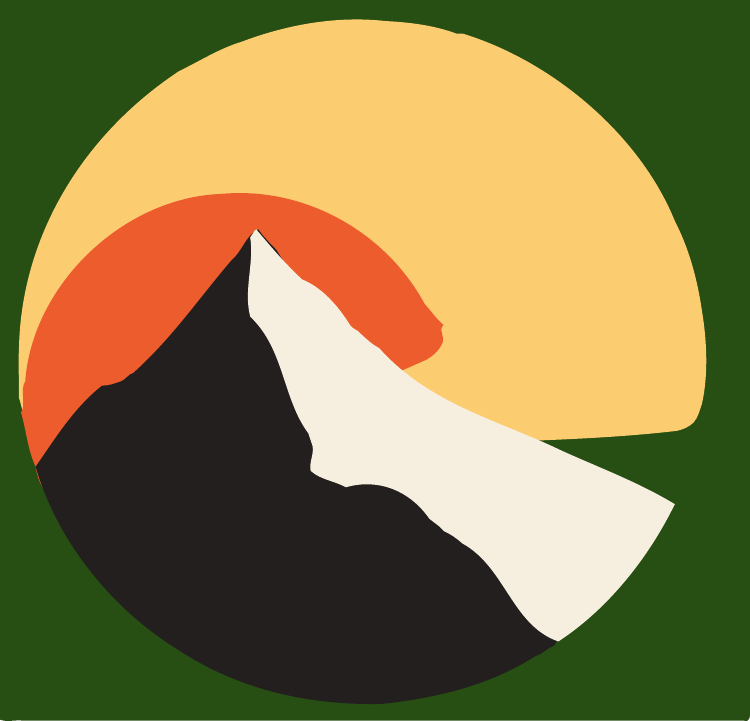}
&    \includegraphics[width = .25\textwidth]{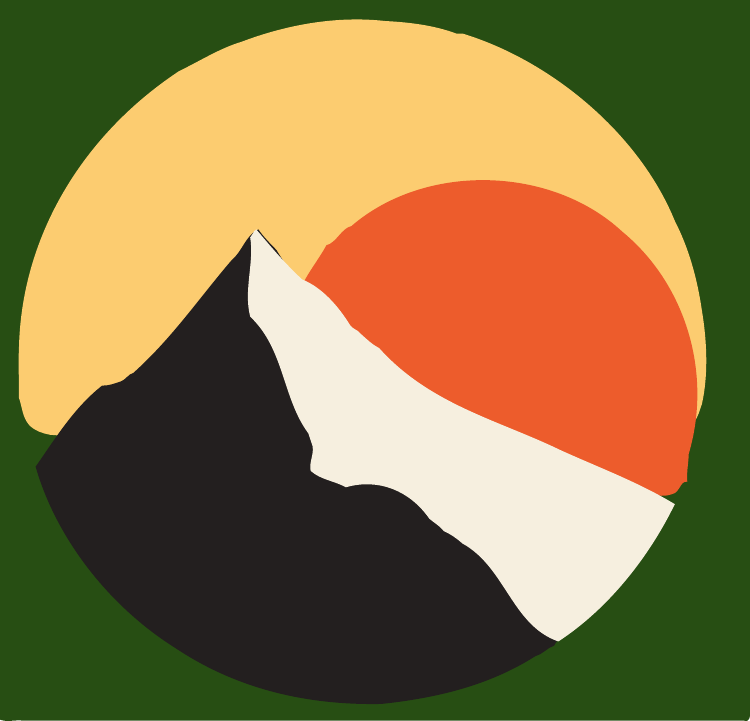}
&    \includegraphics[width = .25\textwidth]{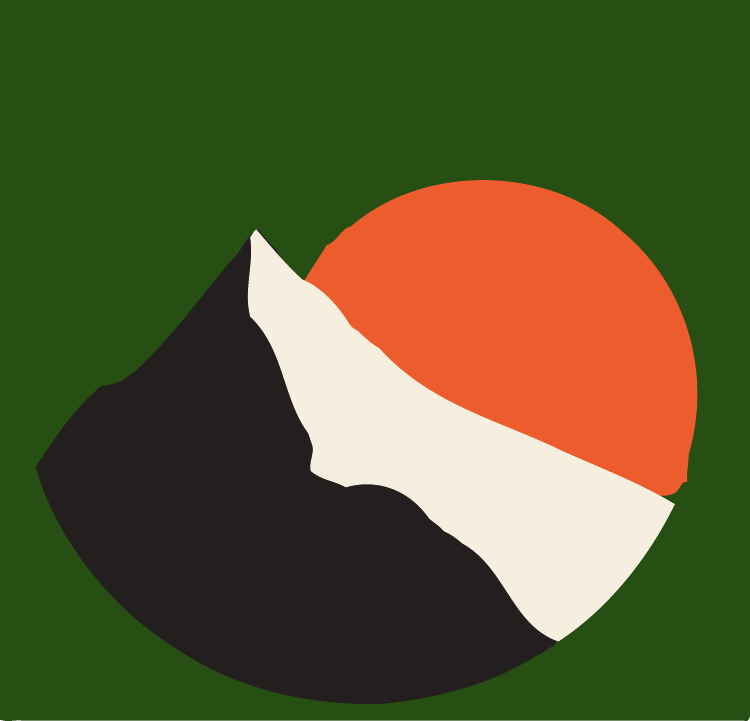} \\
   (d)  & (e) & (f) \\
    \includegraphics[width =.25\textwidth]{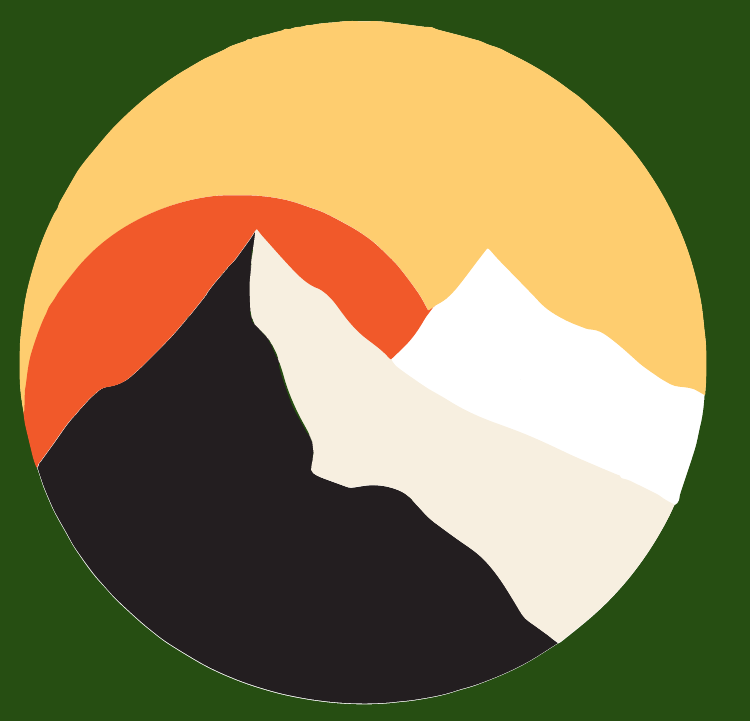}
&
    \includegraphics[width = .25\textwidth]{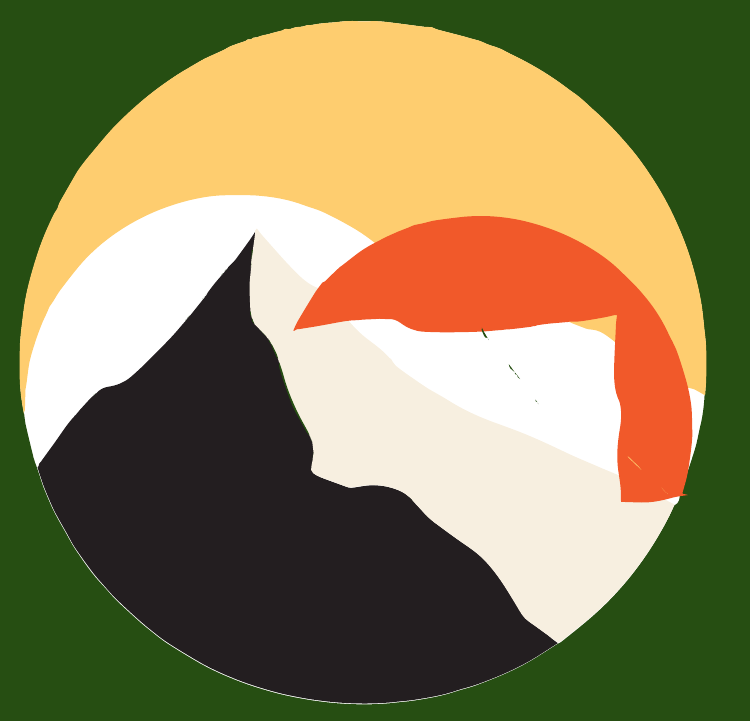}
&     \includegraphics[width = .25\textwidth]{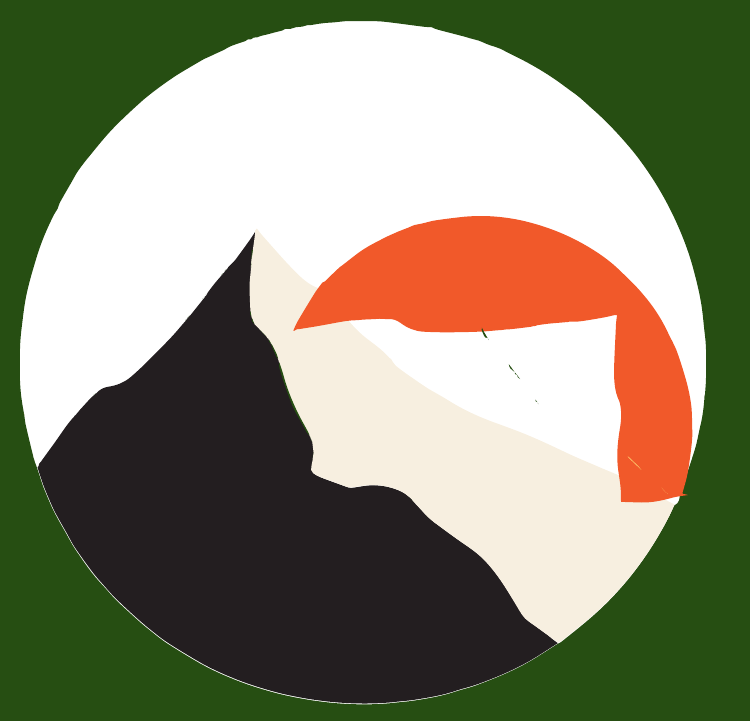}
\end{tabular}
\caption{[Comparison of editing effect] The top row is image vectorization with depth, and the second row is a typical vectorization~\cite{adobeillustrator}.  In the first column, small mountain is removed; in the second, the orange sun is translated and rotated; and in the third, the yellow sky is removed. We applied the same transform for each column. Convexified shape layers makes modification easy and more natural.}
\label{fig: transformation}
\end{figure}

\subsection{Grouping quantization and updated shape layer set} \label{ssec:multiquanti}

The proposed model starts with the given color quantized raster image $f$ which defines the shape layers in Definition \ref{def:shapelayer}. For images with complex contrast or color gradient, such $K$-mean color quantization inclines to partition regions of similar color into smaller regions, which gives little meaning to semantic understanding of object.

We explore adding grouping quantization as a pre-processing to give more semantic vectorization. This is simply done by modifying the shape layer set $\mathcal{S}$, i.e. adding a few grouped regions of similar colors. The main idea is to obtain a coarser segmentation which is less sensitive to color gradient or brightness, and combine it with the finer details from $K$-mean color quantization. 
Given a color quantized image $f: \Omega \to \{ c_{\ell} \}_{\ell=1}^{K}$, with $K$ different colors,  we segment $f$ into $\bar{k}$ phases, where $\bar{k} \ll K$. To do so, we use unsupervised segmentation~\cite{sandberg2010segmentation} which minimizes 
\begin{equation}\label{eq: unsupervised seg}
    E[\bar{k}, \bar{\phi}_i, \bar{c}_{i} | f] = \mu \left( \sum_{i=1}^{\bar{k}} \frac{P(\bar{\phi}_{i})}{|\bar{\phi}_{i}|} \right) \sum_{i=1}^{\bar{k}} P(\bar{\phi}_{i}) + \sum_{i=1}^{\bar{k}}|f - \bar{c}_{i}|^{2}\chi_{\bar{\phi}_{i}}
\end{equation}
where $\bar{\phi}_{i} \subset \Omega$ is a region  representing each cluster phase, and $\chi_{\bar{\phi}_{i}}: \Omega \rightarrow \{0,1\}$ gives the characteristic function of $\bar{\phi}_{i}$,  $f$ is the color quantized input image, $\bar{c}_{i}$ is the average color value in each phase, i.e., $\bar{c}_{i}= \int_\Omega f \chi_{\bar{\phi}_i} dx$, $P(\bar{\phi}_{i})$ is the perimeter of $\bar{\phi}_{i}$ and $| \bar{\phi}_i|$ is the total area of $\bar{\phi}_i$.  We follow the algorithm in~\cite{sandberg2010segmentation}: the minimization process iterates through each pixel to determine if each pixel should be labeled as another existing segmented phase, be labeled as a new one, or remain unchanged. For details of the minimization process, readers are referred to~\cite{sandberg2010segmentation}. In the experiments, we let $\mu$ to be chosen from $0.5$ to $1$, and cap the total number of phases in  \eqref{eq: unsupervised seg} to be less then $6$, such that once $\bar{k}$ reaches $6$, we only let the pixels to move to existing phases or stay in the current phase, and no more new phase is created. 

This pre-processing step gives a more semantic segmentation.  We add these new phases to the set of shape layers $\mathcal{S}$ while removing some of redundant small regions. 
From the phases $\bar{\phi}_i$ for $i=1, 2, \dots, \bar{k}$ with $\bar{k}\leq 6$ from minimizing  \eqref{eq: unsupervised seg}, and we let $\bar{\phi}^j_i$ be each disjoint connected region of phase $\bar{\phi}_i$, i.e., $\bar{\phi}_i = \bigcup_{j=1}^{\bar{N}_i} \bar{\phi}^j_i$ where  $\bar{N}_i$ is the number of connected components in $\bar{\phi}_i$.
For each  $\bar{\phi^j_i}$, we assign a color $\widehat{c}^j_{i} \in \{c_l \}_{l=1}^{K}$ such that 
\[ 
\widehat{c}^j_{i} = \arg \max \mathcal{H}(f(\bar{\phi}^j_i)) 
\]
where $\mathcal{H}$ is the histogram.  We find the color $\widehat{c}^{j}_{i}$ for each connected component $\bar{\phi}^j_i$. This picks the color which appears the most frequently within $\bar{\phi}^j_i$ among $\{c_l\}_{l=1}^K$.  This is different from using $\bar{c}_i$ in  \eqref{eq: unsupervised seg} which is the average computed among all $\bar{\phi}^j_i$ for $j=1,\dots, \bar{N}_i$.  We allow each disjoint connected component $\bar{\phi}^j_i$ to have a different color $\widehat{c}^j_{i}$.  
Let the set $\mathcal{P}$ be the collection of these new shape layers $\mathcal{P}=\{\bar{\phi}^j_i \mid i=1,2,\dots, \bar{k} \;\;  \text{and} \;\; j=1,2, \dots, \bar{N}_i \}$ where each element is associated with a color $\widehat{c}^j_i$ respectively. Among the shape layers given from $f$ in $\mathcal{S}$, if there are shape layers $S_i\in \mathcal{S}$ which is (i) a subset of $\bar{\phi}^j_i$ (ii) with the same associated color $\widehat{c}^j_{i}$, we collect them in a set $\mathcal{S}_R$ and remove these from the set $\mathcal{S}$.  The shape layers in $\mathcal{S}_R$ are redundant in a sense that they are a part of shape layers in $\mathcal{P}$ but smaller sized regions.  We update the shape layer set $\mathcal{S}$ given from $f$, by adding grouping quantization shape layers $\mathcal{P}$ and removing redundant shape layers $\mathcal{S}_R$:
\[ \mathcal{S}_{new} = ( \mathcal{S} \cup \mathcal{P}) \backslash \mathcal{S}_R. \]
This new shape layer set $\mathcal{S}_{new}$ is used in our vectorization process instead of $\mathcal{S}$, and we proceed to depth ordering and convexification. The pseudo code of this step can be found in Appendix \ref{Asec:codes}.

\begin{figure}
\centering
\begin{tabular}{ccc}
(a) & (b) & (c) \\
\includegraphics[width=5cm]{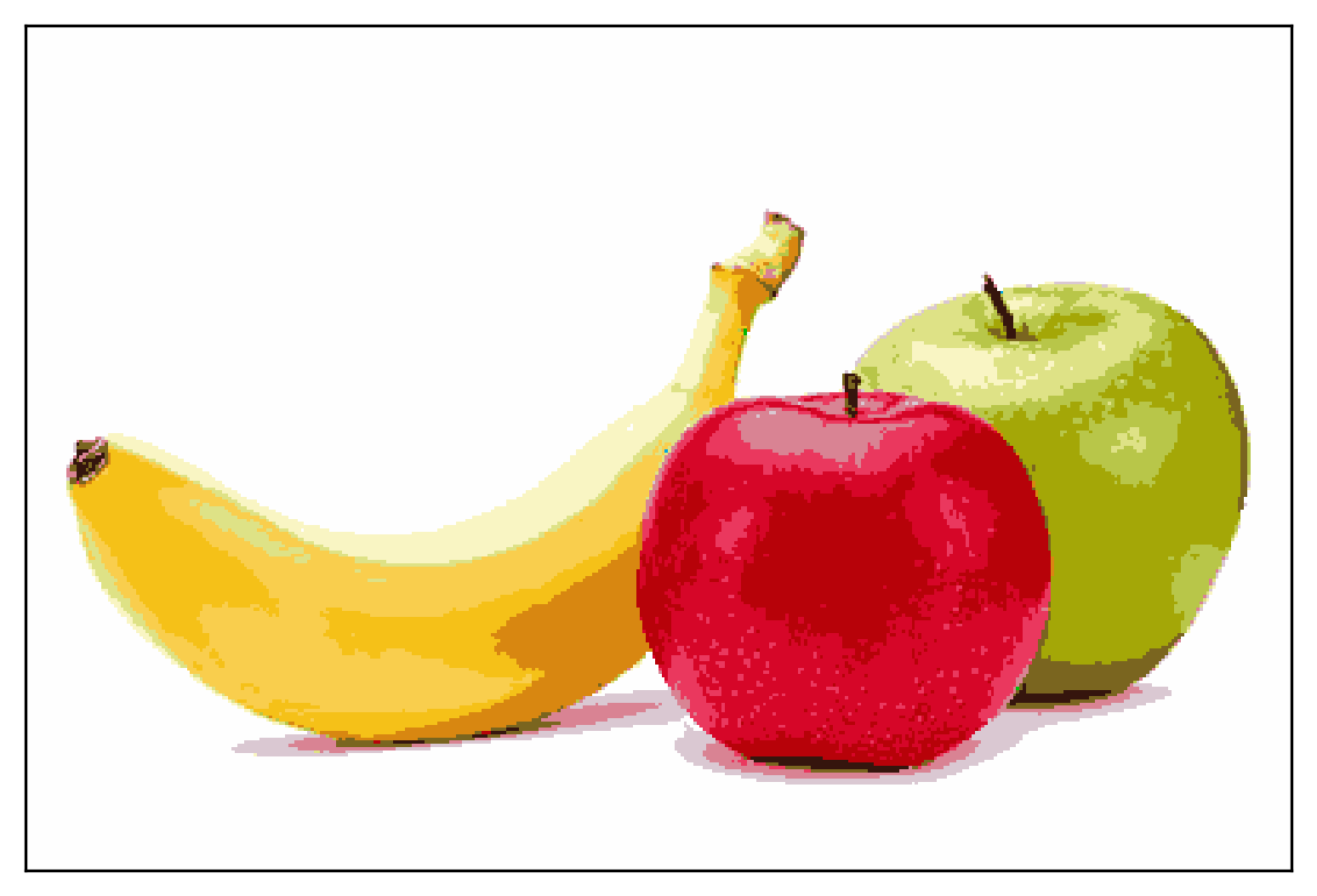}& \includegraphics[width=5cm]{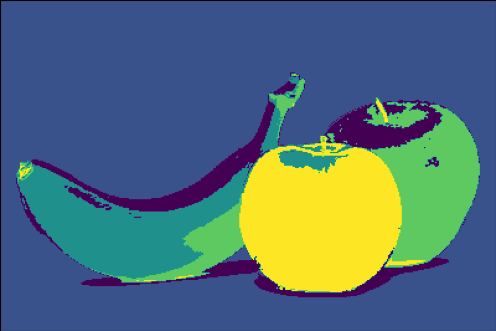} & 
\includegraphics[width=5cm]{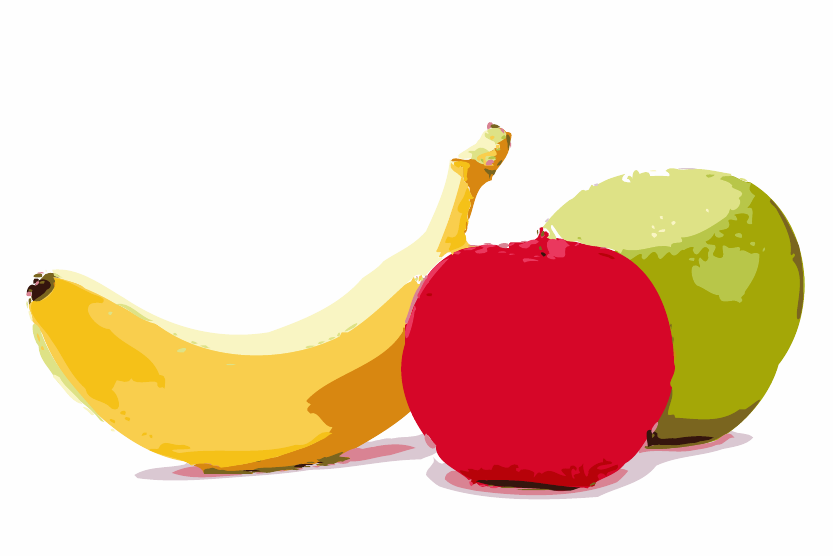} \\
\end{tabular}
\begin{tabular}{cc}
(d) & (e) \\
\begin{tikzpicture}[spy using outlines={rectangle,red,magnification=4,size=2.2cm, connect spies}]
\node (with MQ) at (0,0)   {\includegraphics[width=5cm]{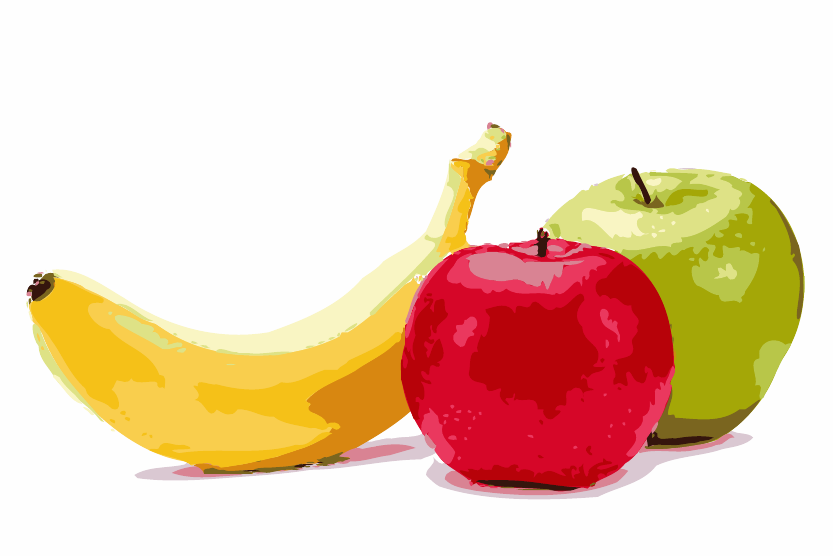}};
\spy[red, connect spies] on (0.75,0.2) in node at (3.5,1.1);
\spy[red, connect spies] on (-2.25,-0.1) in node at (3.5,-1.1);
\end{tikzpicture} &
\begin{tikzpicture}[spy using outlines={rectangle,red,magnification=4,size=2.2cm, connect spies}]
\node (without MQ) at (0,0){\includegraphics[width=5cm]{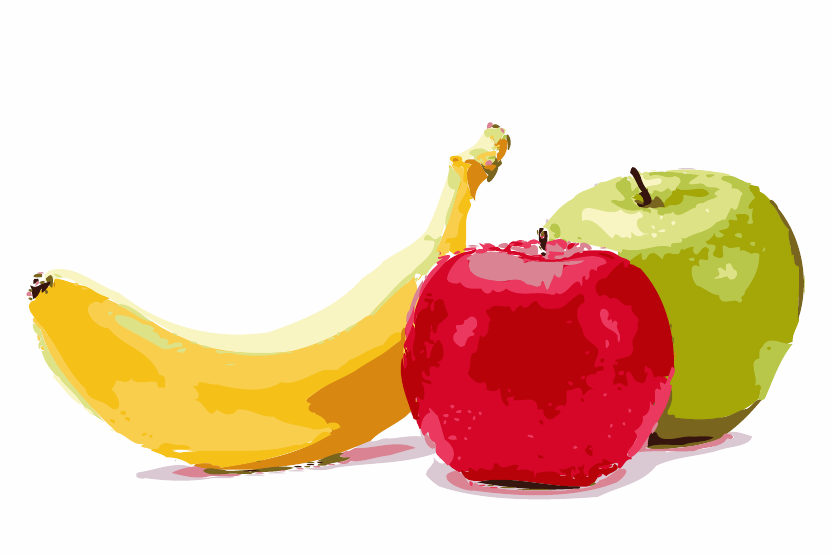}}; 
\spy[red, connect spies] on (0.75,0.15) in node at (3.5,1.1);
\spy[red, connect spies] on (-2.25,-0.1) in node at (3.5,-1.1);
\end{tikzpicture}
\end{tabular}
\caption{[Image vectorization with depth using $\mathcal{S}_{new}$]  (a) The given color quantized image $f$. (b) After multiphase segmentation \eqref{eq: unsupervised seg} applied to $f$ showing different cluster phases. (c) The shape layers in $\mathcal{P}$ with their associated colors. (d) The proposed vectorization with depth using $\mathcal{S}_{new}$.  (e) The proposed vectorization with depth using $\mathcal{S}$.  The boundaries in (d) is less oscillatory compared to (e).  }
\label{fig: fruit phase seg}
\end{figure}
Figure \ref{fig: fruit phase seg} shows the result of grouping quantization for vectorization using $\mathcal{S}_{new}$.  In (a), the given color quantized image $f$ is shown. Unsupervised multiphase segmentation   \eqref{eq: unsupervised seg} is applied to $f$ and gives the segmentation in (b).  Different colors represent different phases, and note that each phase may contain multiple disconnected regions.  (c) shows the shape layers in $\mathcal{P}$ with each associated color $\widehat{c}^j_i$. Notice that in (b), a part of banana and a part of apple are identified to be one phase, yet their associated colors are different and better approximated in (c).  The result using $\mathcal{S}_{new}$ is shown in (d) while, the result using $\mathcal{S}$ is shown in (e). The boundaries in (d) are much better defined and clear compared to the oscillatory boundaries in (e).    During the color quantization, the boundaries are easily affected by contrast and shade, which results in many small regions of different colors in $f$, and  these effects can get emphasized by \bezier curve fitting.  Using $\mathcal{S}_{new}$ helps to add larger semantic shape layers  to the proposed image vectorization with depth. 

\begin{figure}
    \centering
    \begin{tabular}{cc}
       (a)  & (b) \\
     \begin{tikzpicture}[spy using outlines={rectangle,red,magnification=4,size=3cm, connect spies}]
    \node (with MQ) at (0,0)   {\includegraphics[width = 4cm]{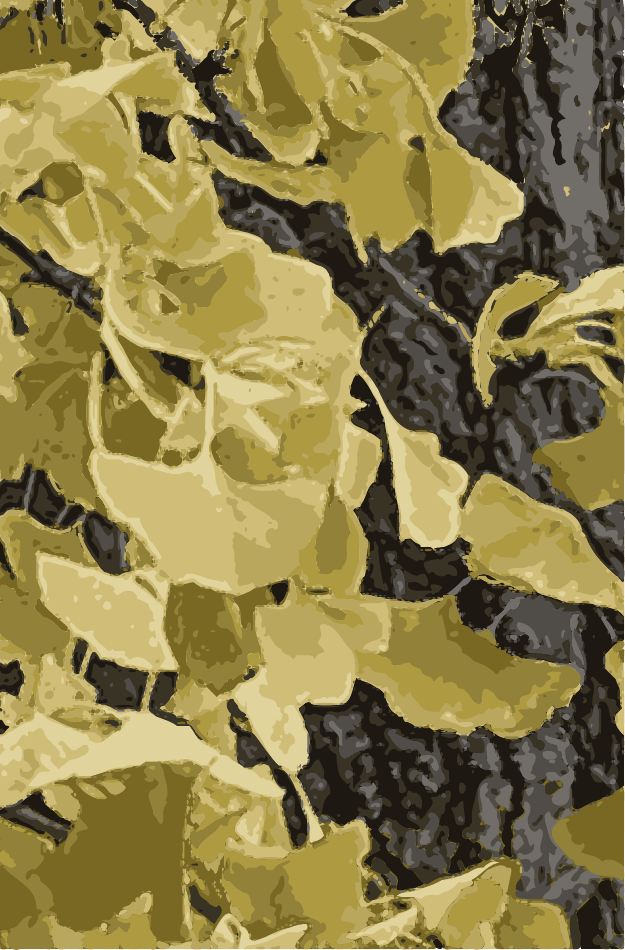} };
    \spy[red, connect spies] on (0.3,0.4) in node at (-3.7,1.5);
    \spy[red, connect spies] on (-1.1,-1.4) in node at (-3.7,-1.5);
    \end{tikzpicture}     & \begin{tikzpicture}[spy using outlines={rectangle,red,magnification=4,size=3cm, connect spies}]
    \node (without MQ) at (0,0){\includegraphics[width = 4cm]{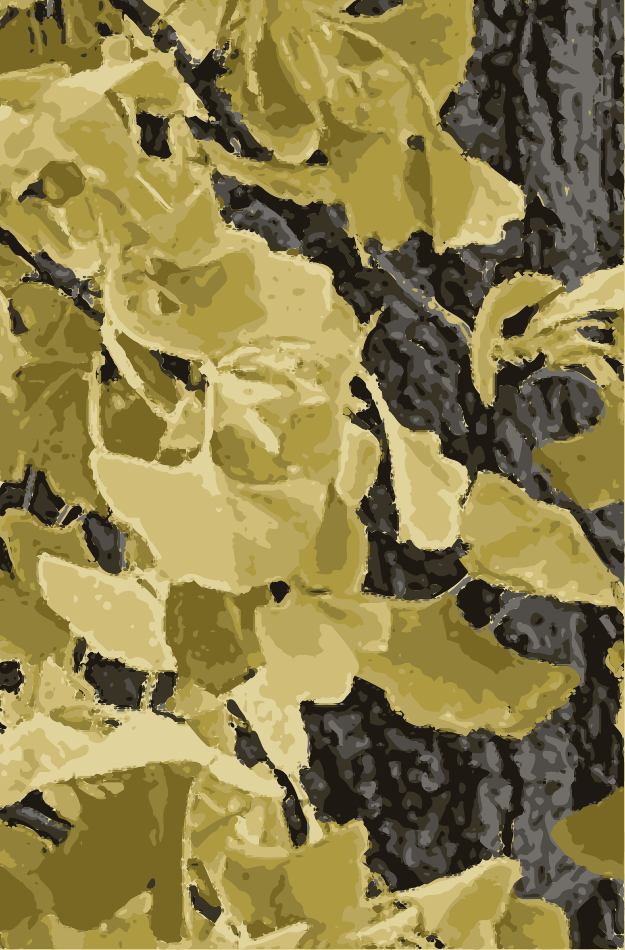} }; 
    \spy[red, connect spies] on (0.3,0.4) in node at (3.7,1.5);
    \spy[red, connect spies] on (-1.1,-1.4) in node at (3.7,-1.5);
    \end{tikzpicture}
    \end{tabular}
   \caption{[Image vectorization with depth using $\mathcal{S}_{new}$] (a) The proposed vectorization with depth using $\mathcal{S}_{new}$.  (b) The proposed vectorization with depth using $\mathcal{S}$.  The boundaries in (a) are less oscillatory compared to (b). }
    \label{fig: tree example}
\end{figure}
Figure \ref{fig: tree example} shows another example of using $\mathcal{S}_{new}$. Figure \ref{fig: tree example} (a) and (b) show image vectorization with depth using $\mathcal{S}_{new}$ and $\mathcal{S}$ respectively.  Using $\mathcal{S}_{new}$, details look sharper and smoother with less oscillation, while using $\mathcal{S}$ has more small noisy regions.
This $\mathcal{S}_{new}$ can reinforce integrity of more semantic shape layers, which is hard to maintain during $K$-mean color quantization step mostly due to brightness and color gradient. This helps to preserve details and minimize gaps between shape layers. It is recommended to carry out this extra step for real images, while it is not as necessary for simpler images or cartoon images.

\subsection{Comparison with layer-based vectorization methods} \label{ss:comparison}

Our approach is unique in a way that we incorporate (computed) depth ordering to vectorization.  To provide comparisons, we pick state-of-the-art methods which considers layer-based vectorization.  We compare our model to LIVE~\cite{ma2022layer}, DiffVG~\cite{li2020diffvg} and LIVSS~\cite{wang2024layered}.
Li et al.~\cite{li2020diffvg} proposed a differentiable rasterizer (DiffVG) which connects the raster image and the vector domain, allowing gradient-based optimization for learning-based approaches toward various vector graphic applications, one of which includes image vectorization. Li et al.~\cite{li2020diffvg} use this differentiable rasterizer to gradually deform randomly initialized shapes until they resemble the input raster image. This can be viewed as a layerwise approach since the shapes overlap each other.  LIVE~\cite{ma2022layer} and LIVSS~\cite{wang2024layered} build on this differentiable rasterizer, that LIVE~\cite{ma2022layer} progressively adds more curves to fit the given image, and LIVSS~\cite{wang2024layered} adds semantic simplification to this process.  

\begin{figure}
    \centering
    \begin{tabular}{ccc}
    (a) & (b) & (c)\\
\includegraphics[width=0.23\textwidth]{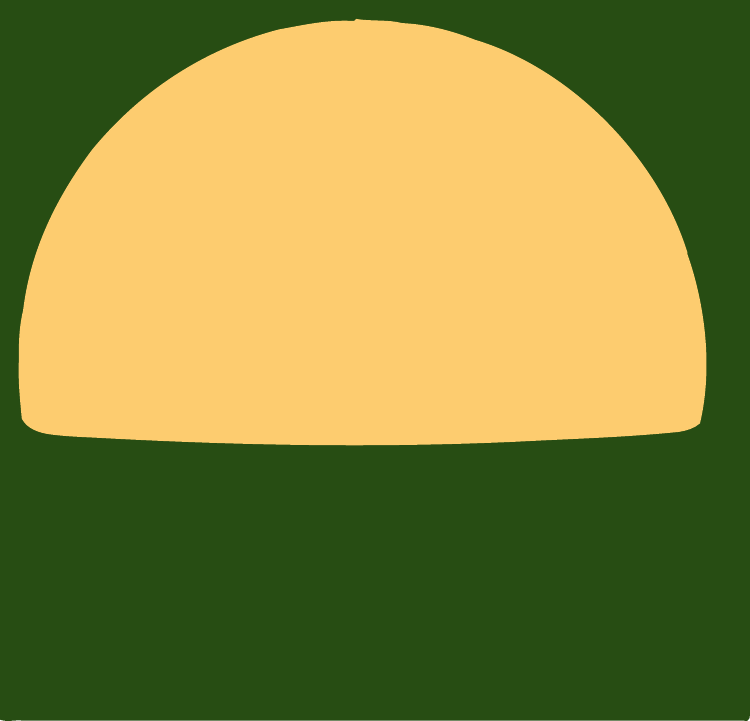} &
\includegraphics[width=0.23\textwidth]{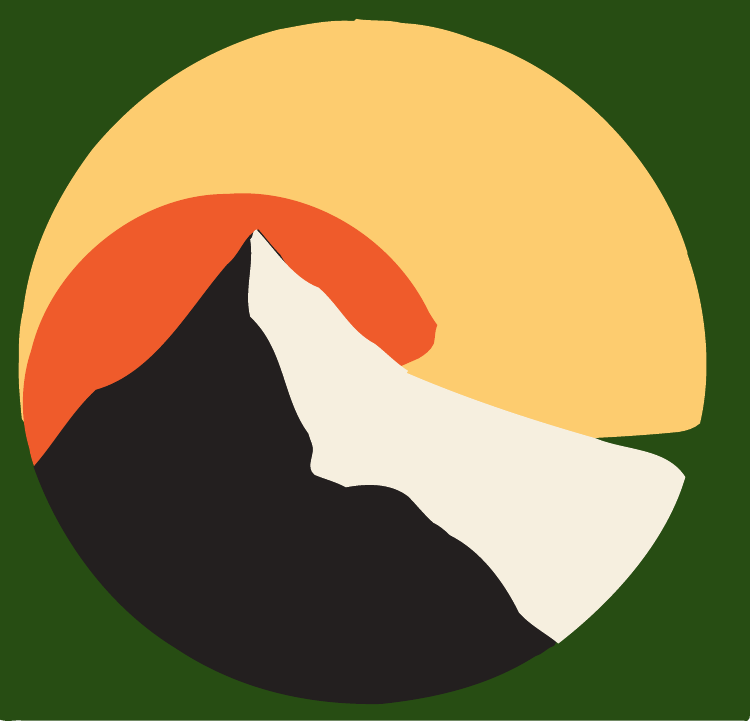} &
\includegraphics[width=0.23\textwidth]{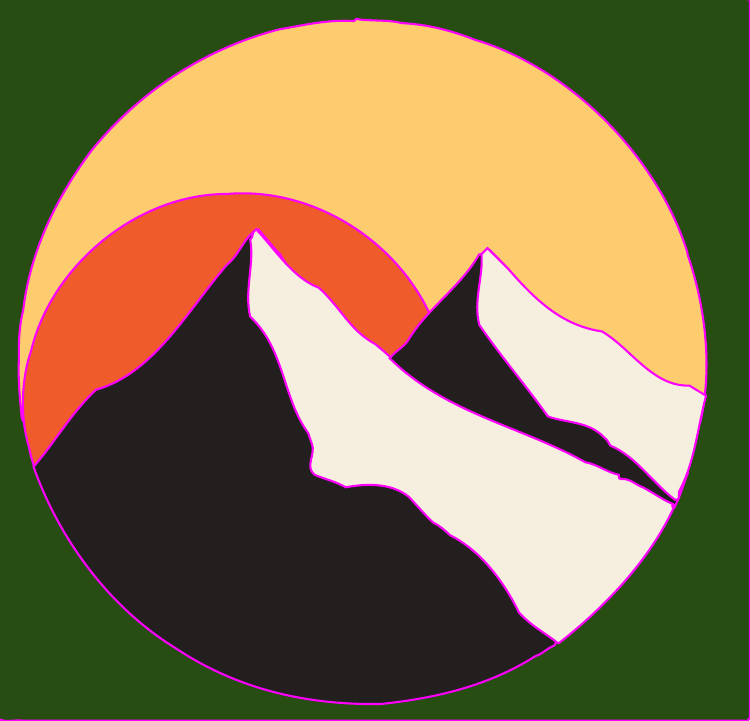} \\
(d) & (e) & (f) \\
\includegraphics[width=0.25\textwidth]{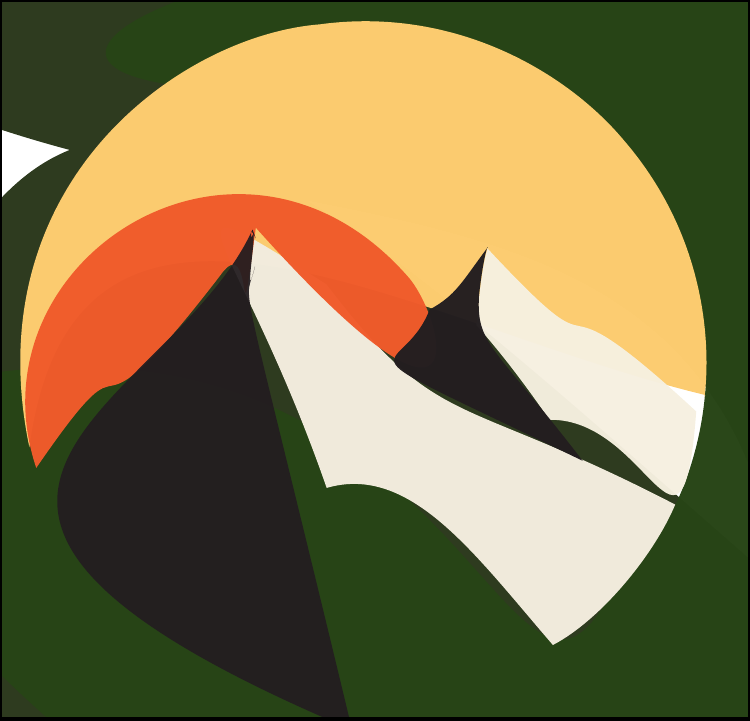}
&    \includegraphics[width=0.25\textwidth]{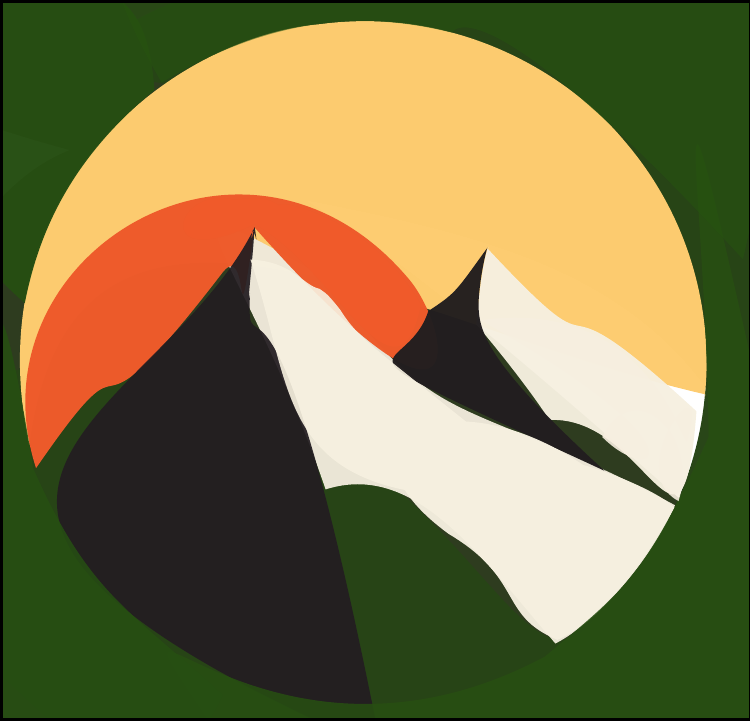}
& \includegraphics[width=.25\textwidth]{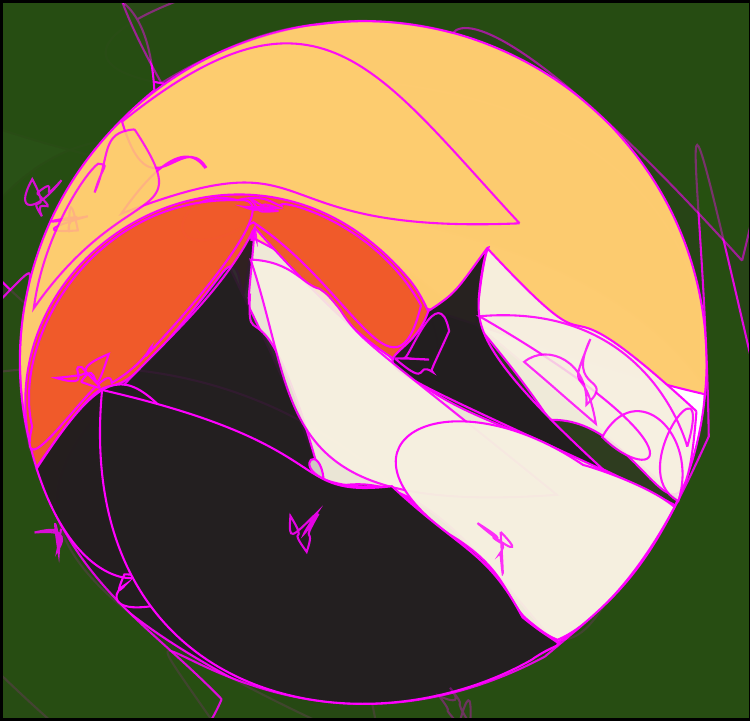}\\
(g) & (h) & (i) \\
\includegraphics[width=0.25\textwidth]{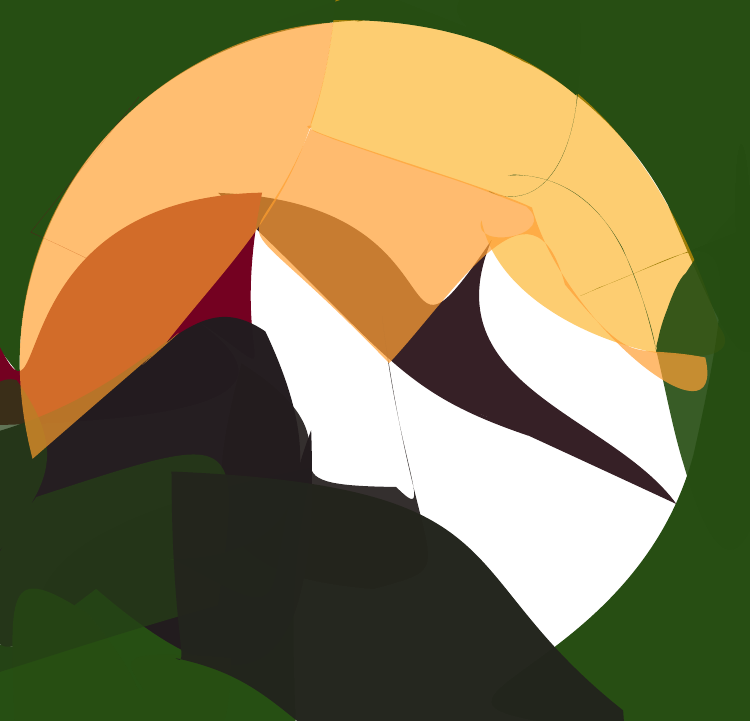} & 
\includegraphics[width=0.25\textwidth]{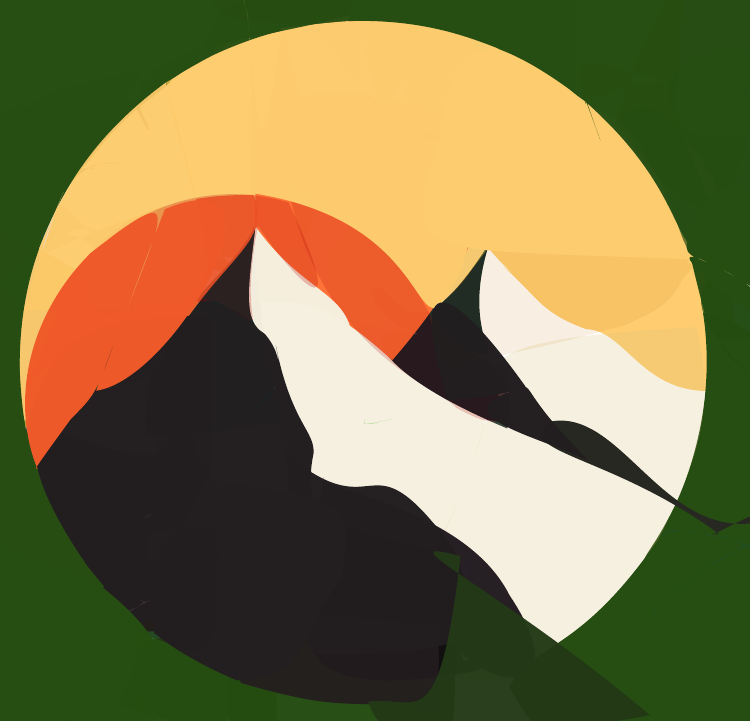} &
\includegraphics[width=0.25\textwidth]{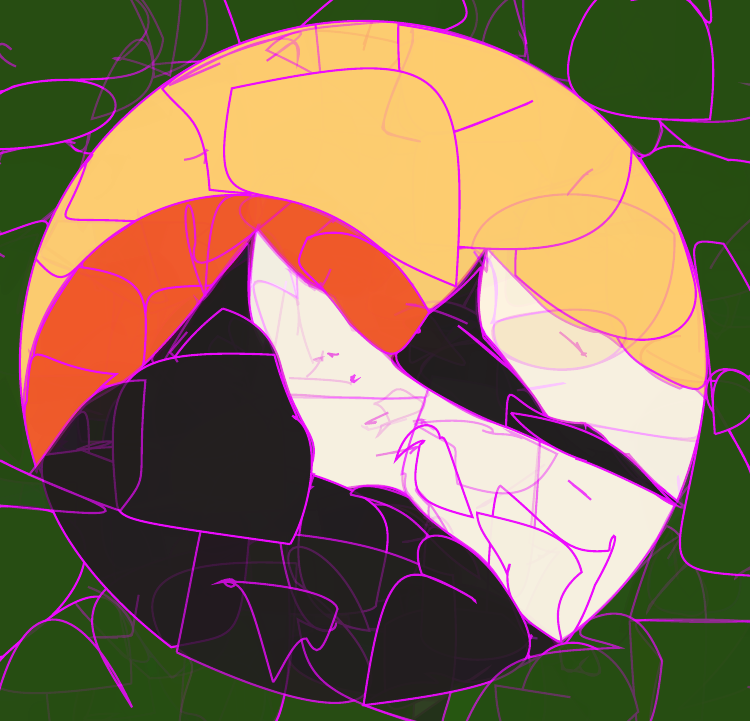} \\
\end{tabular}
\caption{[Comparison with LIVE and DiffVG] The top row, (a)-(c) show the proposed method
$R_{7,4,3}$, $R_{7,4,3,2,6}$ and final vectorization as in Figure \ref{fig: mountain shapes}. 
The second row, (d)-(f) show some layers of LIVE~\cite{ma2022layer}, and the last row (g)-(i) show some layers of  DiffVG~\cite{li2020diffvg}. In the third column, we superpose the boundarires of each shape as pink to emphasize the differences, while other methods use many regions to represent one shape and have excess outside of the given image region, the proposed method (c) gives 7 shape layers depth ordered. }
\label{fig: comparison live mountain long}
\end{figure}
In Figure \ref{fig: comparison live mountain long}, we present the comparisons: the top row shows $R_{7,4,3}$, $R_{7,4,3,2,6}$ and final one as in Figure \ref{fig: mountain shapes} the third row.  In Figure \ref{fig: comparison live mountain long}, the second row shows some layers of LIVE~\cite{ma2022layer}, and the last row that of  DiffVG~\cite{li2020diffvg}.   In the third column, we outline the boundary of each shape as pink to emphasize the differences. 
LIVE~\cite{ma2022layer} and DiffVG~\cite{li2020diffvg} are both initialized by the number of output shapes, or paths as called in \cite{ma2022layer} and \cite{li2020diffvg}, and we use the default parameters provided by the authors for both methods.  Paths are highlighted in pink strokes, and we add black square frames in (d)-(f) to indicate the original size of the input quantized image.  LIVE nor DiffVG does not limit the vectorized paths to be contained in the given image size. 
It is reasonable to perceive 7 shapes for this given image, but LIVE or DiffVG use many paths, i.e., use many regions to represent one shape. 
LIVE~\cite{ma2022layer} considers only minimizing color difference and suppressing self-intersecting \bezier curves, other shape regularity such as arc-length and curvature is not taken into account. Shapes have less regularity at places that they are covered.

\begin{figure}
\centering
\begin{tikzpicture}[node distance=8.45cm]
  \node (input) [label=above: (a)] {\begin{tikzpicture}[spy using outlines={rectangle,red,magnification=3.5,size=2cm, connect spies}]
  \node (given) at (0,0) {\includegraphics[width = 7cm]{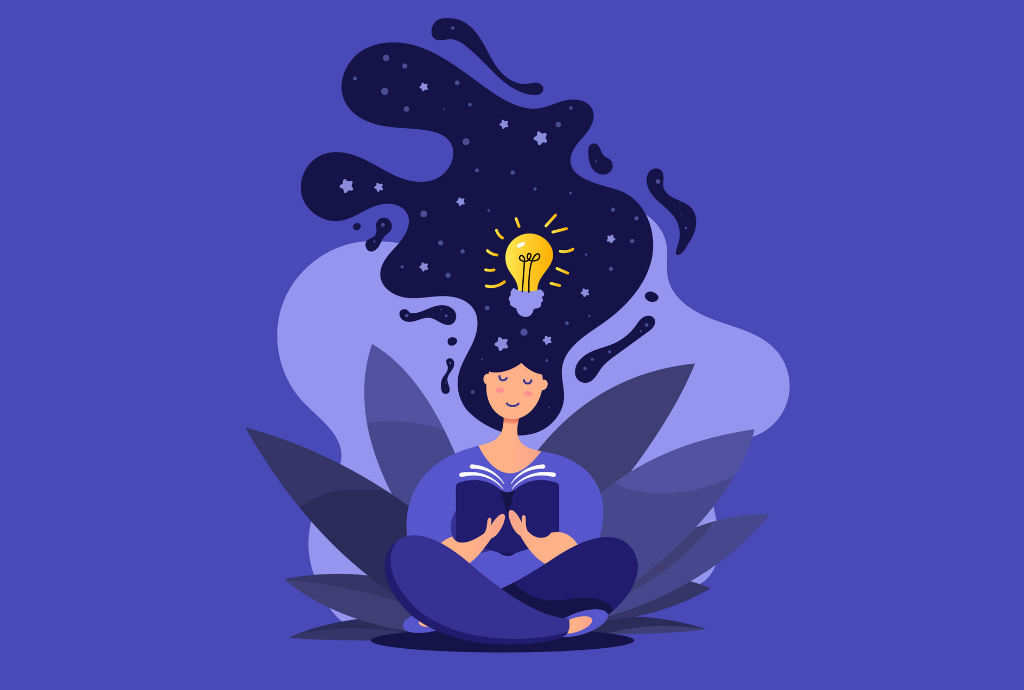}};
  \spy[red, connect spies] on (-0.8,1.8) in node at (-3,3.5);
  \spy[red, connect spies] on (0.1,0.55) in node at (-1,3.5);
  \spy[red, connect spies] on (0.05,-0.3) in node at (1,3.5);
  \spy[red, connect spies] on (-0.05,-1.05) in node at (3,3.5);
  \end{tikzpicture}};
  \node (ours) [label=above: (b), right of = input]{\begin{tikzpicture}[spy using outlines={rectangle,red,magnification=3.5,size=2cm, connect spies}]
  \node (ours)  at (0,0) {\includegraphics[width = 7cm]{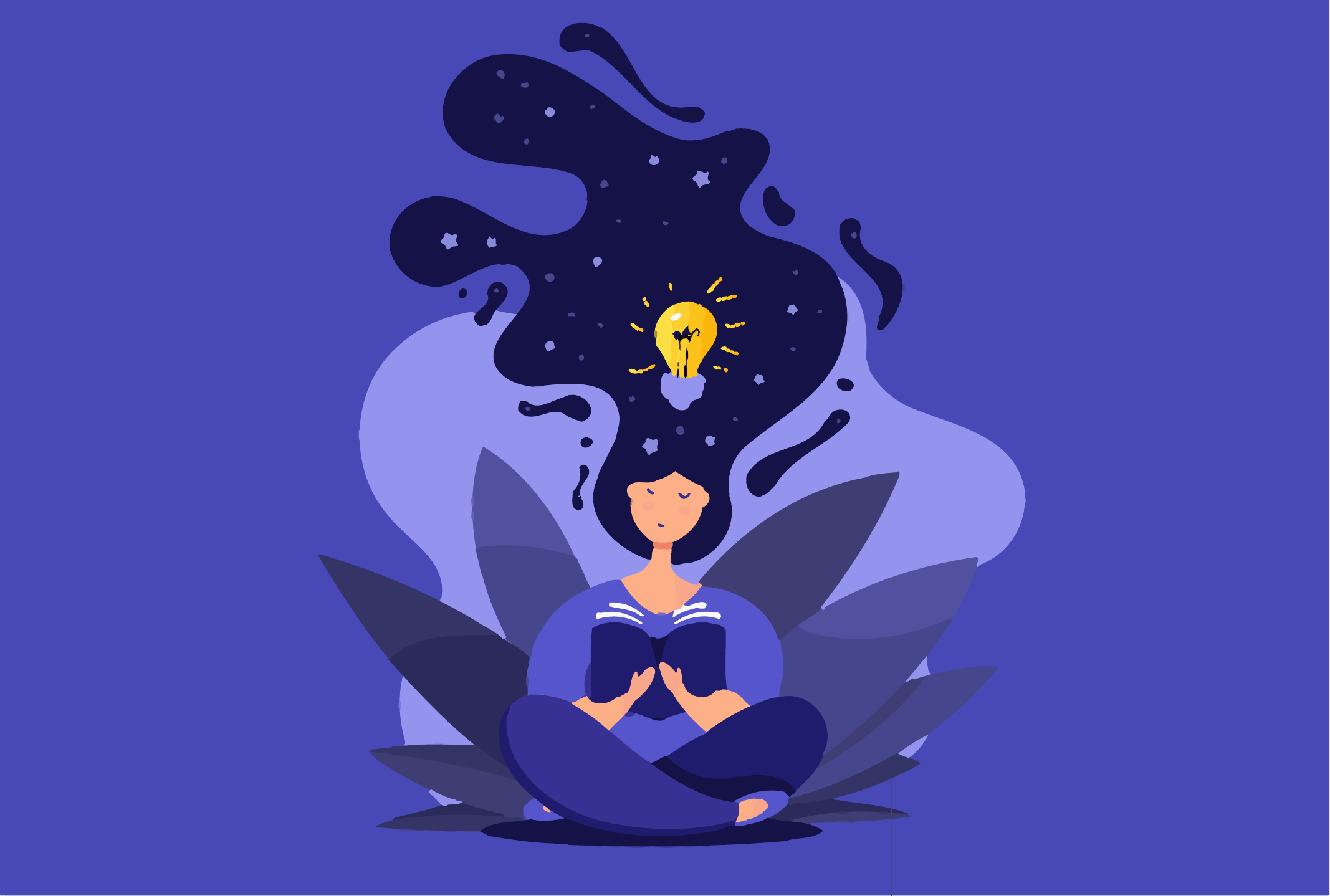}};
  \spy[red, connect spies] on (-0.8,1.8) in node at (-3,3.5);
  \spy[red, connect spies] on (0.1,0.55) in node at (-1,3.5);
  \spy[red, connect spies] on (0.05,-0.3) in node at (1,3.5);
  \spy[red, connect spies] on (-0.05,-1.05) in node at (3,3.5);
  \end{tikzpicture} };
  \node (live) [label=above: (c), below of = input]{\begin{tikzpicture}[spy using outlines={rectangle,red,magnification=4,size=2cm, connect spies}]
  \node (live) at (0,0) {\includegraphics[width = 7cm]{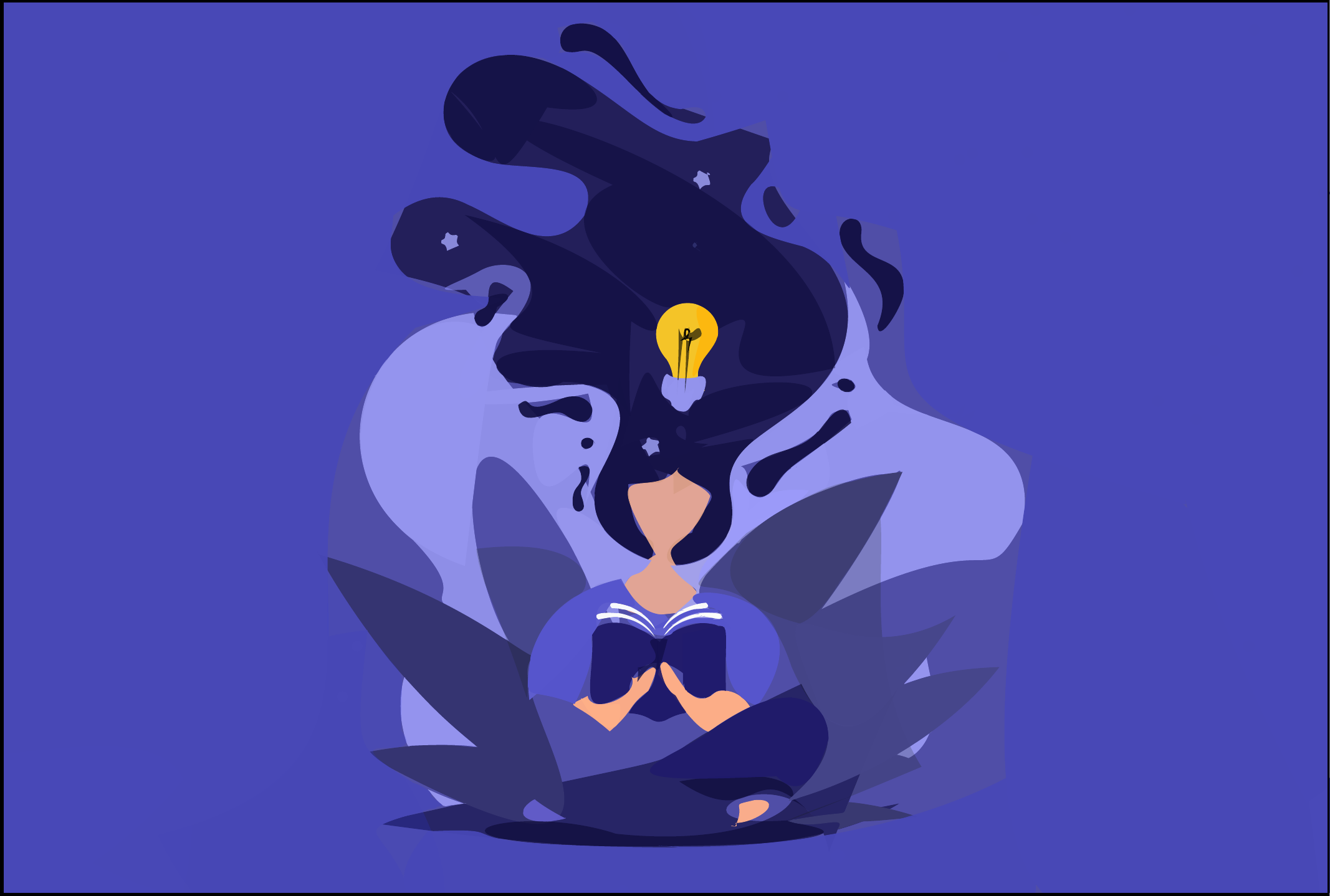}};
  \spy[red, connect spies] on (-0.8,1.8) in node at (-3,3.5);
  \spy[red, connect spies] on (0.1,0.55) in node at (-1,3.5);
  \spy[red, connect spies] on (0.05,-0.3) in node at (1,3.5);
  \spy[red, connect spies] on (-0.05,-1.05) in node at (3,3.5);
  \end{tikzpicture} };
  \node (diffvg) [label=above: (d), right of = live]{\begin{tikzpicture}[spy using outlines={rectangle,red,magnification=4,size=2cm, connect spies}]
  \node (diffvg) at (0,0) {\includegraphics[width = 7cm]{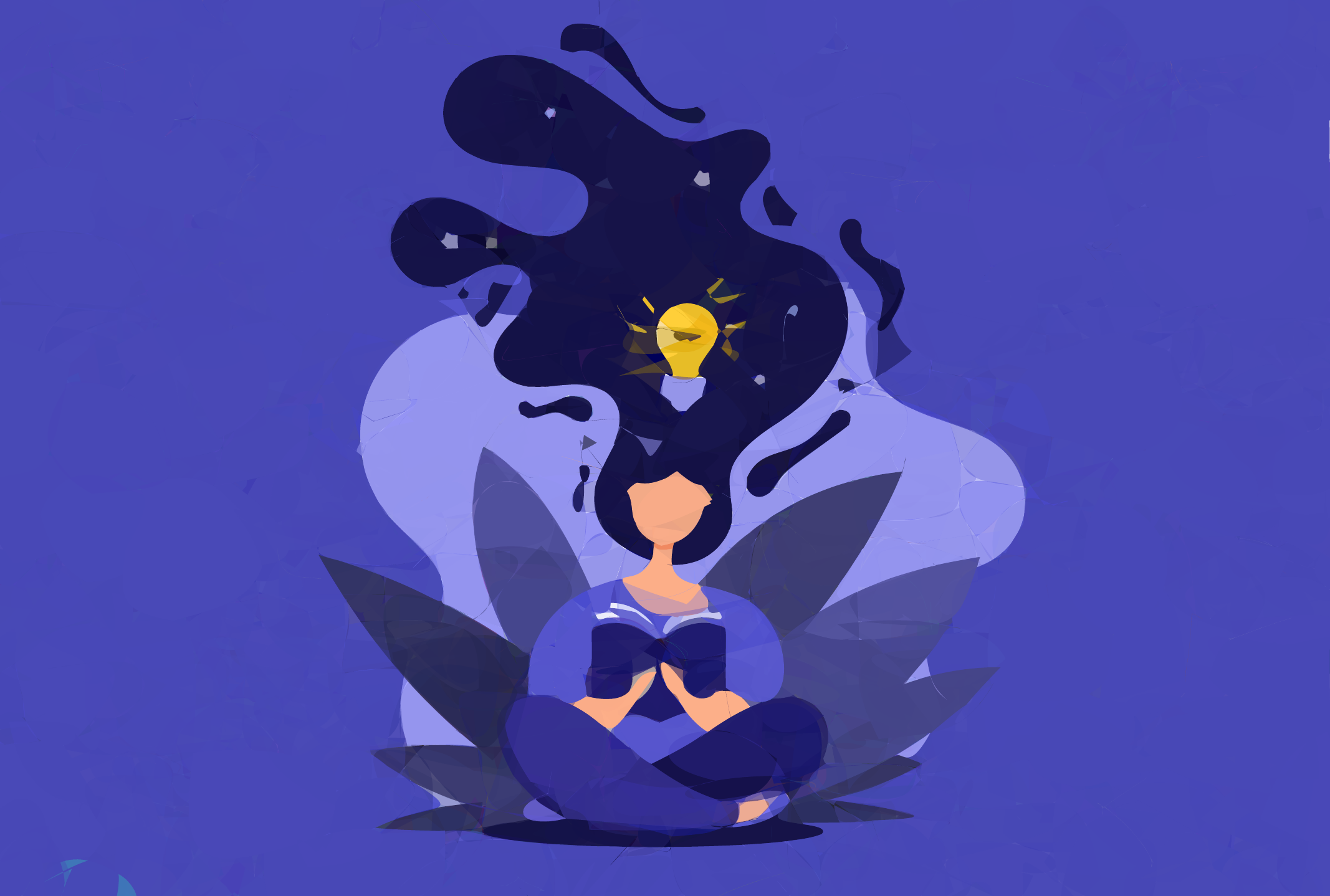}};
  \spy[red, connect spies] on (-0.8,1.8) in node at (-3,3.5);
  \spy[red, connect spies] on (0.1,0.55) in node at (-1,3.5);
  \spy[red, connect spies] on (0.05,-0.3) in node at (1,3.5);
  \spy[red, connect spies] on (-0.05,-1.05) in node at (3,3.5);
  \end{tikzpicture} };
\end{tikzpicture}
\caption{[Comparison with LIVE and DiffVG] (a) The given color quantized image $f$.  (b) the proposed method. (c) and (d) are results from LIVE and DiffVG respectively. Not only vectorization with depth give the correct depth information, it keeps more details with less number of \bezier curves. } 
\label{fig: detail comparison live woman}
\end{figure}
Figure \ref{fig: detail comparison live woman} shows another example, considered in \cite{ma2022layer}.  From the given color quantized image $f$ in (a), we present our result in (b), LIVE~\cite{ma2022layer} result in (c) and DiffVG~\cite{li2020diffvg} result in (d). 
The proposed method and LIVE yield meaningful depth arrangements, while DiffVG, with its random shape initialization approach, may not provide layering information of comparable significance. The proposed  method preserves more details: image vectoriation with depth retains facial expressions and finer details on the book and the background.

In Table \ref{table: comparison}, we present computational complexity and some quantitative measures for comparison.  We run experiments on the same MacBook Pro with M1 pro consisting of 10-core CPU and 32GB memory without using GPU. We note that there is difference in hardware: our method is run on an Apple M1 Pro with 10-core CPU, while DiffVG~\cite{li2020diffvg} runs on 13th Generation Intel\textregistered Core\texttrademark i9-13900K with 24 cores and NVIDIA GeForce RTX 4090. It took 5 hours 44 minutes for LIVE~\cite{ma2022layer} to complete the process while the proposed method took only 37 seconds.  Due to the demanding computing resource needed for LIVE~\cite{ma2022layer}, we did not try higher number of paths such as 128 or 512 as used in \cite{ma2022layer}. 
For quantitative comparison, we present Mean Square Error(MSE) and Peak Signal-to-Noise Ratio (PSNR). We first convert each vectorized output to PNG format using Adobe Illustrator~\cite{adobeillustrator}, and compute the error against the raw input raster image. 
Some cases despite its long computation time and better computation resource for LIVE, it still appears to be far from convergence.
In general, the proposed method gives a lower MSE loss and higher PSNR.

\begin{table}
        \centering
        \begin{tabular}{|c|c|c|c|c|c|c|}
        \hline
        & Initialization & \begin{tabular}[c]{@{}c@{}}CPU/\\ GPU\end{tabular} & \begin{tabular}[c]{@{}c@{}}MSE\\ Loss\end{tabular} & PSNR & \begin{tabular}[c]{@{}c@{}}Number of \\ \bezier Curves\end{tabular} &  Time(s) \\ \hline
        \begin{tabular}[c]{@{}c@{}}The proposed method\\ Figure \ref{fig: flow chart}\\Figure \ref{fig: fruit phase seg} \\ Figure \ref{fig: tree example} \\ Figure \ref{fig: detail comparison live woman} \end{tabular} & \begin{tabular}[c]{@{}c@{}}5 colors\\ 15 colors \\ 10 colors \\ 40 colors \end{tabular} & \begin{tabular}[c]{@{}c@{}}CPU\\ CPU\\ CPU \\ CPU\end{tabular} & \begin{tabular}[c]{@{}c@{}}13.44\\ 84.12 \\ 99.86 \\  9.57\end{tabular} & \begin{tabular}[c]{@{}c@{}}41.62\\ 33.65\\ 32.91  \\ 43.09\end{tabular} & \begin{tabular}[c]{@{}c@{}}93\\ 786 \\ 13257 \\ 844\end{tabular} &  \begin{tabular}[c]{@{}c@{}}37\\ 75\\ 657 \\  288\end{tabular} \\ 
        \hline
        \begin{tabular}[c]{@{}c@{}}LIVE\\ Figure \ref{fig: flow chart}\\ Figure \ref{fig: fruit phase seg} \\ Figure \ref{fig: tree example} \\  Figure \ref{fig: detail comparison live woman}\end{tabular} & \begin{tabular}[c]{@{}c@{}}32 paths\\ 128 paths \\ 512 paths \\  128 paths\end{tabular} & \begin{tabular}[c]{@{}c@{}}CPU\\ GPU\\ GPU \\ GPU\end{tabular} & \begin{tabular}[c]{@{}c@{}}28.55\\ 68.71 \\ 163.95 \\  36.86 \end{tabular} & \begin{tabular}[c]{@{}c@{}}38.34\\ 34.53 \\ 30.75\\  37.24\end{tabular} & \begin{tabular}[c]{@{}c@{}}128\\ 512 \\ 2064 \\  512\end{tabular} & \begin{tabular}[c]{@{}c@{}}20640\\ 5072 \\ 36472 \\  18393\end{tabular} \\ \hline
        \begin{tabular}[c]{@{}c@{}}DiffVG\\ (Figure \ref{fig: flow chart}\\ Figure \ref{fig: fruit phase seg} \\ Figure \ref{fig: tree example} \\  Figure \ref{fig: detail comparison live woman}\end{tabular} & \begin{tabular}[c]{@{}c@{}}128 paths\\ 512 paths \\ 1024 paths\\ 1024 paths\end{tabular} & \begin{tabular}[c]{@{}c@{}}GPU\\ GPU\\ GPU \\  GPU\end{tabular} & \begin{tabular}[c]{@{}c@{}}71.35\\ 83.71 \\ 209.63 \\  40.16\end{tabular} & \begin{tabular}[c]{@{}c@{}}34.37\\ 33.67 \\ 29.69 \\  36.86\end{tabular} & \begin{tabular}[c]{@{}c@{}}517\\ 2059 \\ 4126 \\  4126 \end{tabular} & \begin{tabular}[c]{@{}c@{}}194\\ 179 \\ 310 \\  1188\end{tabular} \\ \hline
        \end{tabular}
        \caption{[Quantitative comparison]  The second column shows how we initialized the methods. The proposed method uses least number of colors.  MSE and PSNR are comparable with clear benefit in the speed. }
        \label{table: comparison}
\end{table}

\begin{figure}
      \centering
  \begin{subfigure}[b]{.3\textwidth}
    \centering
    \subcaption{}
    \includegraphics[width=\textwidth]{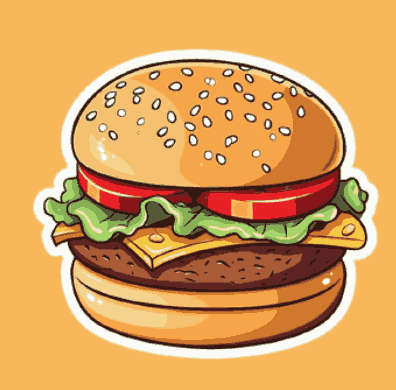} 
  \end{subfigure}
\begin{subfigure}[b]{.3\textwidth}
    \centering
    \subcaption{}
    \includegraphics[width=\textwidth]
    {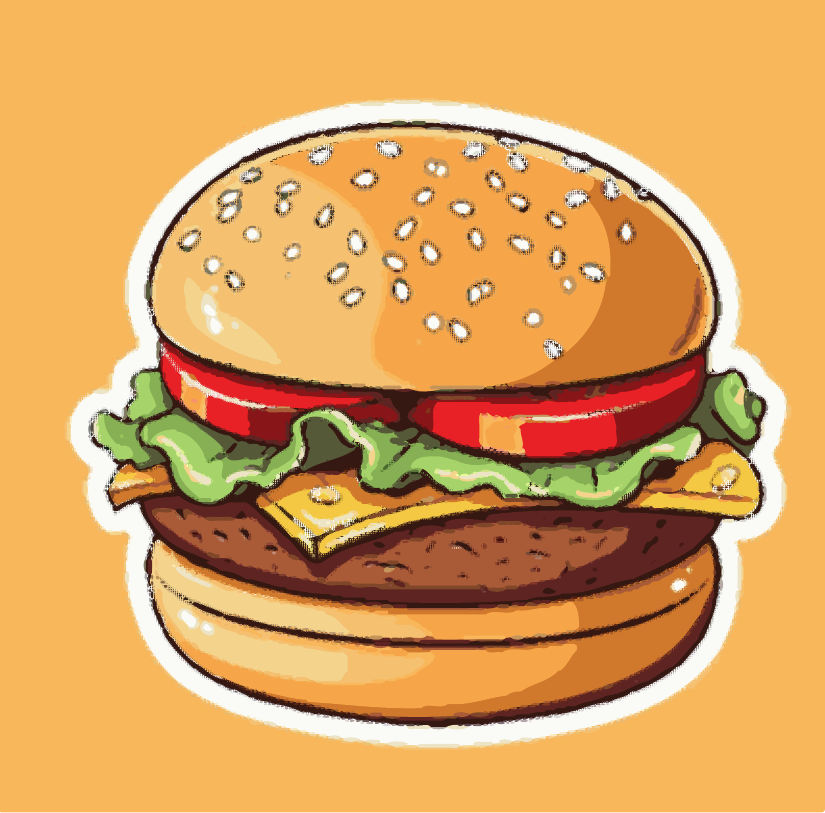} 
  \end{subfigure}
\begin{subfigure}[b]{.3\textwidth}
    \centering
    \subcaption{}
    \includegraphics[width=\textwidth]{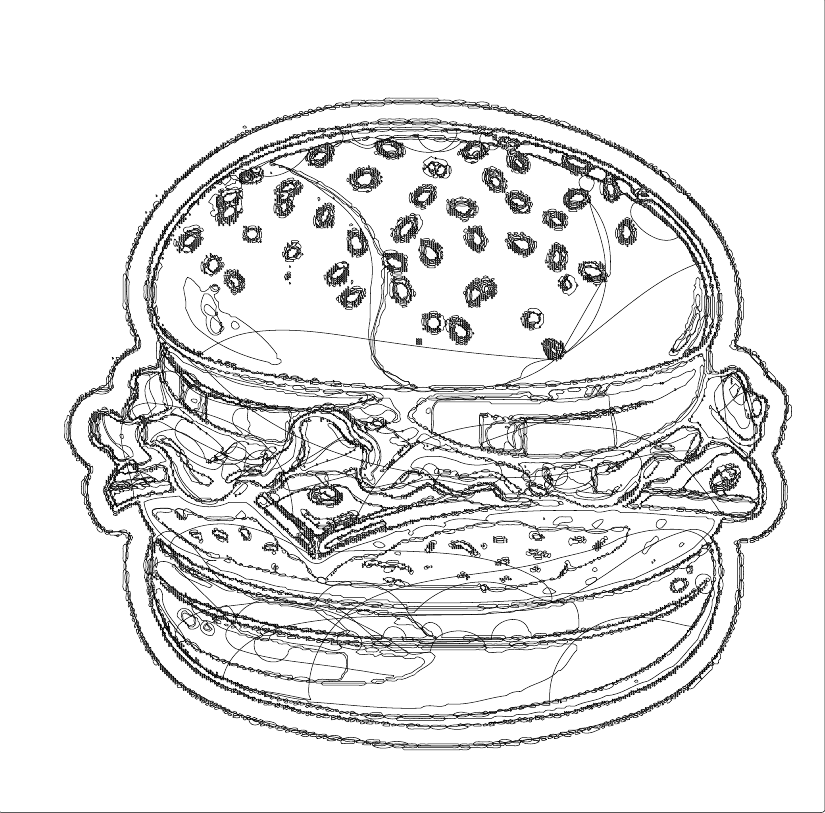}
  \end{subfigure}
\begin{subfigure}[b]{.18\textwidth}
    \centering
    \subcaption{}
    \includegraphics[width=\textwidth]{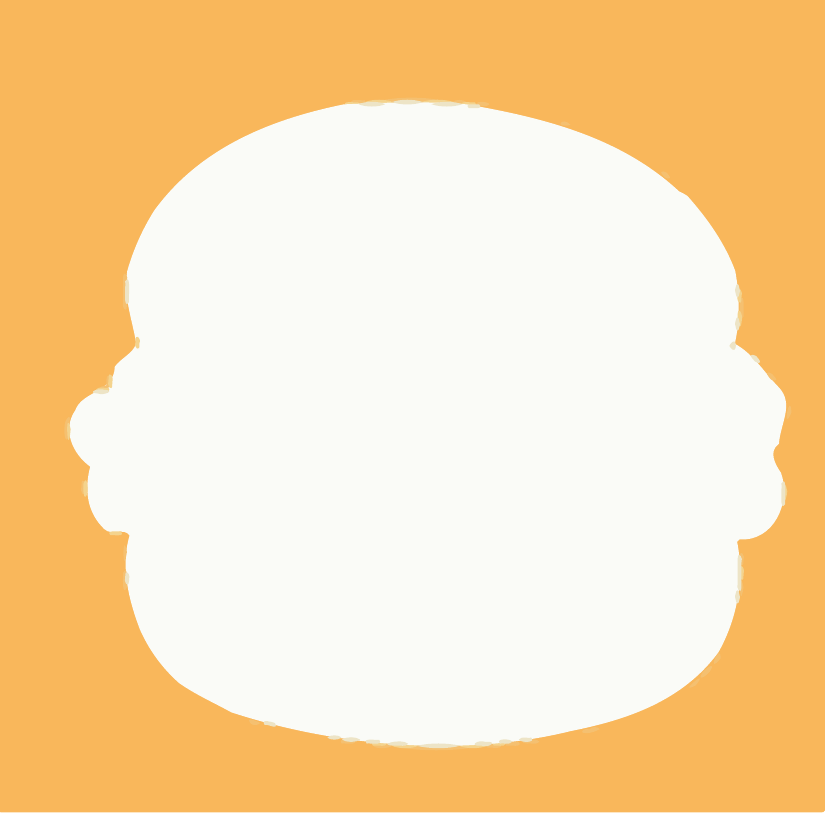}
  \end{subfigure}
\begin{subfigure}[b]{.18\textwidth}
    \centering
    \subcaption{}
    \includegraphics[width=\textwidth]{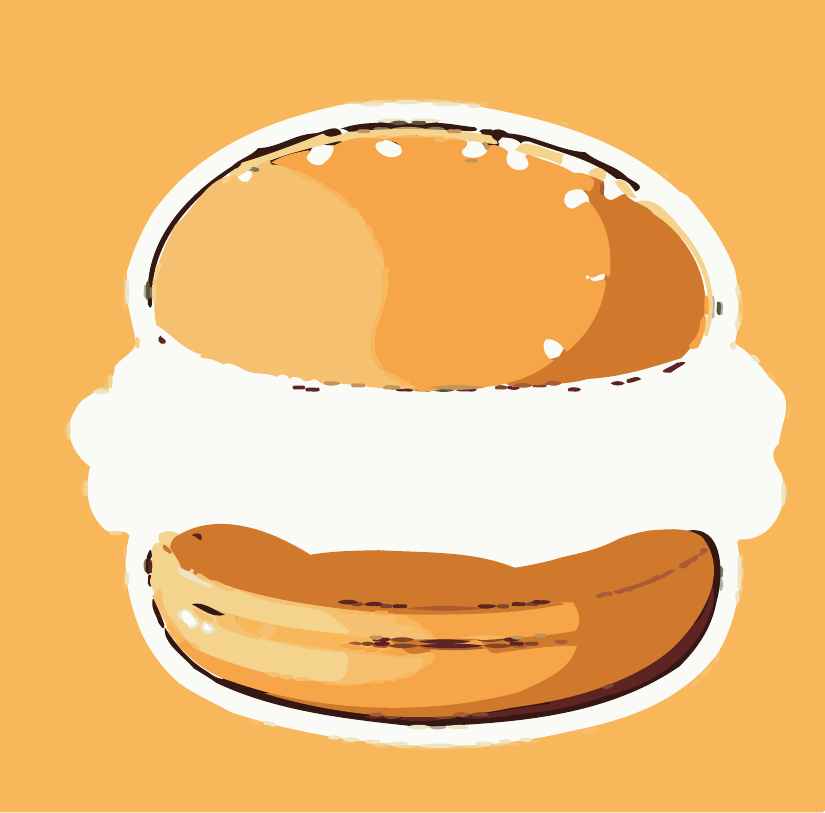}
  \end{subfigure}
\begin{subfigure}[b]{.18\textwidth}
    \centering
    \subcaption{}
    \includegraphics[width=\textwidth]{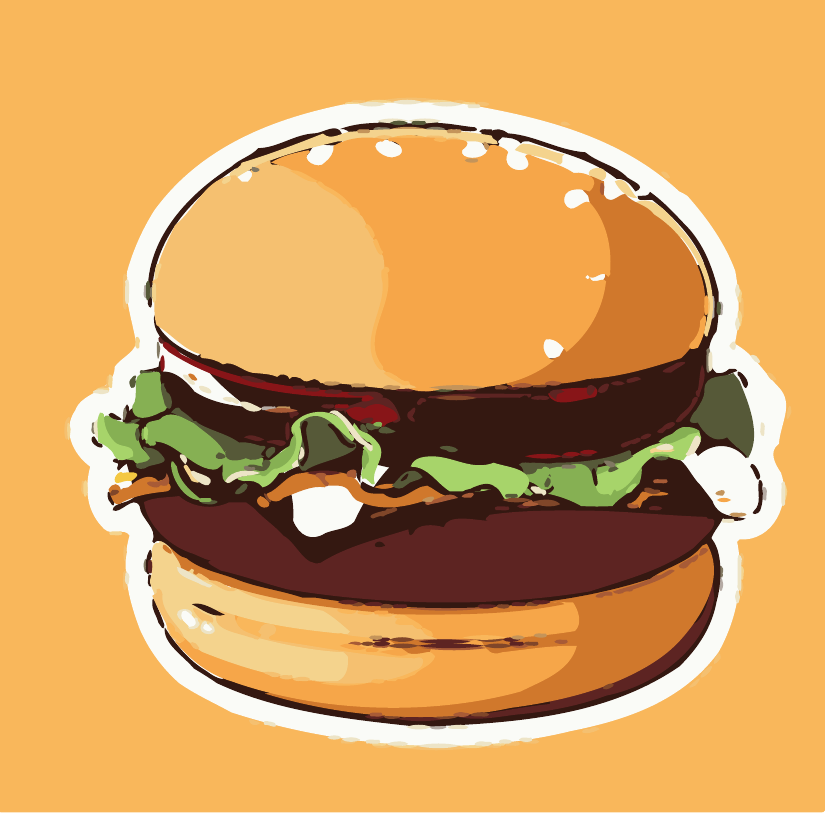}
  \end{subfigure}
\begin{subfigure}[b]{.18\textwidth}
    \centering
    \subcaption{}
    \includegraphics[width=\textwidth]{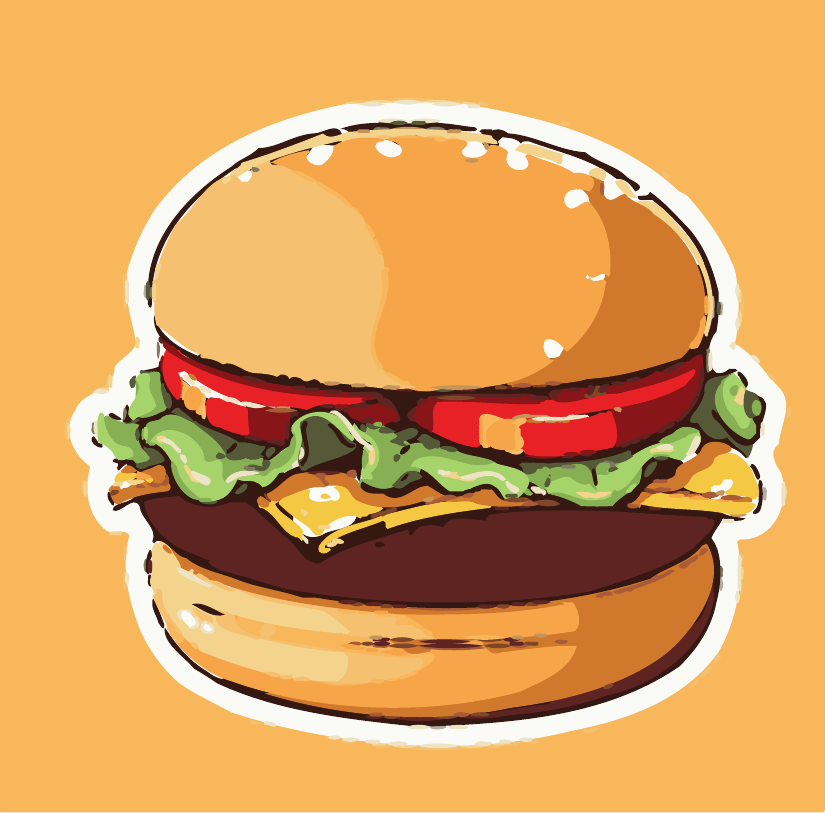}
  \end{subfigure}
  \begin{subfigure}[b]{.18\textwidth}
    \centering
    \subcaption{}
    \includegraphics[width=\textwidth]{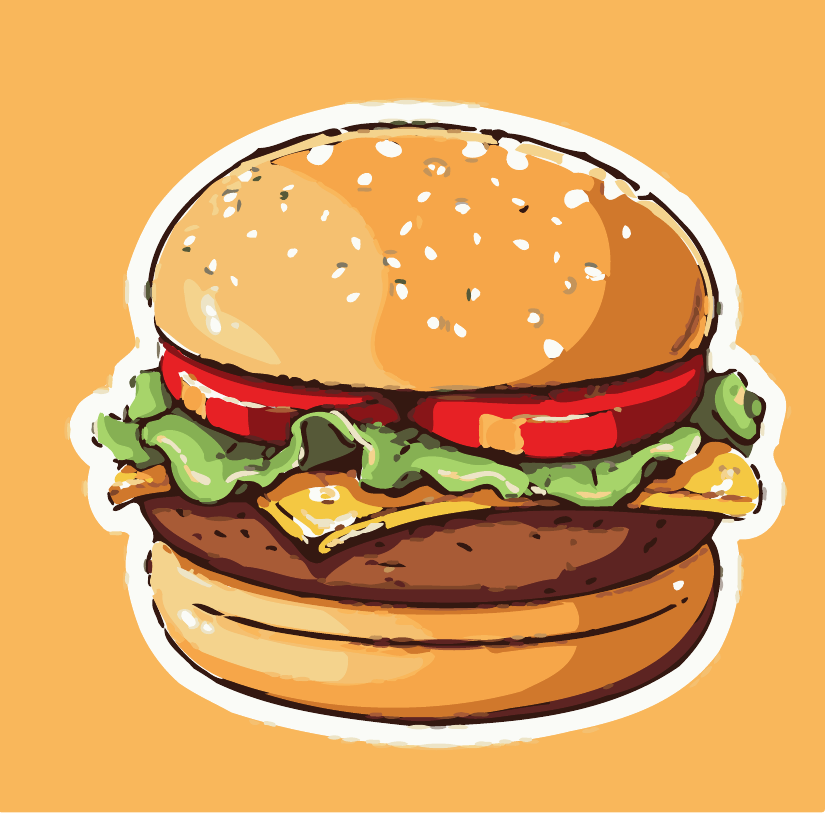} 
  \end{subfigure}
     \caption{[Comparison with LIVSS~\cite{wang2024layered}] (a) The given color-quantized image $f$. (b) The proposed method with $S_{noise}$.  (c) The stroke of result in (b).  (d) to (h) show shape layers in various depth level from the bottom to the top.  }
    \label{fig: burger}
\end{figure}
We experiment with images considered in LIVSS~\cite{wang2024layered} in Figure \ref{fig: burger}.  We direct readers to~\cite{wang2024layered} for comparison.  Figure \ref{fig: burger} (a) displays the color quantized image (396 × 390 pixels) of the image obtained via screen capture from the website of~\cite{wang2024layered}.  Our method takes 173 seconds, while~\cite{wang2024layered} reports 888 seconds; we note that this may be due to the image size being different. 
Figure \ref{fig: burger} (b) presents the proposed method and (c) shows the stroke of our vectorization.  Images (d) through (h) shows various depth levels from the bottom to the top by the proposed method.  (d) is showing two layers, the bottom shape layer is yellow background which is fully yellow, and the second shape layer is the white outline of the burger. 
One significant difference between the proposed method and \cite{wang2024layered} is that our method represents the white stroke around the burger as a single piece, avoiding the use of excess Bézier curves to depict this white shape.  The proposed method identifies the bun as a few large shape layers, showing the gradient changes of bun in (a).  From (d) to (h), it shows the ordering of textures: background, white outline of the burger, two buns, burger meat and lettuces, tomatos, then sesame seeds on the bun, each convexified by curvature based inpainting for vectorization.  To keep more details, after collecting all layers from the bottom to the top in (h), we added the vectorized noise layer $S_{noise}$ in \eqref{eq:noiselayer} to (h) and get (b) just for this experiment.  This is due to the color quantization that very thin shapes typically gets separated into many small regions with slightly different colors near the boundary, e.g., the boundaries of sesame seeds on the top bun. In such case adding vectorized noise layer can help to keep more details.   
Figure \ref{fig: burger}(c) shows the stroke of our vectorized output, which has a complexity visually comparable to that of~\cite{wang2024layered} and is less complex than LIVE~\cite{ma2022layer} and DiffVG~\cite{li2020diffvg} as shown in~\cite{wang2024layered}. 

\begin{figure}
    \centering
    \begin{subfigure}[b]{.3\textwidth}
      \subcaption{}
      \centering
      \includegraphics[width=\textwidth]{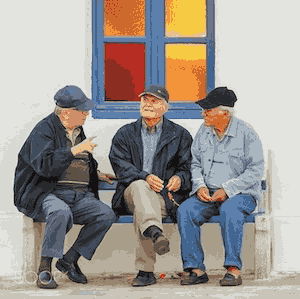}
  \end{subfigure}
  \begin{subfigure}[b]{.3\textwidth}
      \subcaption{}
      \centering
      \includegraphics[width=\textwidth]{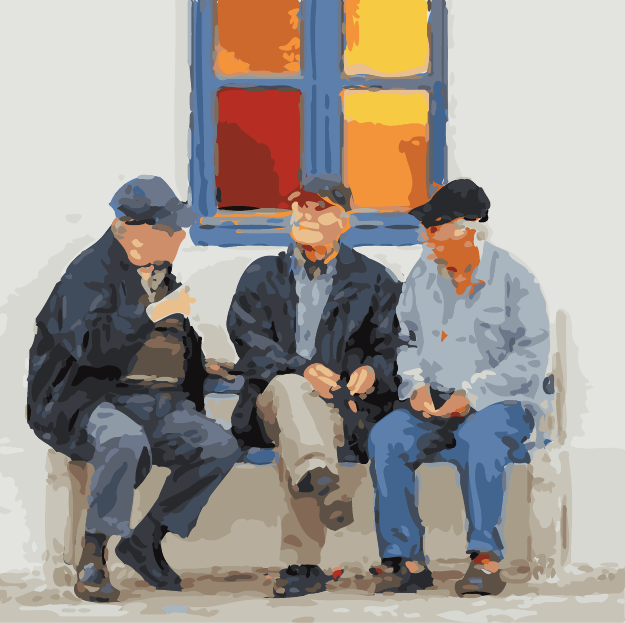}
  \end{subfigure}
  \begin{subfigure}[b]{.3\textwidth}
      \subcaption{}
      \centering
      \includegraphics[width=\textwidth]{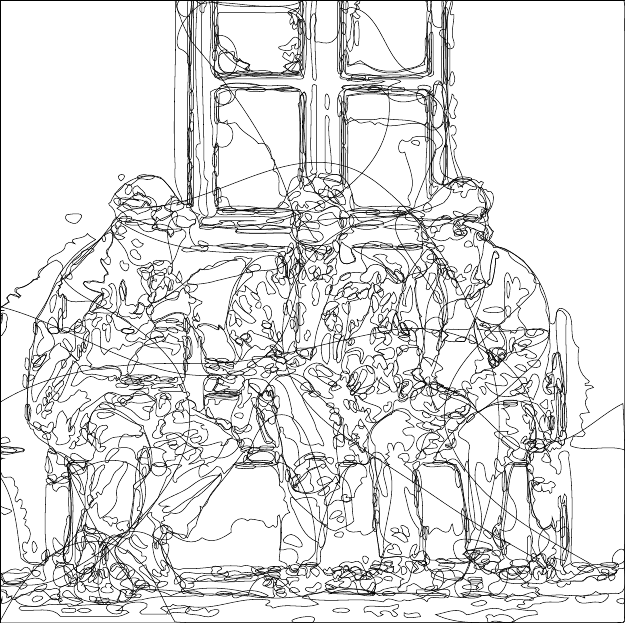}
  \end{subfigure}

  \begin{subfigure}[b]{.2\textwidth}
      \subcaption{}
      \centering
      \includegraphics[width=\textwidth]{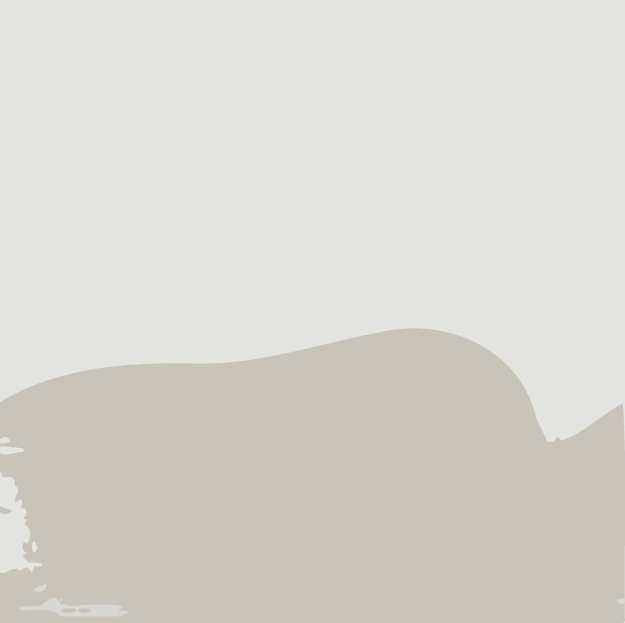}
  \end{subfigure}
   \begin{subfigure}[b]{.2\textwidth}
      \subcaption{}
      \centering
      \includegraphics[width=\textwidth]{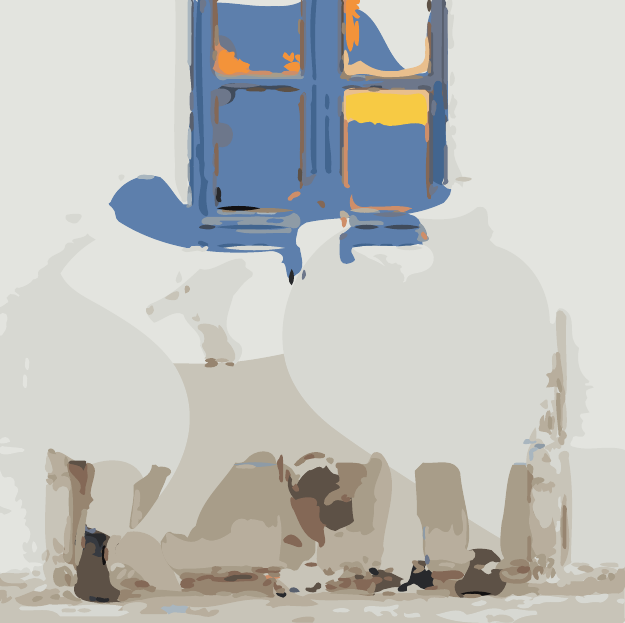}
  \end{subfigure}
  \begin{subfigure}[b]{.2\textwidth}
      \subcaption{}
      \centering
      \includegraphics[width=\textwidth]{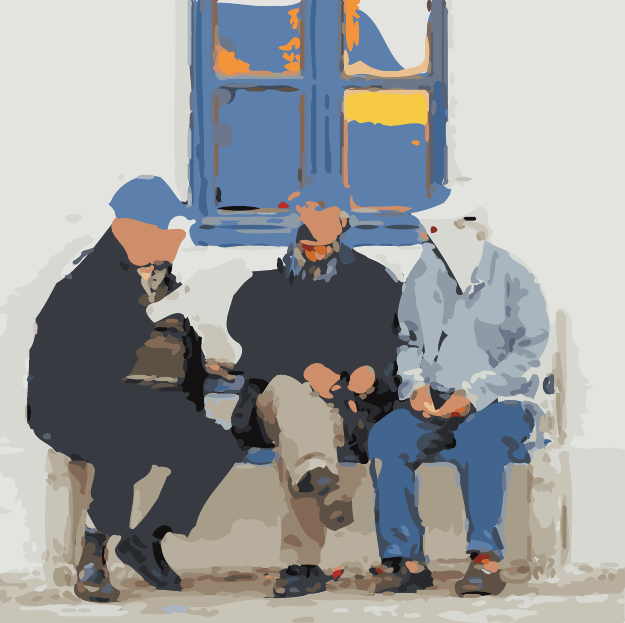}
  \end{subfigure}
  \begin{subfigure}[b]{.2\textwidth}
      \subcaption{}
      \centering
      \includegraphics[width=\textwidth]{pdf/grandpa/grandpa_5.pdf}
  \end{subfigure}
    \caption{[Comparison with LIVSS~\cite{wang2024layered}] (a) The given color quantized image $f$. (b) The proposed method using $\mathcal{S}_{new}$.  (c) The stroke of result in (b).  (d) to (g)  showing shape layers in various depth level from the bottom to the top.  }
    \label{fig: grandpa}
\end{figure}
Figure \ref{fig: grandpa} shows another example that is also in~\cite{wang2024layered}. We use the image size of $300 \times 299$ pixels, and utilize grouping of quantization and used $\mathcal{S}_{new}$. Figure \ref{fig: grandpa}(a) shows the color quantized raster image $f$, (b) shows the proposed vectorized result, and (c) shows the stroke. (d) to (g) are  the layered vectorized result from bottom to some depth ordering.

\subsection{Grouping disconnected regions}\label{ss:grouping}
\begin{figure}
    \centering
    \begin{subfigure}[b]{.3\textwidth}
      \centering
      \subcaption{}
    \includegraphics[width = \textwidth]{pic/illusory/generated.png}
    \end{subfigure}
    \hspace{-.7cm}
    \begin{subfigure}[b]{.3\textwidth}
      \centering
      \subcaption{}
    \includegraphics[width = \textwidth]{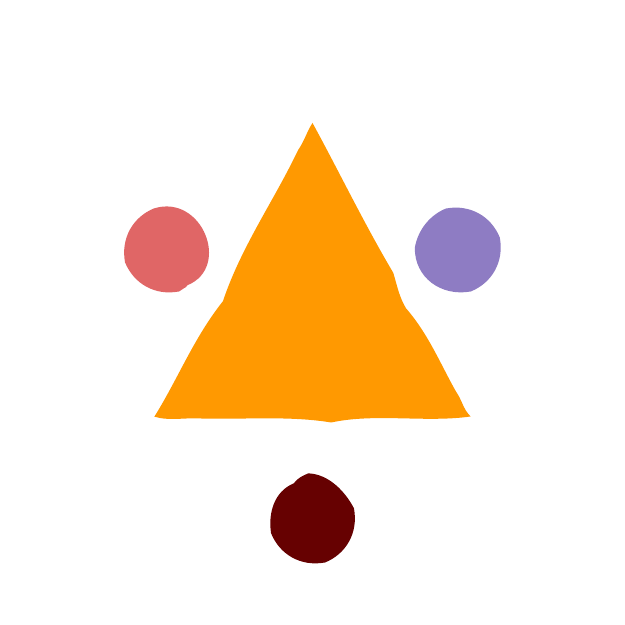}
    \end{subfigure}
    \hspace{-.7cm}
    \begin{subfigure}[b]{.3\textwidth}
      \centering
      \subcaption{}
      \begin{tikzpicture} [spy using outlines={rectangle,red,magnification=15,size=1.75cm, connect spies}]
        \node[anchor=south west,inner sep=0] (image) at (0,0) {\includegraphics[width = \textwidth]{pdf/illusory/illusory_parallel_2.pdf}};
        \spy[red, connect spies] on (2.93, 1.645) in node at (5,1.5);
        \spy[red, connect spies] on (3.5, 2.68) in node at  (5,3.5);
      \end{tikzpicture}        
    \end{subfigure}
    \caption{[Grouping shape layers] (a) The given color quantized $f$ inspired by Kanizsa triangle. (b) The same color region is considered as one shape layer, and the convexified shape layer of the three dots and the orange triangle which now becomes one connected region is shown. (c) The top shape layer is superpose over (b) which successfully recovers the original input raster image.  This process helps straighten the curves. }
    \label{fig: illusory}
\end{figure}

In Definition \ref{def:shapelayer}, each connected components are considered as separate shape layers.  For  images with illusory shapes or if there is a priori knowledge of occlusions among shapes, one can further consider grouping shape layers to add known semantic information.  For example, Figure \ref{fig: illusory}(a) is created inspired by the Kanizsa triangle.  It is not unusual to think the three orange triangles should be connected to each other to make up one big triangle. This is following a common illusion to human, since the straight level line extensions may connect the triangles. 
In this case, by simply considering all connected components with the same color as one shape layer, the proposed method finds one large occluded orange triangle  to be under the blue triangle without any additional modifications to the algorithm.  Figure \ref{fig: illusory}(b) shows the shapes at the bottom: all three dots and the orange triangle are convexified. Figure \ref{fig: illusory}(c) shows the final result with the blue triangle superposed over the orange.
The proposed depth ordering energy aligns with our human instinct, and identifies the orange triangle as one single shape to be underneath the blue, and also the proposed curvature-based inpainting successfully connects the disjoint orange triangles.  This also helps to approximate the direction of boundary sharply and approximate T-junctions better following the level line directions as shown in the zoom.

\section{Concluding Remarks}\label{sec:conclude}
We propose a non-learning based method called image vectorization with depth, that combines depth ordering and curvature-based  inpainting for new ways to shape decompose and vectorize a given raster image.  We propose  a novel depth ordering energy to identify shapes' relative depth ordering and provide analysis of its properties.  Effectiveness of this depth ordering energy is also shown in the experiment section. The combination of depth ordering and shape convexification not only gives an easier way to edit images, but also gives a more semantic vectorization result. Compared to recent work like LIVE~\cite{ma2022layer}, DiffVG~\cite{li2020diffvg} and LIVSS~\cite{wang2024layered}, the proposed method is fast, less demanding in computation resource and, more importantly, better preserves shapes as a whole.

There are some challenges that can be addressed as a future work. One key area of focus is improving denoising techniques for color quantized images, ensuring that meaningful fine details are preserved to enhance vectorization quality. Additionally, developing methods to intelligently group disconnected regions into single shape layers which aligns with human visual perception could streamline the editing process. Another promising direction for future research is the creation of neural networks capable of producing high-quality vectorization with depth. 

\bibliographystyle{plain}
\bibliography{arxiv_submission}

\appendix

\section{Pseudo Code for multiscale quantization, depth ordering and \bezier curve fitting} \label{Asec:codes}

In this section, we include the pseudo code for multiple algorithms mentioned in the previous sections, including depth ordering, grouping quantization and \bezier curve fitting. 

\begin{algorithm}
\caption{Depth Ordering graph $G(M,E)$ of Shape Layers}
\label{alg: depth ordering}
\begin{algorithmic}
\STATE{\textbf{Input:} Shapes $\{\chi_{i}\}_{i=1}^{N_\mathcal{S}}$ given from $\mathcal{S}$ as binary images.}
\STATE{\textbf{Output:} A directed graph $G$.}

\vspace{0.2cm}
\STATE{Set up $N_\mathcal{S}$ number of nodes $M$ that each represents a shape in a directed graph $G$. }

\FOR{every pair of $S_i$ and $S_j$ (or each pair of $S_i$ and $S_j$ that share a mutual boundary for real images)}
\STATE{Use Proposition \ref{prop: subset}
or equation (\ref{eq: ordering result}) to determine which one is above.}
\IF{ $S_i$ is above  $S_j$}
    \STATE{Draw a directed edge from node $i$ to node $j$ in $G$.}
\ELSIF{$S_j$ is above $S_i$}
    \STATE{Draw a directed edge from node $j$ to node $i$ in $G$.}
\ENDIF
\ENDFOR
\FOR{every cycle in $G$}
\STATE{Delete the edge $E_{i,j}$ in this cycle that is the solution to equation \eqref{eq: edge largest CS difference}.}
\ENDFOR
\STATE{Perform topological sort~\cite{cormen01introduction} to $G$ to obtain linear ordering $\mathcal{D}$.}
\end{algorithmic}
\end{algorithm}

\begin{algorithm}
\caption{\bezier Curve Fitting Algorithm}
\label{alg: fitting algorithm}
\begin{algorithmic}
    \STATE{\textbf{Input:} $\{ p_{i}\}_{i=1}^{n}$: a set of discrete points on 2D plane describing the boundary of a shape in either clockwise or counter-clockwise orientation; $\mathcal{T}$: threshold to classify local curvature extremum; $\tau$: error tolerance. }
    \STATE{\textbf{Output:} A set of tuples of coefficients for output \bezier curves.}

    \vspace{0.2cm}
    \STATE{Use equation \eqref{eq: discretized curvature} to compute curvature for each point in $\{p_{i} \}_{i=1}^{n}$.}
    \STATE{Initialize $\mathcal{I} = \{ i_{k} \mid |\kappa(p_{i_{k}})| > \mathcal{T} \text{ and } i_{k} < i_{k+1} \}$. If there is no such $p_{i}$, initialize $\mathcal{I} = \{p_{1}\}$.}
    \STATE{Let $n = | \mathcal{I} |$ and $O$ be an empty output list.}
    \FOR{$k = 1 , \cdots, n$}
    \STATE{Get the set of points $p_{i_{k}} , p_{i_{k}+1}, \cdots, p_{i_{k+1 \mod n}}$.}
    \STATE{Let $s_{1} \gets i_{k}$, $s_{2} \gets i_{k+1 \mod n}$.}
    \STATE{\textbf{do}}
    \STATE{\quad Solve the least-squares problem in  \eqref{eq: wls fitting} for $\mathbf{P}_{0},\mathbf{P}_{1},\mathbf{P}_{2},\mathbf{P}_{3}$ using the points $p_{s_{1}}$ to $p_{s_{2}}$.}
    \STATE{\quad \textbf{if} the Hausdorff distance of this \bezier curve and the points $p_{s_{1}}$ to $p_{s_{2}}$ $>\tau$ \textbf{then}}
    \STATE{\quad \quad $s_{2} \gets \text{the index between $s_{1}$ and $s_{2}$ where the maximum error occurs}$}.
    \STATE{\quad \textbf{else}}
    \STATE{\quad \quad$s_{1}\gets s_{2}, s_{2} \gets i_{k+1 \mod n}$.}
    \STATE{\quad \textbf{end if}}
    \STATE{\textbf{while} $s_{1} \neq i_{k+1 \mod n}$}
    \STATE{Store the final $(\mathbf{P}_{0},\mathbf{P}_{1},\mathbf{P}_{2},\mathbf{P}_{3})$ to $Z$.}
    \ENDFOR
    \STATE{Return $Z$}
    
\end{algorithmic}
\end{algorithm}

\end{document}